\documentclass[opre,nonblindrev]{informs3}

\DoubleSpacedXI 

\usepackage{endnotes}
\usepackage{amsmath, amssymb}
\usepackage{enumitem}
\usepackage{graphicx}
\usepackage{subcaption}
\usepackage{float}
\let\footnote=\endnote
\let\enotesize=\normalsize
\def\notesname{Endnotes}%
\def\makeenmark{\hbox to1.275em{\theenmark.\enskip\hss}}
\def\enoteformat{\rightskip0pt\leftskip0pt\parindent=1.275em
  \leavevmode\llap{\makeenmark}}

\definecolor{myblue}{RGB}{0, 20, 114}
\usepackage[hidelinks,colorlinks=true,urlcolor=myblue,linkcolor=myblue,citecolor=myblue]{hyperref}
\def\EMAIL#1{\href{mailto:#1}{#1}}

\def\theARTICLETOP{}
\def\theLRHFirstLine{}
\def\theLRHSecondLine{}
\def\theRRHFirstLine{}
\def\theRRHSecondLine{}

\usepackage{natbib}
 \bibpunct[, ]{(}{)}{,}{a}{}{,}%
 \def\bibfont{\small}%

\newenvironment{experiment}[1]
{
\noindent\textbf{Large Language Model Policy Optimization #1}\par
\vspace{0.5em}
\begin{enumerate}[leftmargin=*, label=\textbf{Step \arabic*.}]
}
{
\end{enumerate}
}

\TheoremsNumberedThrough     
\ECRepeatTheorems

\EquationsNumberedThrough    

\begin{document}


\RUNAUTHOR{Hu, Gao, Hu, Zhou}

\RUNTITLE{Adaptive Simulation Experiment for LLM Policy Optimization}

\TITLE{Adaptive Simulation Experiment for Large Language Model Policy Optimization}

\ARTICLEAUTHORS{%
\AUTHOR{Mingjie Hu}
\AFF{School of Management, Fudan University \\
H. Milton Stewart School of Industrial and Systems Engineering, Georgia Institute of Technology\\
\EMAIL{23110690009@m.fudan.edu.cn}} 
\AUTHOR{Siyang Gao}
\AFF{Department of Systems Engineering, City University of Hong Kong, \EMAIL{siyangao@cityu.edu.hk }}
\AUTHOR{Jian-qiang Hu}
\AFF{School of Management, Fudan University, \EMAIL{hujq@fudan.edu.cn }}\AUTHOR{Enlu Zhou}
\AFF{H. Milton Stewart School of Industrial and Systems Engineering, Georgia Institute of Technology, \EMAIL{enlu.zhou@isye.gatech.edu}}
} 

\ABSTRACT{%
Large language models (LLMs) have significant potential to improve operational efficiency in operations management. Deploying these models requires specifying a policy that governs response quality, shapes user experience, and influences operational value. In this research, we treat LLMs as stochastic simulators and propose a pairwise comparison-based adaptive simulation experiment framework for identifying the optimal policy from a finite set of candidates. We consider two policy spaces: an unstructured space with no parametric assumption, and a structured space in which the data are generated from a preference model. For both settings, we characterize the fundamental data requirements for identifying the optimal policy with high probability. In the unstructured case, we derive a closed-form expression for the optimal sampling proportions, together with a clear operational interpretation. In the structured case, we formulate a regularized convex program to compute the optimal proportions. We then develop an adaptive experimental procedure, termed LLM-PO, for both policy spaces, and prove that it identifies the optimal policy with the desired statistical guarantee while asymptotically attaining the fundamental data requirements. Numerical experiments demonstrate that LLM-PO consistently outperforms benchmark methods and improves LLM performance.
}%



\maketitle

%


\section{Introduction}\label{sec:Intro}

In recent years, LLMs have developed rapidly and have significantly changed the technological landscape \citep{ouyang2022training, wei2022emergent,guo2025deepseek}. Researchers have continued to develop effective methods to improve the performance of LLMs on a wide range of downstream tasks, including coding, reasoning, planning, summarization, and decision support, which has led to their widespread adoption in industry. 

Although training LLMs requires substantial upfront investment, their relatively low inference cost after deployment makes them increasingly accessible to small businesses and operational organizations. This creates new opportunities for these organizations to improve operational efficiency, provide better service, and grow in a cost-effective way \citep{vidgof2023large}. For instance, industry reports show that the fintech company Klarna has deployed OpenAI-powered assistants for customer service, and the enterprise software provider Zendesk offers AI agents that handle customer requests in production. Similar adoption has also appeared in healthcare operations. For example, Providence, a not-for-profit health system, has used Azure OpenAI Service to manage and triage patient messages in support of clinical workflows.

When deploying LLMs in real-world business and operational environments, organizations must specify key design choices such as system prompts, safety guardrails, and sampling hyperparameters. System prompts define the model’s operational rules and response style before it interacts with users. Safety guardrails help filter harmful inputs and keep outputs safe, controllable, and consistent with organizational values. Sampling hyperparameters (such as temperature) further influence the model’s responses by controlling the degree of randomness and creativity. These design choices together define a \emph{policy} for the LLM. 

Optimizing the policy is critical for the real-world deployment of LLMs. In customer-facing applications such as Klarna and Zendesk, it directly influences whether LLMs produce accurate, concise, and reliable responses, and therefore affects customer satisfaction and service efficiency. In healthcare operations such as Providence, it also determines how effectively language models support the routing and handling of patient communications within clinical systems. In general, the choice of policy affects response performance, user experience, operational reliability, and ultimately the organizational value generated by model deployment.

However, optimizing policies for LLMs poses several fundamental challenges:
\begin{enumerate}
    \item \textbf{Black-box system.} Most LLMs are inherently black-box systems with stochastic outputs. For any given input, we can only query the model and observe sampled responses, without access to its internal structure, gradients, or parameters. This substantially limits the applicability of classical optimization methods that rely on explicit knowledge of the underlying model.
    \item \textbf{Expensive Data Collection.}  Evaluating a policy requires repeated API calls or local model inference, both of which incur considerable monetary and computational expenses, making sample efficiency a primary concern. 
    \item \textbf{Preference Data.} Evaluating LLM responses usually requires some measure of response quality, and in many real-world applications, it is difficult to assign a reliable numerical score to a response. One practical alternative is to elicit pairwise preferences between two responses rather than direct numerical evaluations. However, such feedback provides only relative information, which makes policy optimization more challenging because the learner cannot directly observe absolute policy quality.
    \item \textbf{Performance Guarantee.} Practitioners often want assurance that the policy selected for deployment is truly the best one among the candidate policies. Therefore, it is important not only to identify a strong policy efficiently, but also to provide rigorous guarantees on the quality of the selected policy.
\end{enumerate}

\begin{figure}
    \centering
\includegraphics[width=0.8\linewidth]{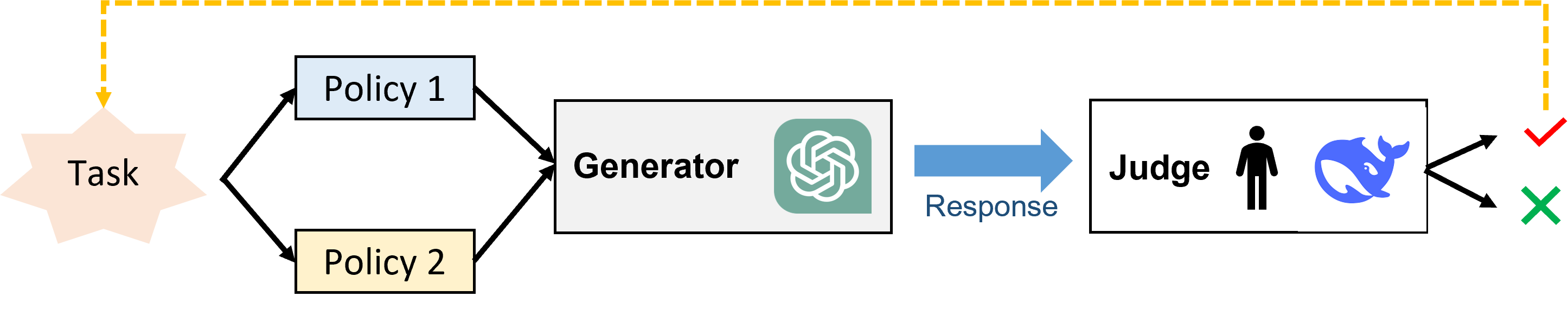}
    \caption{Adaptive Simulation Experiment Framework}
    \label{fig: exp process}
\end{figure}

To address these challenges, we propose an adaptive simulation experiment framework for policy optimization. By viewing an LLM as a black-box stochastic simulator, policy optimization can be naturally formulated as a simulation optimization problem, in which the objective function is unknown and can only be explored via noisy observations. For an overview of simulation optimization, please refer to \citep{fu2005simulation}. 

The framework is based on two main considerations. First, it uses an adaptive experimental design that sequentially selects policies for evaluation based on the evidence accumulated so far. As the experiment progresses, sampling increasingly concentrates on the most informative comparisons, while clearly inferior policies are sampled less often. This adaptive allocation is crucial for improving sample efficiency when evaluations are costly.

Second, to accommodate preference-based feedback, we introduce a pairwise-comparison experimental protocol, illustrated in Figure \ref{fig: exp process}. In each iteration, a task instance (e.g., a user query) is sampled, two candidate policies are selected, and the corresponding responses are generated by the language model. A judge, which may be either a human evaluator or another language model, then provides a binary preference indicating which response is favored. This comparison outcome is incorporated into the current estimate of policy quality, and the policy selection rule is updated accordingly for future iterations. By repeating this process, the experiment gradually accumulates the most informative preference data for identifying the optimal policy.

To provide a performance guarantee for the selected policy, we formulate the problem in the fixed-confidence setting and design a sequential experiment that identifies the optimal policy with probability at least $1-\delta$ for a prescribed risk level $\delta$. This formulation also gives a clear stopping criterion for data collection when preference queries are costly.

We summarize the main contributions of this work as follows.
\begin{itemize}
    \item We propose a pairwise comparison-based adaptive simulation experiment framework for policy optimization in LLMs.
    
    \item We characterize the fundamental data requirements for identifying the optimal policy with high confidence in both unstructured and structured policy spaces.
    
    \item For the unstructured policy space, we derive a closed-form optimal sampling allocation. For the structured policy space, we develop an \(\ell_2\)-regularized approach to address the non-uniqueness of optimal sampling allocations.
    
    \item We develop an adaptive simulation experiment, termed LLM-PO, by specifying its sampling, stopping, and decision rules, and prove that it identifies the optimal policy with probability at least \(1-\delta\) and achieves optimal data requirements almost surely.
    
    \item We conduct both synthetic and real experiments to demonstrate that the proposed method can identify the optimal policy efficiently.
\end{itemize}

\subsection{Literature Review}
Our work is related to ranking and selection, prompt engineering, and preference-based optimization. We review the most relevant literature as follows.

\textbf{Ranking and Selection.} This research formulates policy optimization as a simulation optimization problem and is closely related to the literature on ranking and selection (R\&S) in simulation optimization \citep{chen2015ranking}. The goal of R\&S is to identify the best alternative from a finite set of candidates. Depending on the application setting, existing methods are typically studied under either the fixed-budget or the fixed-confidence setting. In the fixed-budget setting, the total simulation budget is specified in advance, and the objective is to adaptively allocate samples to maximize the probability of correct selection \citep{chen2000simulation,peng2018ranking,wang2023best}. In the fixed-confidence setting, the objective is to guarantee that the probability of correct selection exceeds a prescribed confidence level \citep{kim2001fully, fan2016indifference}. 

Our problem is fundamentally different from standard R\&S. We consider a pairwise-comparison setting, where the learner observes only binary preference outcomes rather than numerical rewards. As a result, existing procedures and their optimality guarantees are not directly applicable. This feedback structure introduces new technical challenges that must be addressed to achieve efficient policy optimization.

To the best of our knowledge, the most closely related work is the pairwise comparison literature in R\&S \citep{xiao2023ranking,li2024thompson,groves2019top}, which is connected to the unstructured policy space studied in this paper. However, this line of work considers a fixed-budget setting and does not provide statistical guarantees on the quality of the selected solution. More importantly, its definition of the optimal policy and its assumptions on the data-generating distribution differ fundamentally from ours, making the proposed methods inapplicable to our unstructured setting.

\textbf{Prompt Engineering.} Existing prompt engineering research studies how to design or optimize prompts for LLMs to improve the task performance. Early work formulates prompt design as a discrete search problem over natural-language templates, including automatic prompt engineering and instruction induction \citep{zhou2022large,honovich2023instruction}. Subsequent studies have developed a range of automated search methods, including gradient-based optimization \citep{pryzant2023automatic}, evolutionary search \citep{guo2023connecting,fernando2023promptbreeder}, and surrogate-model-guided search \citep{zhang2024language, schneider2024hyperband}. Despite their empirical success, these methods are often not suitable for real operations because they offer limited control over both the prompts explored during search and the final prompt selected for deployment. The resulting prompt is often a free-form string with limited interpretability, which may violate safety guardrails or create compliance risks. In contrast, our framework optimizes over a predefined, interpretable set of policy candidates, so it can ensure that the selected policy is controllable and directly deployable in practice.

\textbf{Preference-based Optimization.} Preference-based optimization studies learning and optimization problems where numerical rewards are unavailable, and only relative preference feedback (typically in the form of pairwise comparisons) is observed. A major line of research in this area is dueling bandits \citep{bengs2021preference}, where the objective is either to minimize regret \citep{zoghi2014relative,wu2016double} or to identify the best arm with as few comparisons as possible \citep{jamieson2015sparse, bengs2024identifying}. While these works are relevant to our setting in the unstructured case, our formulation differs in the performance measure, the definition of the optimal policy, and the type of optimality guarantee. Our framework also extends to structured policy spaces, which improve the scalability of the experimental procedure but substantially complicate both the design of adaptive experiments and the theoretical analysis. Related structured preference-based formulations have been studied in contextual dueling bandits \citep{li2024feel} and reinforcement learning from human feedback \citep{scheid2024optimal,zhu2023principled}. However, our work focuses on adaptive simulation experiment design for policy optimization, leading to fundamentally different problem formulation and theoretical treatment.

The rest of the paper is organized as follows. Section \ref{sec: formulation} formulates LLM policy optimization as an adaptive simulation experiment design problem. Section \ref{sec: lower bound} establishes a lower bound on the data requirements and characterizes the corresponding optimal sampling proportions. Section \ref{sec: experiment design} presents the proposed experimental procedure together with its optimality analysis. Section \ref{sec: experiments} reports the numerical results. Section \ref{sec: conclusion} concludes the paper. All technical details are deferred to the Electronic Companion of this paper.

\section{Large Language Model Policy Optimization}
\label{sec: formulation}

Deploying an LLM system requires specifying several key components, including the system prompt, safety guardrails, and sampling hyperparameters (e.g., temperature). These components together define a policy, which governs the model’s reasoning process and output behavior. As a result, optimizing the policy is central to the reliability and performance of downstream applications. Formally, let $\mathcal{P}=\{P_1,\ldots,P_M\}$, $\mathcal{G}=\{G_1,\ldots,G_N\}$, and $\mathcal{T}=\{T_1,\ldots,T_L\}$ denote finite sets of system prompts, safety guardrails, and sampling hyperparameters, respectively. Each policy is represented by a tuple $i=(P_{m},G_{n},T_{l})$ drawn from the Cartesian product $\mathcal{P} \times \mathcal{G} \times \mathcal{T}$. We denote the total number of feasible policies by $K$ and index them by $[K]=\{1,\ldots,K\}$.

A task (i.e., a user query or prompt) $x\in\mathcal{X}$ is drawn i.i.d from an unknown
distribution $\mathcal{Q}$. Let $\mathcal{A}$ denote the LLM, which maps the input space $\mathcal{X} \times [K]$ to a probability distribution $\Delta_{\mathcal{V}}$ over the language space $\mathcal{V}$. Given a policy $i \in [K]$ and a task $x$, the LLM generates a response 
\begin{equation*}
   \hat{Y} = \mathcal{A}(x,i).
\end{equation*}
The random variable $\hat{Y} \in \mathcal{V}$ represents the generated response. In particular, for any fixed task $x$ and policy $i$, the response $\hat{Y}$ is generally stochastic.

To evaluate policies, we select a pair of distinct policies $(i,j)$ and generate their corresponding responses. A judge (either a human evaluator or another LLM) then returns a pairwise comparison outcome $D(i,j)\in\{0,1\}$ for $i\neq j$, where $D(i,j)=1$ indicates that the response under policy $i$ is preferred to that under policy $j$. We assume $D(i,j) = 1-D(j,i)$. Since $D(j,i)$ is fully determined by $D(i,j)$ for all $i$ and $j$, it suffices to consider the set of policy pairs $\mathcal{S} = \{(i,j):i<j, i,j\in[K]\}$. Let $|\mathcal{S}|$ denote the cardinality of $\mathcal{S}$. 

In the pairwise comparison framework, since both response generation by the LLM and the evaluation process are stochastic, the pairwise comparison outcome $D(i,j)$ is random. We define the pairwise mean performance as $\mu(i,j) = \mathbb{E}[D(i,j)]$, where the expectation is taken over the randomness in the sampled tasks, the generated responses, and the evaluation process. Accordingly, $\mu(i,j)$ represents the probability that the response produced under policy $i$ is preferred to that produced under policy $j$. 

We define the best policy as 
\begin{equation}
\label{eq: optpolicy}
    i^\star \in \argmax_{i\in[K]}\min_{j\neq i} \mu(i,j). 
\end{equation}
That is, the best policy $i^\star$ maximizes the worst-case pairwise winning probability over all other policies. In particular, when such a policy exists, $i^\star$ is undominated in the sense that
$\mu(i^\star,j)>1/2$ for all $j\neq i^\star$. By viewing the LLM as a black-box simulator, policy optimization can be formulated as a simulation optimization problem. However, the feedback consists of binary pairwise comparisons, so policy performance cannot be estimated directly by empirical averages and must instead be inferred from relative outcomes. This makes many classical sampling rules and optimality guarantees for simulation optimization not applicable. In addition, response generation and evaluation are costly, and hence call for adaptive data collection strategies that can identify the optimal policy efficiently.

To address these challenges, we propose an adaptive experiment framework for LLM policy optimization. The experiment proceeds sequentially. At each time step $t$, a random task $x_t$ is drawn from the distribution $\mathcal{Q}$. The decision maker then selects a pair of policies $(i_t,j_t)$, generates the corresponding responses, and obtains a binary comparison outcome $D(i_t,j_t)$ from the judge. Let 
$$
    \mathcal{F}_t = \sigma(x_1,(i_1,j_1),D(i_1,j_1),\ldots,x_t,(i_t,j_t),D(i_t,j_t))
$$
denote the sigma-field generated by all information observed up to time $t$. An adaptive simulation experiment design consists of the following three components:
\begin{itemize}
    \item \textbf{Policy pair sampling rule} $(i_t,j_t)_{t\ge1}$: specifies the policy pair $(i_t,j_t)$ to be evaluated at time $t$, where $(i_t,j_t)$ is $\mathcal{F}_{t-1}$ measurable.
    \item \textbf{Experiment stopping rule} $\tau$: determines when to terminate data collection and thus specifies the total number of binary preference observations. The random variable $\tau$ is a stopping time with respect to $\mathcal{F}_t$.
    \item \textbf{Final decision rule} $\hat{i}_{\tau}$: outputs an estimate of the best policy based on the information available at time $\tau$. The decision rule $\hat{i}_{\tau}$ is $\mathcal{F}_{\tau}$-measurable.
\end{itemize}
The goal of an optimal experiment design is to identify the best policy $i^\star$ with a high probability while minimizing the number of collected preference observations.

\section{Fundamental Data Requirements}
\label{sec: lower bound}

In this section, we establish a lower bound on the experiment stopping time $\tau$, which characterizes the minimum amount of data that any adaptive simulation experiment requires to identify the optimal policy under a prescribed statistical guarantee. 

\subsection{Unstructured Policy Space}

In this subsection, we consider an unstructured policy space, where no parametric or structural assumptions are imposed on the binary preference data-generating process. We begin by introducing the $\delta$-probability approximately correct ($\delta$-PAC) criterion, which provides a statistical benchmark for evaluating an adaptive simulation experiment (Definition~\ref{def: delta-PAC}). Throughout this paper, we focus on designing simulation experiments that satisfy this $\delta$-PAC criterion.

\begin{definition}
\label{def: delta-PAC}
A simulation experiment design is said to be $\delta$-PAC if its final decision rule $\hat{i}_{\tau}$ satisfies
$
    \mathbb{P}(\hat{i}_{\tau}\neq i^\star) \leq \delta.
$
\end{definition}

Among all feasible $\delta$-PAC simulation experiments, an optimal experiment design identifies the optimal policy with the fewest binary preference observations while guaranteeing an error probability of at most $\delta$. To characterize this fundamental limit, we establish a problem-dependent lower bound on the required sample size, namely, the minimum number of preference comparisons that any $\delta$-PAC experiment must collect. This lower bound provides an information-theoretic benchmark for experiment design by capturing the intrinsic difficulty of the underlying decision problem. It also offers an interpretable characterization of the hardness of the policy optimization task.

To derive a lower bound on the stopping time $\tau$, we employ the change-of-measure technique that forms the foundation of many classical lower bounds in sequential hypothesis testing \citep{garivier2016optimal}. Our main idea is to formulate policy optimization as a hypothesis testing problem and analyze it from an adversarial perspective. For a fixed problem instance, we construct an alternative instance in which the identity of the optimal policy is changed, while the induced preference data distribution remains nearly indistinguishable from that of the original instance in terms of Kullback-Leibler (KL) divergence. Any $\delta$-PAC simulation experiment must then reliably distinguish between these two instances based on the observed comparisons; otherwise, it incurs a non-negligible error probability on at least one of them. By constructing such a most confusing alternative and comparing the corresponding likelihoods of the observations, we derive the minimum amount of statistical evidence that must be accumulated. This requirement directly yields a lower bound on the number of comparisons, and hence on $\tau$, required to detect the change in the optimal policy.

Let $\mu = (\mu(i,j))_{i,j\in[K]}$ denote the true instance of the LLM policy optimization problem.  To facilitate the analysis, let $\lambda \in \mathbb{R}^{K\times K}$ denote a generic candidate instance. We restrict attention to the class of instances under which the optimal policy is unique. Specifically, define
$$
    \mathcal{H} := \left\{\lambda: \lambda(i,i) = \frac{1}{2}, \forall i\in[K], \lambda(i,j) = 1-\lambda(j,i),\forall i\neq j,\exists i^\star(\lambda)\in[K],\min_{j\neq i^\star(\lambda)}\lambda(i^\star(\lambda),j)>\frac{1}{2} \right\},
$$
where $i^\star(\lambda)$ denotes the unique optimal policy under instance $\lambda$.
Given the instance $\mu$, we are interested in alternative instances under which the optimal policy differs from that under $\mu$. Accordingly, define the set of alternatives by
\begin{equation*}
    \mathrm{Alt}(\mu) := \left\{\lambda \in \mathcal{H}: i^\star(\mu)\neq i^\star(\lambda)\right\}.
\end{equation*}

Theorem~\ref{thm: lower bound} characterizes the fundamental data requirement for reliable LLM policy optimization. It shows that, to identify the best policy with error probability at most $\delta$, any adaptive simulation experiment must use at least $\mathcal{T}^\star(\mu)\log(1/\delta)$ comparisons asymptotically. Here, $\mathcal{T}^\star(\mu)$ captures the intrinsic difficulty of distinguishing the optimal policy from competing policies under the underlying preference structure. In the context of LLM policy optimization, this result highlights a fundamental limit: when policies induce similar response quality, more costly comparisons are unavoidable, regardless of the experiment design.

\begin{theorem}
\label{thm: lower bound}
Fix a risk level \(\delta \in (0,1/2)\). Then, for any \(\delta\)-PAC adaptive simulation experiment with stopping time \(\tau\) and any problem instance \(\mu \in \mathcal{H}\), the expected sample size satisfies
$$
    \mathbb{E}_{\mu}[\tau]
    \;\ge\;
    \mathcal{T}^\star(\mu)\,\text{kl}(\delta,1-\delta),
    \qquad
    \liminf_{\delta \to 0}
    \frac{\mathbb{E}_{\mu}[\tau]}{\log(1/\delta)}
    \;\ge\;
    \mathcal{T}^\star(\mu),
$$
where $\mathcal{T}^\star(\mu)^{-1}$ is defined as
\begin{equation}
\label{eq: complexity opt}
    \sup_{\omega \in \Omega}
    \inf_{\lambda \in \mathrm{Alt}(\mu)}
    \sum_{(i,j)\in\mathcal{S}} \omega_{ij}\, d\bigl(\mu(i,j),\lambda(i,j)\bigr).
\end{equation}
Here, \(d(x,y)\) denotes the KL divergence between Bernoulli distributions with means \(x\) and \(y\), and
$
\Omega:=\{\omega\in\mathbb{R}_+^{|\mathcal{S}|}:\sum_{(i,j)\in\mathcal{S}}\omega_{ij}=1\}
$
is the simplex of sampling proportions over all pairwise comparisons.
\end{theorem}

Fix a suboptimal policy $i\neq i^\star(\mu)$. For notational convenience, write the pairwise parameter as $\mu(j,i)$ for $j<i$ and as $\mu(i,h)$ for $i<h$. Define
$$
\mathcal{I}_1(i):=\{j\in[K]: j<i,\ \mu(j,i)>\tfrac12\},\qquad
\mathcal{I}_2(i):=\{h\in[K]: i<h,\ \mu(i,h)<\tfrac12\}.
$$
Then $\mathcal{I}_1(i)\cup\mathcal{I}_2(i)$ 
is exactly the set of policies that outperform $i$. Next, define
$$
d_i^\star:=\max\Big\{\max_{j\in\mathcal{I}_1(i)} d\left(\mu(j,i),\tfrac12\right),\ \max_{h\in\mathcal{I}_2(i)} d\left(\mu(i,h),\tfrac12\right)\Big\},
$$
which represents the largest per-sample information available for ruling out policy $i$. Finally, let $\tilde j(i)\in\mathcal{I}_1(i)\cup\mathcal{I}_2(i)$ be any maximizer achieving $d_i^\star$. In other words, $\tilde j(i)$ is an opponent whose comparison with $i$ is the most informative for eliminating $i$.

We now solve the optimization problem \eqref{eq: complexity opt} and characterize the corresponding optimal sampling proportions. The main difficulty is that the sampling proportions are coupled across pairwise comparisons such that for each suboptimal policy $i\neq i^\star(\mu)$, several comparisons may contribute evidence for eliminating $i$. The KKT conditions reveal a useful bottleneck structure. Among all comparisons involving $i$, only those with its most informative opponents need to be sampled, because these comparisons provide the strongest statistical evidence. This yields a reduction from the original optimization problem to allocating total sampling effort across suboptimal policies and their most informative comparison pairs. This leads to the explicit characterization that the optimal allocation is proportional to $1/d_i^\star$. The resulting solution is summarized in Corollary~\ref{corollary: optimal ratio}.
\begin{corollary}
\label{corollary: optimal ratio}
There exists an optimal solution \(\omega^\star\) to \eqref{eq: complexity opt} such that, for each suboptimal policy \(i \neq i^\star(\mu)\), all sampling effort associated with ruling out \(i\) is concentrated on the single comparison between \(i\) and \(\tilde j(i)\). That is, \(\omega_{ij}^\star = 0\) for all other pairs involving \(i\). Moreover,
$$
\omega_{\tilde j(i),\,i}^\star
=
\frac{1/d_i^\star}{\sum_{k \neq i^\star(\mu)} 1/d_k^\star},
\qquad \text{if } \tilde j(i)<i, \qquad \omega_{i,\,\tilde j(i)}^\star
=
\frac{1/d_i^\star}{\sum_{k \neq i^\star(\mu)} 1/d_k^\star},
\qquad \text{if } \tilde j(i)>i.
$$
\end{corollary}
\begin{remark}
\label{remark: nonunique omega}
The optimal solution to \eqref{eq: complexity opt} is not necessarily unique. For any suboptimal policy \(i \neq i^\star(\mu)\), define the set of most informative opponents by
$
\mathcal J_i^\star \;:=\; \arg\max_{j\neq i} d_{ij}(\mu).
$
If \(|\mathcal J_i^\star|>1\), then the sampling mass associated with policy \(i\) may be distributed arbitrarily across the comparisons \(\{(i,j): j\in \mathcal J_i^\star\}\) without changing the objective value. Consequently, all such allocations are optimal. In particular, when multiple opponents are equally informative, any convex combination of these optimal allocations is also optimal.
\end{remark}

Corollary~\ref{corollary: optimal ratio} suggests that, when comparing candidate policies, it is not necessary to evaluate each suboptimal policy with many others. It is sufficient to compare it only with the policy that most clearly outperforms it. This insight is especially valuable when each comparison requires costly response generation and evaluation. The corollary also shows how to allocate the simulation budget across policies so that effort is concentrated on the most decision-critical comparisons. For practitioners, it offers an interpretable and implementable rule for designing efficient evaluation procedures for LLM deployment.

\subsection{Structured Policy Space}
The number of feasible policies grows combinatorially with the design choices, including the system prompts $M$, safety guardrails $N$, and sampling hyperparameters $L$, resulting in $M\times N\times L$ candidate policies. Even when $M$, $N$, and $L$ are only moderately large, the resulting search space can be enormous, making exhaustive sampling and evaluation highly inefficient. In practice, however, policies that share similar prompts, guardrails, or sampling configurations tend to induce correlated and systematic patterns in responses, which create exploitable structure in the data-generating process. In this subsection, we formalize one such structural assumption and show how it can be used to improve sampling efficiency.

We assume that each policy $i$ is represented by a feature vector $x_i\in \mathbb{R}^d$, and that preference feedback is governed by an unknown parameter vector $\theta_\star\in\mathbb{R}^{d}$. Under a linear reward model, the latent score of policy $i$ is given by
$
    r_i = \theta_\star^\top x_i.
$ These latent scores are introduced only as a convenient parametrization of the binary preference model, and are not directly observed in the adaptive simulation experiment. To model pairwise feedback, we adopt the Bradley-Terry model \citep{bradley1952rank}. For any ordered pair of policies $(i,j)$, let $D(i,j)\in\{0,1\}$ denote the binary outcome. Conditioned on $\theta_\star$, we assume 
$$
\mu(i,j)
:=\mathbb{P}\big(D(i,j)=1\mid \theta_\star\big)
=\sigma\big(\theta_\star^\top(x_i-x_j)\big)
=\frac{1}{1+\exp\big(-\theta_\star^\top(x_i-x_j)\big)},
$$
where $\sigma(z):=1/(1+e^{-z})$ is the logistic function. The combination of a linear score structure and the Bradley-Terry preference model is standard in the preference learning literature \citep{scheid2024optimal,agnihotri2025best}.
This model immediately implies the reciprocity relation $\mu(j,i)=1-\mu(i,j)$. Moreover, $\mu(i,j)>\tfrac12$ if and only if
$\theta_\star^\top x_i>\theta_\star^\top x_j$, meaning that policy $i$ is more likely to be preferred than policy $j$ when it has a larger latent score.

\begin{assumption}
We impose the following regularity conditions throughout the subsequent analysis.
\leavevmode
\begin{enumerate}[label=(A\arabic*), leftmargin=2.8em, itemsep=0.4em]
    \item Conditional on the true parameter \(\theta_\star\), preference observations are independent across repeated iterations.
    
    \item The true parameter \(\theta_\star\) lies in the interior of a compact parameter space \(\Theta \subset \mathbb{R}^d\); that is, \(\theta_\star \in \operatorname{int}(\Theta)\). Moreover, there exists a constant \(B < \infty\) such that 
    $
   \|\nu\|_2 \le B
    $ for all $\nu\in \Theta$.
    
    \item Let \(z_{ij} := x_i - x_j\) denote the feature difference between policies \(i\) and \(j\). There exists a constant \(L < \infty\) such that
    $
    \|z_{ij}\|_2 \le L, \forall i,j \in [K].
    $
\end{enumerate}
\end{assumption}
\begin{remark}
Assumptions (A1)-(A3) ensure that the logistic preference model is well behaved both statistically and analytically. These assumptions are standard and mild, and are commonly adopted in the logistic bandit literature \citep{faury2020improved}.
\end{remark}

We consider the optimization problem \eqref{eq: complexity opt} under the structured policy space. In this setting, however, the linear feature structure together with the Bradley-Terry likelihood introduces substantial technical difficulty, so a closed-form characterization of the optimal sampling proportions is generally impossible. The main challenge is that the problem couples a combinatorial collection of opponents $i\neq i^\star(\mu)$ with a nonlinear, parameter-dependent information structure induced by $\sigma(\theta_\star^\top (x_i-x_j))$. This coupling prevents a closed-form solution for the optimal sampling allocation.

To overcome this difficulty, we derive an upper bound in the hard-instance regime by focusing on the least favorable instances that lie in a local neighborhood of the true instance, where the Bernoulli KL divergence admits a quadratic approximation. Under this local view, the problem acquires a Fisher information interpretation. Theorem~\ref{thm: combinatorial_structure} then shows that $\mathcal{T}^\star(\mu)$ can be upper bounded by $\mathcal{U}^\star(\mu)$, and yields an explicit max-min optimization problem that chooses the sampling proportions $\omega$ to maximize the minimum Fisher-normalized separation between the optimal policy and each suboptimal opponent. This provides a tractable surrogate for $\mathcal{T}^\star(\mu)$ and clearly reveals the problem-dependent factors that govern the complexity of the problem. Before stating the result, for any positive definite matrix \(M\), define the matrix-induced norm
$
\|x\|_{M} := \sqrt{x^\top M x}.
$

\begin{theorem}
\label{thm: combinatorial_structure}
Under the Bradley-Terry preference model and the local hard-instance approximation, \(\mathcal{T}^\star(\mu)\) is upper bounded by
$
\mathcal{U}^\star(\mu),
$
which satisfies
\begin{equation*}
    \mathcal{U}^\star(\mu)^{-1}
    =
    \sup_{\omega\in\Omega}
    \min_{i\neq i^\star(\mu)}
    \frac{1}{2}\,
    \frac{\bigl(\theta_\star^\top z_{i\,i^\star(\mu)}\bigr)^2}
    {\|z_{i\,i^\star(\mu)}\|_{H(\theta_\star,\omega)^{-1}}^2},
\end{equation*}
with
\[
H(\theta_\star,\omega)
=
\sum_{(i,j)\in\mathcal{S}}\omega_{ij}\,\sigma'\bigl(\theta_\star^\top z_{ij}\bigr)\,z_{ij}z_{ij}^\top
\]
denoting the Fisher information matrix at \(\theta_\star\) under sampling proportion \(\omega\).
\end{theorem}

Theorem~\ref{thm: combinatorial_structure} suggests that the evaluation budget should be concentrated on policy pairs whose comparisons are most informative for distinguishing the best policy from competing policies, that is, for estimating $\theta_{\star}$ along the separating direction $z_{ii^*(\mu)}$, rather than being spent on clearly uninformative comparisons. As a result, $ \mathcal{U}^\star(\mu)$ not only serves as a tractable approximation to the original term $\mathcal{T}^\star(\mu)$, but also provides an interpretable measure of how the structure of the policy space shapes the difficulty and cost of reliable LLM policy optimization.

\section{Adaptive Simulation Experiment}
\label{sec: experiment design}

In this section, we develop an adaptive simulation experiment for the policy optimization problem in LLMs under different policy spaces. We begin by introducing estimators for the problem instance. We then specify the sampling, stopping, and decision rules, and combine them into a unified experimental design. Finally, we establish the optimality guarantee of the proposed procedure.

\subsection{Parameter Estimation}
Since the problem instance $\mu$ is unknown, the optimal sampling proportions $\omega^\star$ cannot be computed directly. To address this issue, we construct an empirical estimator of $\mu$ at time $t$, denoted by $\hat{\mu}(t)= (\hat{\mu}(i,j;t))_{i\neq j}$, based on the observed binary pairwise comparison data.

Suppose that by the end of time $t-1$, we have observed the comparison data
$\{(i_s,j_s,D_s)\}_{s=1}^{t-1}$, where $D_s:=D(i_s,j_s)\in\{0,1\}$. In the unstructured policy space, we estimate $\mu(i,j)$ at time $t$ by the Monte Carlo estimator
$$
    \hat\mu(i,j;t) = \frac{1}{N_{ij}(t)} \sum_{s=1}^{t-1} D(i_s,j_s) \mathbb{I}(i_s=i,j_s=j), \quad N_{ij}(t) = \sum_{s=1}^{t-1} \mathbb{I}(i_s=i,j_s=j),
$$
where $N_{ij}(t)$ is the number of times that the policy pair $(i,j)\in\mathcal{S}$ has been sampled up to time $t$.

The statistical properties of this estimator, including its consistency, depend on the sampling rule used to select policy pairs. We specify this sampling rule in the next subsection.

In the structured policy space, fully exploiting the available structure requires estimating the global unknown parameter $\theta_\star$.  Under the conditional independence assumption in $(A1)$, the log-likelihood of $\theta$ based on the observations collected up to time $t-1$ is
$$
\mathcal{L}_t(\theta)
=\sum_{s=1}^{t-1}\Big[
D_s \log \sigma\big(\theta^\top(x_{i_s}-x_{j_s})\big)
+(1-D_s)\log \sigma\big(-\theta^\top(x_{i_s}-x_{j_s})\big)
\Big].
$$
We estimate $\theta_\star$ by the $\ell_2$-regularized maximum likelihood estimator
\begin{equation}
\label{eq:reg_mle_def}
\hat{\zeta}_t \in \arg\max_{\theta\in\mathbb{R}^d}\ \Big\{\mathcal{L}_t(\theta)-\frac{\lambda_t}{2}\|\theta\|_2^2\Big\},
\end{equation}
where $\lambda_t>0$ is a regularization parameter satisfying $\lambda_t\to 0$ as $t\to\infty$. The $\ell_2$ regularization stabilizes the estimation problem and makes the objective function strictly concave. It ensures that the estimator is well defined even under small sample budgets. Letting $\lambda_t$ vanish asymptotically eliminates the bias introduced by regularization and ensures convergence to the true parameter $\theta_{\star}$. In practice, \eqref{eq:reg_mle_def} can be solved efficiently using standard convex optimization methods, such as Newton or quasi-Newton algorithms. 

For $t\ge 1$ and $\theta\in\mathbb{R}^d$, let $z_s := x_{i_s}-x_{j_s}$ for notational convenience, and define $g_t(\theta)$ through the gradient decomposition
$$
\nabla_{\theta}\left(\mathcal{L}_t(\theta)-\frac{\lambda_t}{2}\|\theta\|_2^2\right)
= \sum_{s=1}^{t-1}D_s z_s-\underbrace{\left(\sum_{s=1}^{t-1}\sigma(\theta^\top z_s)z_s+\lambda_t\theta\right)}_{=:~g_t(\theta)}.
$$
We also define the Hessian 
\begin{equation*}
    H_t(\theta) := \sum_{s=1}^{t-1} \sigma^\prime(\theta^\top z_s)z_sz^\top_s + \lambda_t I_d,
\end{equation*}
where $I_d$ is the $d$-dimensional identity matrix. In addition, define the matrix $V_t$ at time $t$ by
\[
V_t:=\sum_{s=1}^{t-1} z_s z_s^\top,
\]
and let $\Lambda_{\min}(t)$ and $\Lambda_{\max}(t)$ denote the minimum and maximum eigenvalues of $V_t$, respectively. We next introduce the projected regularized maximum likelihood estimator
\begin{equation}
\label{eq: theta_estimator}
    \hat{\theta}_t = \argmin_{\theta\in\Theta} \left\lVert g_t(\theta) - g_t(\hat{\zeta}_t) \right\rVert_{H^{-1}_t(\theta)}.
\end{equation}
Rather than using $\hat{\zeta}_t$ directly, we base our experiment on the projected estimator $\hat{\theta}_t$, since it admits a time-uniform concentration inequality (see Lemma \ref{lemma: param_concentration} below) that is crucial for the optimality analysis.

The following two lemmas establish key statistical properties of the projected estimator $\hat{\theta}_t$. The analysis is nontrivial because of the time-varying regularization term $\lambda_t$ and the two-step construction of the estimator; see the Electronic Companion for technical details. These results may also be of independent interest in preference learning. 

Lemma \ref{lemma: param_concentration} shows that, under sufficient exploration, the estimated preference parameter $\hat{\theta}_t$ concentrates around the true parameter $\theta_{\star}$ exponentially fast, up to a vanishing bias term $b_t$.

\begin{lemma}
\label{lemma: param_concentration}
Suppose there exist constants $c>0$ and $t_0\ge 1$ such that $\Lambda_{\min}(t)\ge c\sqrt{t}$ $\text{a.s.}$ for all $t\ge t_0$. Then there exist constants $c_1,c_2,C_b>0$, independent of $\varepsilon$ and $t$, such that
for all $t\ge t_0$ and all $\varepsilon>0$,
$$
\mathbb{P}\left(\|\hat\theta_t-\theta_\star\|_2 \ge \varepsilon + b_t\right)
\ \le\
c_2\,t^{d/2}\,\exp\big(-c_1\,\varepsilon^2\,t^{1/2}\big),
$$
where $b_t := C_b\,t^{-1/4}\sqrt{\lambda_t}.
$
\end{lemma}

Lemma~\ref{lemma: param_consistence} shows that, under sufficient exploration and vanishing regularization, the projected estimator $\hat{\theta}_t$ converges almost surely to the true parameter $\theta_{\star}$ and also establishes an almost-sure convergence rate.

\begin{lemma}
\label{lemma: param_consistence}
Suppose there exist constants \(c>0\) and \(t_0\ge 1\) such that $\Lambda_{\min}(t)\ge c\sqrt{t}$ $\text{a.s.}$ for all $t\ge t_0$,
and assume that the regularization sequence satisfies \(\lambda_t>0\) for all \(t\) and \(\lambda_t\to 0\) $\text{a.s.}$ Then the projected estimator \(\hat{\theta}_t\) is strongly consistent:
$
\hat{\theta}_t \xrightarrow{\mathrm{a.s.}} \theta_\star.
$
In addition, for every \(\beta\in(0,1/4)\),
$
\|\hat{\theta}_t-\theta_\star\|_2=o\bigl(t^{-\beta}\bigr)
$ $\text{a.s.}$
\end{lemma}

\subsection{Experiment Design}
To design an adaptive simulation experiment, we need to specify the sampling, stopping, and decision rules. Given the estimator $\hat{\mu}(t)$ of the problem instance, a natural plug-in decision rule is
\begin{equation*} 
    \hat{i}_{\tau} \in \argmax_{i\in[K]}\min_{j\neq i}\hat{\mu}(i,j;\tau).
\end{equation*}
Thus, the main challenge is to design the sampling and stopping rules so that the resulting procedure is $\delta$-PAC and attains the minimum data requirements.

In the unstructured policy space, Corollary~\ref{corollary: optimal ratio} provides a closed-form optimal sampling proportion $\omega^\star(\mu)$. By plugging in the estimator $\hat{\mu}(t)$, we obtain the empirical allocation $\omega^\star(\hat{\mu}(t))$, which is then used to guide the sampling process.

In the structured policy space, the problem instance $\mu$ is fully characterized by the unknown parameter $\theta_\star$. Therefore, we use $\mu$ and $\theta_\star$, as well as their estimators $\hat{\mu}(t)$ and $\hat{\theta}_t$, interchangeably whenever this causes no ambiguity. The optimal sampling proportion is defined as a solution to
\begin{equation*}
    \omega^\star(\theta_\star)\in \argmax_{\omega\in\Omega}\;
    \min_{i\neq i^\star(\theta_\star)}
    \frac{1}{2}\,
    \frac{\big(\theta_\star^\top z_{i\,i^\star(\theta_\star)}\big)^2}{\big\|z_{i\,i^\star(\theta_\star)}\big\|^2_{H(\theta_\star,\omega)^{-1}}},
\end{equation*}
and is thus computed by solving this optimization problem.

However, a major technical challenge is that the optimal allocation $\omega^\star(\theta_\star)$ may not be unique. This occurs naturally in symmetric instances where several pairs are equally informative for distinguishing the optimal policy from its opponents and induce the same Fisher information. Consequently, multiple sampling proportions may achieve the same optimal objective value. This non-uniqueness substantially complicates the analysis, because a naive tracking rule may oscillate among these optimal allocations, so the empirical sampling proportion need not converge, even though the total information collected is asymptotically optimal.

To address this issue, we solve a regularized empirical allocation problem at each iteration. Specifically, for a regularization parameter $\gamma_t>0$, we define $ \omega^\star(\hat{\theta}_t)$ as
\begin{equation}
\label{eq: empirical opt}
    \argmax_{\omega\in\Omega}\;
    \min_{i\neq i^\star(\hat{\theta}_t)}
    \,
    \frac{\big(\hat{\theta}_t^\top z_{i\,i^\star(\hat{\theta}_t)}\big)^2}
    {2\big\|z_{i\,i^\star(\hat{\theta}_t)}\big\|^2_{H(\hat{\theta}_t,\omega)^{-1}}}
    - \frac{\gamma_t\|\omega\|_2^2}{2}  .
\end{equation}
The regularization term renders the objective strictly concave in $\omega$, and hence guarantees a unique optimum at each iteration. This uniqueness is crucial for stabilizing the sampling process and preventing oscillations caused by the existence of multiple asymptotically optimal allocations. At the same time, we choose $\gamma_t \to 0$ so that the bias introduced by regularization vanishes asymptotically. With an appropriate choice of $\gamma_t$, one can show that the regularized solution converges to an optimal solution of the original problem, and more precisely, it selects the minimum-$\ell_2$-norm element from the set of optimal allocations.

In addition to tracking the optimal sampling proportions described above, the sampling rule must ensure sufficient exploration to guarantee consistency of the problem-instance estimator. We therefore adopt the following rule, following \citet{garivier2016optimal}. Let 
\begin{equation}
\label{eq: sampling rule}
\mathcal{C}_t :=
\begin{cases}
\displaystyle \argmin_{(i,j)\in U_t} N_{ij}(t),
& \text{if } U_t\neq \emptyset,
\\[2ex]
\displaystyle \argmax_{(i,j)\in \mathcal{S}}
\Bigl[t\,\omega^\star_{ij}(\hat{\mu}(t))-N_{ij}(t)\Bigr],
& \text{o.w}.
\end{cases}
\end{equation}
We then choose $(i_{t+1},j_{t+1})$ as any element of $\mathcal{C}_t$. Here, 
\[
U_t:=\{(i,j)\in\mathcal{A}_0: N_{ij}(t)<c^\prime\sqrt{t}\}
\] denotes the set of under-sampled pairs in $\mathcal{A}_0$, where $c^\prime>0$ is a fixed constant. This set serves as the exploration set and helps control the growth rate of the minimum eigenvalue $\Lambda_{\min}(t)$ in the structured policy space. In the unstructured policy space, consistency requires $\mathcal{A}_0=\mathcal{S}$, since every pair must be sampled infinitely often. In the structured policy space, it suffices to choose $\mathcal{A}_0\subset \mathcal{S}$ so that the associated feature vectors span $\mathbb{R}^d$. Thus, the first case in \eqref{eq: sampling rule} ensures sufficient exploration, while the second tracks the empirical optimal allocation and drives the realized sampling proportions toward the target allocation.

We next design stopping rules for $\delta$-PAC identification. In the unstructured policy space, we use the test statistic
\begin{equation*}
     Z(t) =  \inf_{\lambda\in \mathrm{Alt}(\hat{\mu}(t))}
    \sum_{(i,j)\in\mathcal{S}} N_{ij}(t)\, d\big(\hat{\mu}(i,j;t),\lambda(i,j)\big),
\end{equation*}
with stopping time
\begin{equation}
\label{eq: unstructed_stop}
    \tau:= \inf\left\{t\ge1: Z(t) >\rho(\delta,t)\right\},
\end{equation}
where 
\begin{equation*}
   \rho(\delta,t) = \log\left(\frac{Ct^{2}\log(1/\delta)^{2|\mathcal{S}|+1}}{\delta}\right).
\end{equation*}
for some appropriately chosen constant $C$. This test statistic and stopping rule were proposed by \citet{garivier2016optimal}.
Here, $Z(t)$ quantifies the evidence provided by the data against competing instances, and $\rho(\delta,t)$ sets the confidence level required for stopping. Larger values of $\rho(\delta,t)$ make the rule more conservative and thus reduce the error probability.

In the structured policy space, we consider the test statistic
\begin{equation*}
    Z(t) =  \min_{i\neq i^\star(\hat{\theta}_t)}
    \frac{1}{2}\,
    \frac{\big(\hat{\theta}_t^\top z_{i\,i^\star(\hat{\theta}_t)}\big)^2}{\big\|z_{i\,i^\star(\hat{\theta}_t)}\big\|^2_{H_t(\hat{\theta}_t)^{-1}}}
\end{equation*}
and define the stopping time
\begin{equation}
\label{eq: stop}
    \tau := \inf\left\{t\ge t_0: \Lambda_{\min}(t)\ge c\sqrt{t}, Z(t)>\beta(\delta,t)\right\},
\end{equation}
where 
\begin{equation*}
    \quad \beta(\delta,t) = \left(2(1+2LB)\big(\sqrt{m_0^{-1}\Psi_t(\delta)}
+
\sqrt{\lambda_t}B\big)\right)^2,
\end{equation*}
with
$$
m_0:=\min_{|u|\le BL}\sigma'(u)>0,
\quad
\Psi_t(\delta):=
\Big(1+\frac{\lambda_0}{\lambda(t)}\Big)
\left(
2\log\frac{1}{\delta}
+
\log \frac{\det(V_t)}{\lambda(t)^d}
+
d\log \Big(1+\frac{\lambda_0}{\lambda(t)}\Big)
+
d\log \frac{\lambda(t)}{\lambda_0}
\right).
$$
This stopping rule is motivated by the bound on the estimating error $g_t(\hat{\zeta}_t)-g_t(\theta_\star)$ established in the proof of Lemma~\ref{lemma: param_concentration}. The following lemma provides an explicit upper bound on the threshold $\beta(\delta,t)$, which facilitates both practical implementation and the $\delta$-PAC analysis in Lemma \ref{lemma: delta-guarantee}.

\begin{lemma}
\label{eq: beta_order}
Suppose there exist constants $c>0$ and $t_0\ge 1$ such that $\Lambda_{\min}(t)\ge c\sqrt{t}$ $\text{a.s.}$ for all $t\ge t_0$. Then there exists
constants \(c_1,c_2>0\) and \(\alpha>0\), independent of \(\delta\) and \(t\), such that
for all \(\delta\in(0,1)\) and all \(t\ge t_0\),
\[
\beta(\delta,t)\le c_1\log\Big(\frac{c_2 t^\alpha}{\delta}\Big).
\]   
\end{lemma}

We combine the sampling, stopping, and decision rules into the following adaptive simulation procedure, termed LLM-PO. Starting from an initial set of pairwise comparisons, the procedure repeatedly updates the problem parameter using the accumulated preference data, computes the empirically optimal sampling allocation based on the current estimate, and then collects a new comparison according to this adaptive rule. In this way, the procedure gradually concentrates simulation effort on the most informative policy pairs, while maintaining sufficient exploration to ensure reliable parameter estimation and correct policy identification.

\begin{experiment}{(LLM-PO)}
\label{experiment}
\item \textbf{Initialization:} Fix the risk level $\delta\in(0,1/2)$ and the initial sample size $n_0$. For each pair $(i,j)\in\mathcal{A}_0$, collect $n_0$ samples, set $N_{ij}(t)=n_0$ for all $(i,j)\in\mathcal{A}_0$, and initialize $t \gets n_0|\mathcal{A}_0|$.

\item \textbf{Parameter estimation:} Based on the data collected up to time $t$, estimate the problem parameter $\hat{\mu}(t)$ or $\hat{\theta}_t$, and determine the empirical best policy $i^\star(\hat{\mu}(t))$.

\item \textbf{Adaptive pair selection:}
Obtain the empirical optimal sampling proportion $\omega^\star(\hat{\mu}(t))$ by applying Corollary~\ref{corollary: optimal ratio} in the unstructured policy space or by solving \eqref{eq: empirical opt} in the structured policy space. Then, select the next policy pair $(i_{t+1},j_{t+1})$ according to the sampling rule in \eqref{eq: sampling rule}.

\item \textbf{Data collection:} Conduct the simulation experiment for the selected policy pair $(i_{t+1},j_{t+1})$ and observe the binary comparison outcome $D_{t+1}$.

\item \textbf{Stopping:} Terminate the procedure if the stopping rule in \eqref{eq: unstructed_stop} for the unstructured policy space, or that in \eqref{eq: stop} for the structured policy space, is satisfied. Otherwise, update $t \gets t+1$ and return to Step 2.
\end{experiment}

\subsection{Optimality Analysis}
In this subsection, we analyze the optimality of the LLM-PO experiment design. Lemma~\ref{lemma: unstructured-PAC-guarantee} establishes the $\delta$-PAC guarantee in the unstructured policy space. Its proof relies on a self-normalized deviation inequality for one-dimensional exponential families from \citet{magureanu2014lipschitz}.

\begin{lemma}
\label{lemma: unstructured-PAC-guarantee}
For any \(\delta\in(0,1)\), any adaptive simulation experiment equipped with the stopping rule \eqref{eq: unstructed_stop} in the unstructured policy space is \(\delta\)-PAC. That is,
\[
\mathbb{P}\left(\tau<\infty,\ \hat{i}_{\tau}\neq i^\star(\mu)\right)\le \delta.
\]
\end{lemma}

Theorem~\ref{thm: unstructured almost surely opt} establishes an almost-sure upper bound on the stopping time $\tau$, showing that it matches the fundamental data requirements $\mathcal{T}^\star(\mu)$ up to a constant factor and is therefore asymptotically optimal. This result is proved under the assumption that the optimal sampling proportion is unique. This condition holds, for example, when all pairwise means $\mu(i,j)$, $(i,j)\in \mathcal{S}$, are distinct, so that each suboptimal policy has a unique most informative opponent. When the optimal allocation is not unique, the same approach used for the structured policy space can be applied. We can introduce an $\ell_2$-regularized optimization problem to select a unique empirical optimizer. The choice of the regularization decay, the convergence of the realized sampling proportions, and the corresponding optimality analysis are similar to those in the structured setting (as in Theorem~\ref{prop: ratio_convergence} and Theorem~\ref{thm: almost_sure_optimality}).

\begin{theorem}
\label{thm: unstructured almost surely opt}
Assume that the optimal sampling proportion $\omega^\star(\mu)$ is unique. Then, under the sampling rule \eqref{eq: sampling rule} and the stopping rule \eqref{eq: unstructed_stop}, the stopping time \(\tau\) of the adaptive simulation experiment in the unstructured policy space satisfies
\[
\mathbb{P}\left(
\limsup_{\delta\to 0}
\frac{\tau}{\log(1/\delta)}
\lesssim
\mathcal T^\star(\mu)
\right)=1,
\]
where \(\lesssim\) denotes inequality up to a multiplicative constant independent of \(\delta\) and \(t\).
\end{theorem}

We next establish several analytic properties of the optimization problem in \eqref{eq: empirical opt}. For a given parameter $\theta$, define
\begin{equation*}
\Psi(\theta,\omega)
:=
\min_{i\neq i^\star(\theta)}
\frac{1}{2}\,
\frac{\big(\theta^\top z_{i\,i^\star(\theta)}\big)^2}
{\big\|z_{i\,i^\star(\theta)}\big\|^2_{H(\theta,\omega)^{-1}}},
\end{equation*}
and let $
C^\star(\theta):=\argmax_{\omega\in\Omega}\Psi(\theta,\omega)
$ 
denote the corresponding set of optimal solutions. Lemma~\ref{lemma: opt property} shows that $\Psi(\theta,\omega)$ is locally continuous and that $C^\star(\theta)$ is nonempty, compact, and convex.

\begin{lemma}
\label{lemma: opt property}
Assume that $H(\theta,\omega)\succ 0$ for all $(\theta,\omega)$ under consideration.
Then $\Psi(\theta,\omega)$ is continuous in $(\theta,\omega)$ on a neighborhood of $\theta_\star\times\Omega$. Moreover, there exist constants $r_0>0$ and $L_{\Psi}>0$ such that 
$$
\sup_{\omega\in\Omega}
\big|
\Psi(\theta,\omega)-\Psi(\theta_\star,\omega)
\big|
\le
L_{\Psi}\|\theta-\theta_\star\|_2,
\qquad
\forall \theta:\ \|\theta-\theta_\star\|_2\le r_0.
$$
In addition, for each fixed $\theta$, the optimal solution set $
C^\star(\theta)
$
is non-empty, compact, and convex.
\end{lemma}

Theorem~\ref{prop: ratio_convergence} shows that, with appropriate decay rates for the regularization parameters $\lambda_t$ and $\gamma_t$, the realized sampling proportions $(N_{ij}(t)/t)_{(i,j)\in\mathcal{S}}$ converge to a unique optimal allocation $\omega^{\dagger}\in C^\star(\theta_\star)$, namely, the minimum-$\ell_2$-norm element of the optimal solution set $C^\star(\theta_\star)$.
\begin{theorem}
\label{prop: ratio_convergence}
Under the policy pair sampling rule defined in \eqref{eq: sampling rule}, and choosing the regularization parameters $\lambda_t=t^{-1/2}$ and $\gamma_t=t^{-1/8}$,  it holds that
\[
\mathbb{P}\left(
\forall (i,j)\in\mathcal S,\ 
\lim_{t\to\infty}\frac{N_{ij}(t)}{t}
=
\omega^\dagger_{ij}
\right)=1,
\]
where $\omega^\dagger
:=
\argmin_{\omega\in C^\star(\theta_\star)} \|\omega\|^2_2 $. 
\end{theorem}

The main technical novelty in the proof of Theorem \ref{prop: ratio_convergence} lies in balancing the convergence of the parameter estimator with the decay of the regularization. Since the oracle allocation set $C^\star(\theta_\star)$ may contain multiple optimal solutions, we introduce a vanishing quadratic regularizer to select a unique representative, namely the minimum-norm solution $\omega^{\dagger}$. Our analysis addresses this by choosing the decay rate of $\gamma_t$ so that the estimation-induced perturbation
\begin{equation*}
    \delta_t
:=
\sup_{\omega\in\Omega}
\big|
\Psi(\hat\theta_t,\omega)-\Psi(\theta_\star,\omega)
\big|
\end{equation*}
satisfies $\delta_t = o(\gamma_t)$. This scale separation ensures that the statistical error becomes asymptotically negligible compared with the regularization, allowing the regularizer to consistently act as a tie-breaking mechanism. As a result, the regularized optimizer $\omega^\star(\hat{\theta}_t)$ is shown to converge to the unique minimum-norm element $\omega^{\dagger}$ of  $C^\star(\theta_\star)$. This interplay between estimation accuracy and regularization decay is the central ingredient of the proof.

Lemma \ref{lemma: delta-guarantee} establishes the $\delta$-PAC guarantee of the stopping rule \eqref{eq: stop} over the structured policy space. 
\begin{lemma}
\label{lemma: delta-guarantee}
For any \(\delta\in(0,1/2)\), any adaptive simulation experiment equipped with the stopping rule \eqref{eq: stop} in the structured policy space is \(\delta\)-PAC. That is,
\[
\mathbb{P}\left(\tau<\infty,\ \hat{i}_{\tau}\neq i^\star(\mu)\right)\le \delta.
\]
\end{lemma}

Theorem~\ref{thm: almost_sure_optimality} establishes an almost-sure upper bound on the stopping time $\tau$, showing that in the structured policy space, it matches the surrogate complexity $\mathcal{U}^\star(\mu)$ up to a constant factor. As the experiment collects more pairwise comparison data, the procedure estimates the underlying preference parameter $\theta_{\star}$ more accurately and gradually concentrates the evaluation budget on the most informative policy comparisons. Consequently, the statistical evidence for policy optimization accumulates at an asymptotically linear rate, whereas the stopping threshold grows only logarithmically. This implies that the procedure eventually stops almost surely, with sample complexity on the order of $\mathcal{U}^\star(\mu)\log(1/\delta)$.
\begin{theorem}
\label{thm: almost_sure_optimality}
Under the sampling rule \eqref{eq: sampling rule} and the stopping rule \eqref{eq: stop}, the stopping time \(\tau\) of the adaptive simulation experiment in the structured policy space satisfies
\label{thm: almost surely opt}
\begin{equation*}
    \mathbb{P} \left(
\limsup_{\delta\rightarrow 0}
\frac{\tau}{\log(1/\delta)}
\lesssim
\mathcal U^\star(\mu)
\right)=1.
\end{equation*}
\end{theorem}

\section{Numerical Experiments}
\label{sec: experiments}
In this section, we evaluate the performance of LLM-PO through extensive synthetic and real-world experiments. In both unstructured and structured policy spaces, we compare LLM-PO with several benchmark methods. In particular, Thompson Sampling and UCB are state-of-the-art methods from the literature that we adapt to our problem setting.
\begin{itemize}
    \item \textbf{RoundRobin:} employs a non-adaptive strategy that allocates comparisons uniformly across all candidate pairs in a deterministic order, without adjusting the sampling scheme based on past observations.
    \item \textbf{RandomPair:} uses a non-adaptive strategy that selects a candidate pair uniformly at random at each round.
    \item \textbf{EpsGreedy:} uses an $\epsilon$-greedy strategy that explores by sampling a random pair with probability $\epsilon$, and otherwise compares the current best policy with its empirically strongest opponent.
    \item \textbf{Thompson Sampling:} uses a practical Thompson-sampling strategy that samples pairwise preference probabilities from Beta posteriors to select a candidate policy and its opponent; this benchmark is inspired by Double Thompson Sampling \citep{wu2016double}.
    \item \textbf{RUCB:} uses an RUCB-style strategy that builds upper confidence bounds for pairwise preference probabilities, selects a candidate from the plausible winner set, and compares it against the opponent most likely to beat it; this benchmark is inspired by RUCB \citep{zoghi2014relative}.
\end{itemize}

\subsection{Synthetic Experiments}
We consider an unstructured policy space with $16$ candidate policies, which induces $|\mathcal{S}|=120$ admissible comparison pairs. Pairwise comparison outcomes are generated from a preference matrix $\mu$, where the comparison outcome between policies $i$ and $j$ follows a Bernoulli distribution with mean $\mu(i,j)$. To construct a challenging synthetic instance, we first assign each policy $i\in\{0,\ldots,15\}$ a latent score
\begin{equation}
    r_i = 0.55 - \frac{1.90i}{15} + \varepsilon_i, \qquad \varepsilon_i \sim \mathcal{N}(0,0.005^2),
\end{equation}
and then define the pairwise preference probabilities as
\begin{equation*}
    \mu(i,j) = \frac{1}{1+\exp(-(r_i-r_j))},\quad \forall i\neq j
\end{equation*}
For completeness, we set $\mu(j,i) = 1-\mu(i,j)$ and $\mu(i,i)=1/2$. This construction induces a clear global ranking over the $16$ candidate policies, while keeping neighboring policies close in latent score and hence close in pairwise preference probability. As a result, several suboptimal policies remain competitive with the best one, so identifying the optimal policy with fixed confidence requires a substantial number of noisy pairwise comparisons across the full set of $|\mathcal{S}|=120$ possible pairs. This makes the instance well-suited for evaluating the effectiveness of fixed-confidence adaptive experiment design in large unstructured policy spaces.

For each method, we perform $200$ independent replications to evaluate empirical performance at the risk level $\delta=0.05$. To control computational cost, the simulation budget is capped at $30000$ pairwise comparisons. Any method that does not satisfy the stopping criterion within this budget is automatically terminated at the maximum budget. All performance metrics are estimated by sample averages across replications, and 95\% confidence intervals are reported accordingly. 

Figure \ref{fig:syn-pcs} shows the empirical probability of correct selection (PCS), that is, the frequency of correctly selecting the optimal policy, as a function of the simulation budget. The proposed LLM-PO experiment consistently achieves the highest PCS over the entire budget range. It approaches perfect identification rapidly, reaching a PCS close to $1$ after only a few thousand comparisons. In contrast, Thompson Sampling and RUCB improve more gradually, although both eventually achieve high accuracy as the budget increases. The simpler exploration strategies, including RoundRobin, RandomPair, and EpsGreedy, perform substantially worse and require significantly more samples to attain comparable accuracy.

\begin{figure}
    \centering    \includegraphics[width=0.6\linewidth]{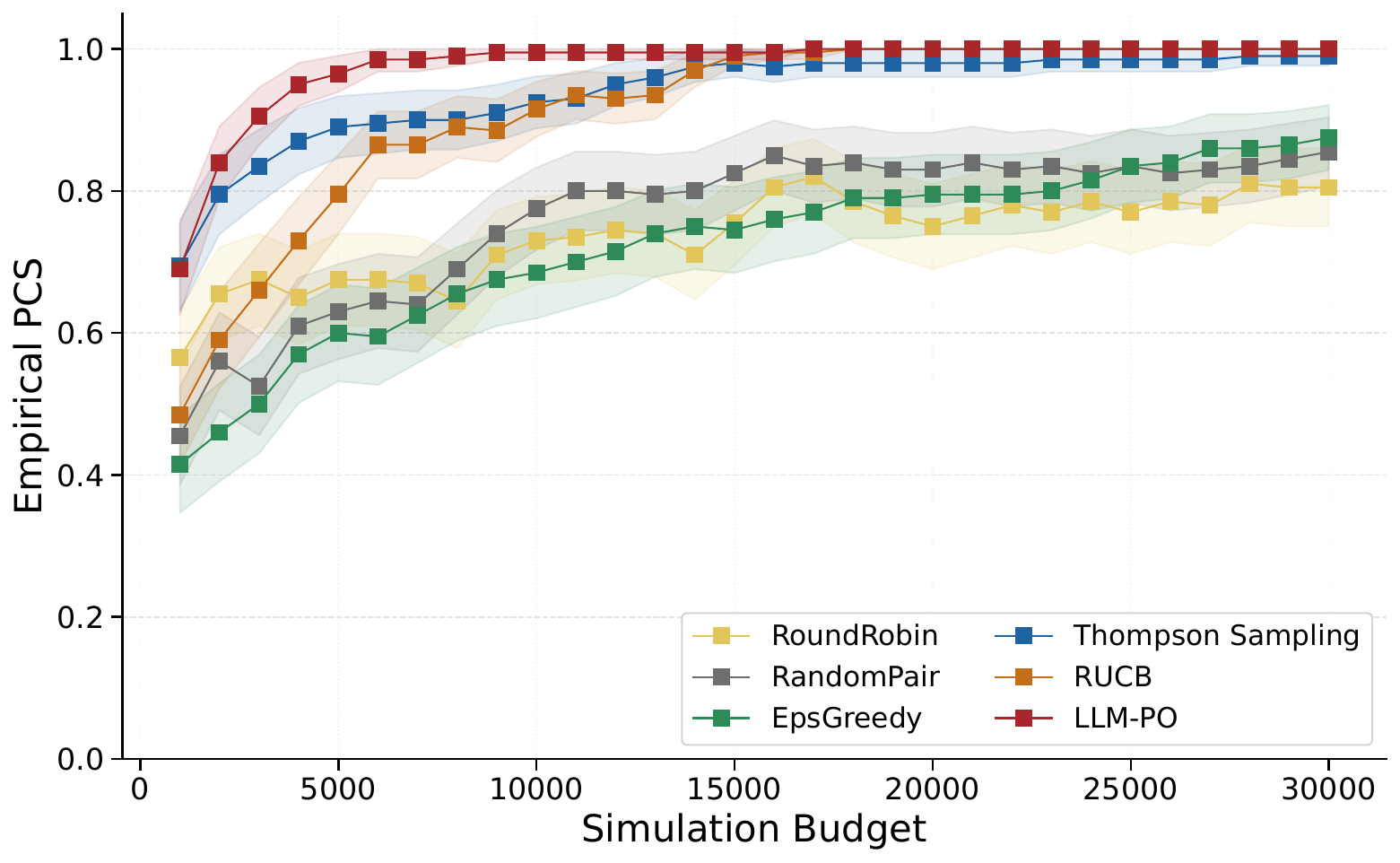}
    \caption{Empirical Probability of Correct Selection ($\delta=0.05$) with 95\% Confidence Intervals}
    \label{fig:syn-pcs}
\end{figure}

In implementation, we adopt the heuristic stopping threshold $\rho(\delta,t)= 2\log((\log t+1)/\delta)$, which has been adopted in prior work on best arm identification, such as \citet{garivier2016optimal,wang2021fast}. Although this choice is not directly justified by theory, it is computationally simple to implement and, importantly, provides reliable empirical performance together with a valid fixed-confidence guarantee.
Because the three simpler exploration strategies fail to stop within the maximum budget of $30000$ comparisons, we report average stopping times only for the remaining three experiment designs. Figure \ref{fig:syn-stop} shows that the proposed LLM-PO experiment stops substantially earlier than Thompson Sampling and RUCB, demonstrating superior sample efficiency. In addition, the stopping criterion is both effective and somewhat conservative. Specifically, at the mean stopping time of each experiment, the corresponding empirical PCS is significantly higher than the required $95\%$ confidence threshold. This indicates that the stopping rule succeeds in delivering the desired fixed-confidence guarantee, while tending to stop only after sufficient evidence has been accumulated.

\begin{figure}
    \centering
    \includegraphics[width=0.5\linewidth]{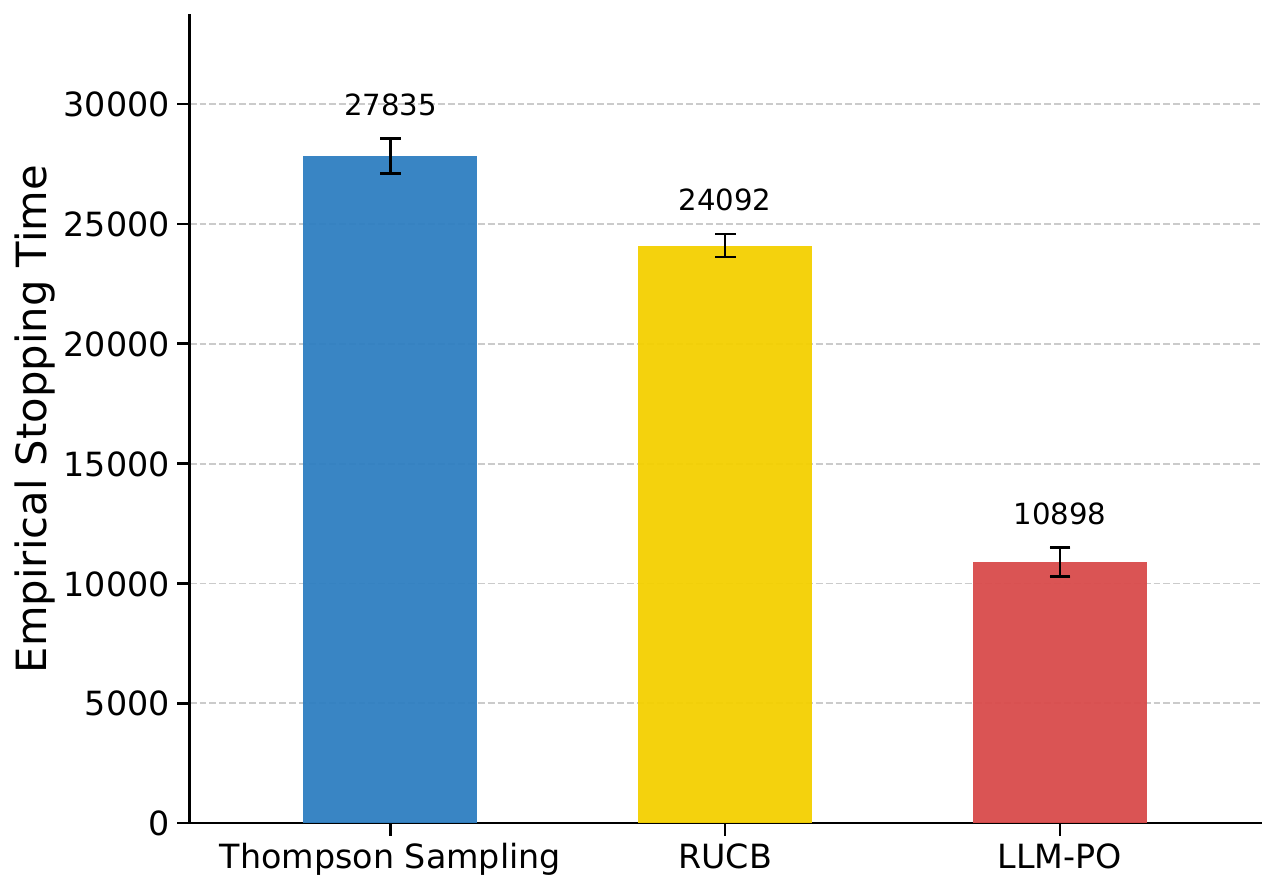}
    \caption{Empirical Stopping Time ($\delta=0.05$) with 95\% Confidence Intervals}
    \label{fig:syn-stop}
\end{figure}

We consider a structured policy space with $K=32$ candidate policies, each associated with a $d=6$-dimensional feature vector. This induces $|\mathcal{S}|=496$ admissible comparison pairs. Pairwise comparison outcomes are generated from a linear Bradley-Terry model: for each pair of policies $i\neq j$, the comparison outcome follows a Bernoulli distribution with mean $\mu(i,j)$. To construct the synthetic instance, we first generate a latent parameter vector $\theta_\star\in\mathbb{R}^d$ from a standard Gaussian distribution and normalize it to have unit $\ell_2$-norm. We then define a sequence of base scores evenly spaced from  $0.3$ to $-0.3$, and for each policy $i\in\{0,\ldots,31\}$, generate its feature vector as
\begin{equation*}
    x_i = b_i\theta_\star + \xi_i,
\end{equation*}
where $b_i\in \mathbb{R}$ is the corresponding base score and $\xi_i\sim \mathcal{N}(0, 0.25^2I_d)$ is an independent Gaussian perturbation. After generation, all policies are reordered according to their true linear scores $x_i^\top \theta_\star$, so that the best policy is the one with the largest true score. The pairwise preference probabilities are then defined by
$$
\mu(i,j)=\sigma\big(\theta_\star^\top(x_i-x_j)\big)
=\frac{1}{1+\exp\big(-\theta_\star^\top(x_i-x_j)\big)},\qquad \forall i\neq j.
$$
For completeness, we set $\mu(j,i) = 1-\mu(i,j)$ and $\mu(i,i)=1/2$. This construction yields a structured preference model in which pairwise comparisons are governed by a shared low-dimensional parameter $\theta_\star$. 

We extend the benchmark to structured policy spaces. Specifically, all methods use the collected preference data to estimate the unknown parameter $\theta_\star$ through a regularized logistic maximum likelihood estimator. The resulting estimate is then used to identify the empirically best policy and to guide the adaptive sampling process.

Figure \ref{fig:structured pcs} shows that LLM-PO consistently outperforms all benchmark methods in terms of empirical PCS in the structured policy space. In particular, LLM-PO exhibits a substantially faster increase in PCS as the simulation budget grows and attains near-perfect selection accuracy with fewer comparisons, demonstrating its superior sample efficiency for policy optimization. Owing to the availability of structured information, the benchmark methods also perform reasonably well, even though the number of policy pairs is much larger than in the unstructured policy space.

\begin{figure}
    \centering
    \includegraphics[width=0.6\linewidth]{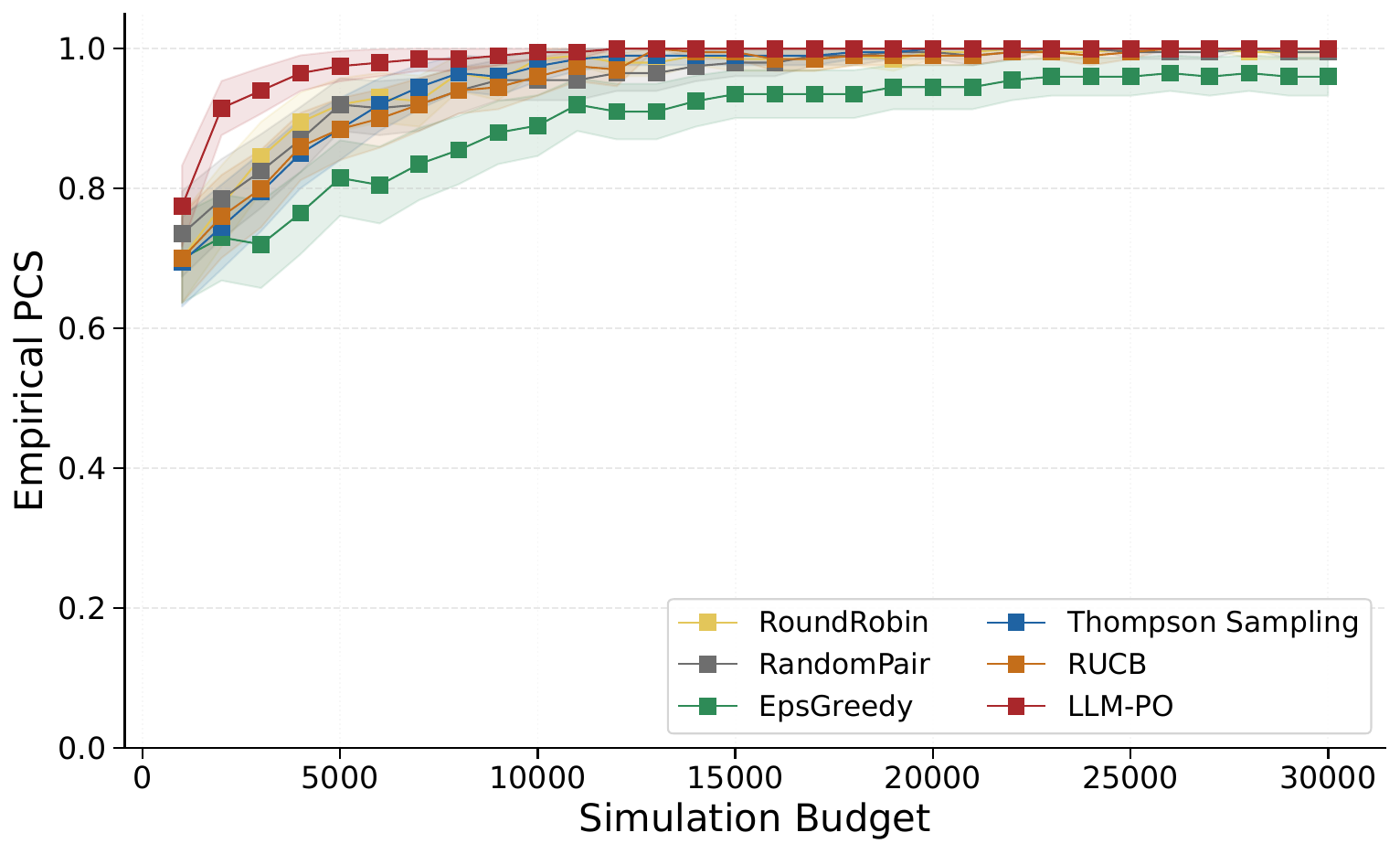}
    \caption{Empirical Probability of Correct Selection ($\delta=0.05$) with 95\% Confidence Intervals}
    \label{fig:structured pcs}
\end{figure}

Figure \ref{fig:structured stop} presents the empirical stopping time of different methods. The results show that LLM-PO stops substantially earlier than all benchmark methods, requiring only about $6542$ comparisons on average, whereas the competing methods require roughly $15000$ to $23000$ comparisons. This gap is considerable and indicates that LLM-PO can reach a reliable decision much more quickly. Combined with the PCS results in Figure \ref{fig:structured pcs}, these findings provide strong evidence for the effectiveness of LLM-PO. it not only achieves higher selection accuracy, but also does so with significantly fewer samples. Therefore, LLM-PO is more sample-efficient and better suited for policy optimization in the structured policy space.
\begin{figure}
    \centering
    \includegraphics[width=0.6\linewidth]{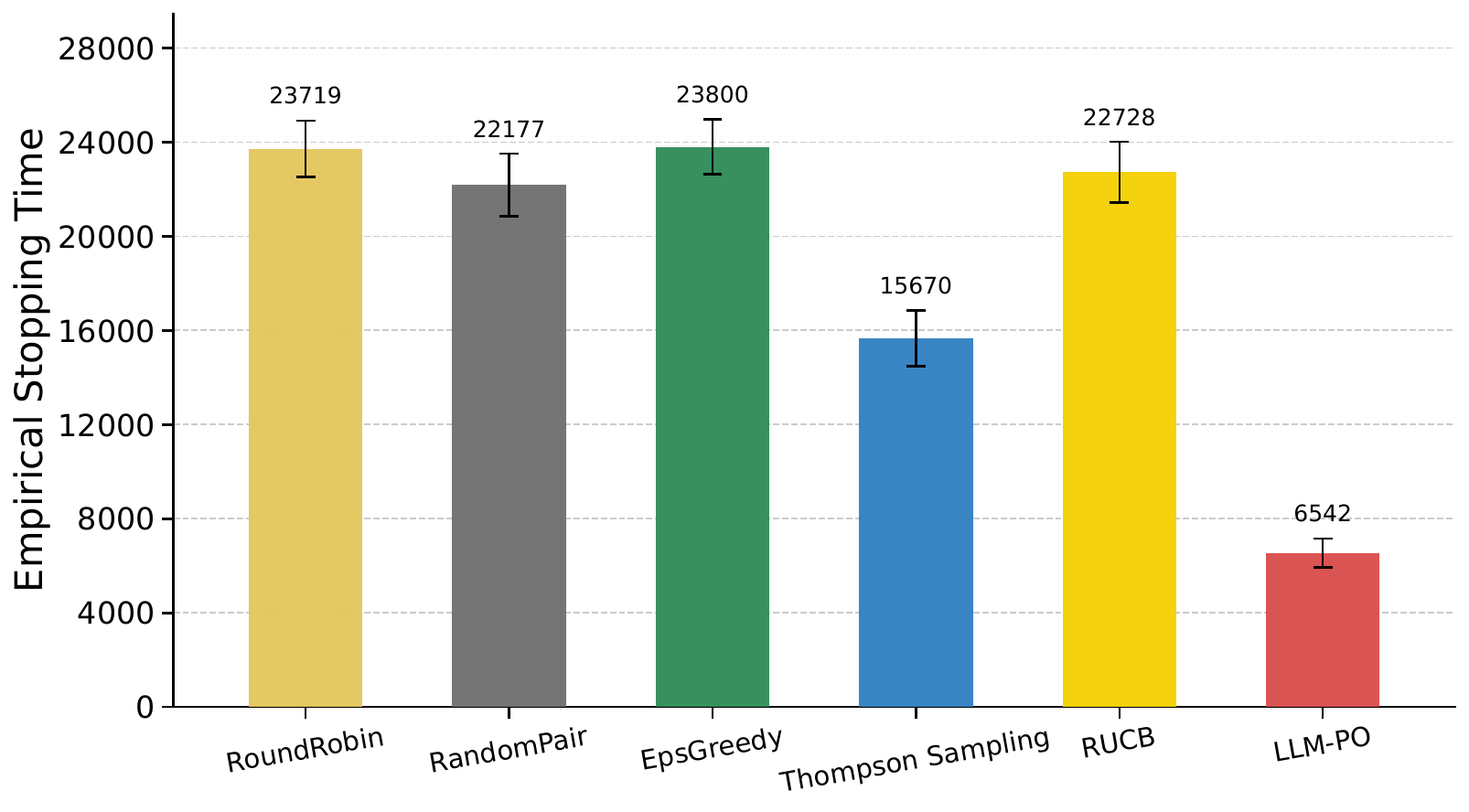}
    \caption{Empirical Stopping Time ($\delta=0.05$) with 95\% Confidence Intervals}
    \label{fig:structured stop}
\end{figure}

\subsection{Real Experiments}
In this subsection, we evaluate the performance of LLM-PO through extensive real-world experiments on policy optimization for LLMs. The evaluation tasks are drawn from two standard benchmark datasets, Instruction Induction \citep{honovich2023instruction} and BIG-bench \citep{suzgun2023challenging}. We use Llama-3:8B as the response-generating model.

We consider four tasks that evaluate different core abilities of language models. More details are provided in Section \ref{sec: real details} of the Electronic Companion:
\begin{itemize}
    \item \textbf{Object Counting:} requires the model to determine the number of specified objects mentioned in a given scene description.
    \item \textbf{Word Unscrambling:} requires the model to recover the correct word from a set of scrambled letters.
    \item \textbf{Second Word Letter:} requires the model to identify a specified letter in the second word of a given text sequence.
    \item \textbf{Sum:} requires the model to compute the result of an arithmetic addition problem.
\end{itemize}

Given the nature of the evaluated tasks, we define each policy as a combination of a system prompt and a reasoning strategy. The system prompt controls the model’s response style, while the reasoning strategy specifies how the model solves the task. In our experiments, we consider two system prompts and three reasoning strategies for each task, which gives $6$ candidate policies and $15$ policy pairs.
To evaluate different methods, we first construct a ground-truth oracle policy for each task. Specifically, each policy pair is compared via Monte Carlo simulation with $30$ pairwise comparisons, and the procedure is repeated over $10$ independent replications to estimate the preference matrix $\mu$. The oracle optimal policy is then defined as the policy, or set of policies, that significantly outperforms the others under the empirical preference matrix $\mu$.
This exhaustive oracle evaluation is used only for benchmarking and is not needed in practical applications.

Since reference answers are available for these tasks, we primarily adopt a rule-based judge in implementation. The judge first checks the correctness of the two responses and always prefers the correct one. If both responses are correct, it favors either the one with clearer reasoning steps or the one with a more concise final answer. If both responses are incorrect, it selects between the two policies uniformly at random. When the answer cannot be reliably verified by rules, we instead employ an LLM (Qwen2.5-7B) as a Judge with a strict prompt-based scoring criterion.

Our experiments were conducted on a Linux server equipped with eight NVIDIA A100 PCIe 80GB GPUs. Since the sum task is relatively easy for all
methods, we set its maximum simulation budget to $60$ comparisons. To control computational cost, we set the maximum simulation budget for the remaining tasks to $90$ comparisons. Due to variation in task difficulty, each task was evaluated with a different number of independent replications, and the corresponding $90\%$ confidence intervals are reported.

Figure \ref{fig:real result} shows that the proposed method LLM-PO consistently outperforms or remains highly competitive with all benchmark methods in the four tasks. On the more challenging Object Counting task, LLM-PO demonstrates a substantial improvement in empirical PCS and clearly outperforms all baselines by a wide margin. On World Unscrambling, it again achieves the best performance, with a noticeably higher PCS than the benchmark methods. On Second Word Letter, LLM-PO attains perfect selection accuracy, while the competing methods remain slightly below this level. On the relatively easier Sum task, where several methods already achieve near-perfect performance, LLM-PO remains fully competitive with the best baselines and still achieves perfect PCS. Overall, these results demonstrate that our adaptive simulation framework is not only robust across tasks with varying levels of difficulty but also consistently more sample-efficient and reliable than existing benchmark methods. They show that simulation-based policy optimization can identify better policies for LLMs and translate limited simulation budgets into real performance gains on downstream tasks.

\begin{figure}[htbp]
    \centering
    \begin{subfigure}[b]{0.45\textwidth}
        \centering
        \includegraphics[width=\textwidth]{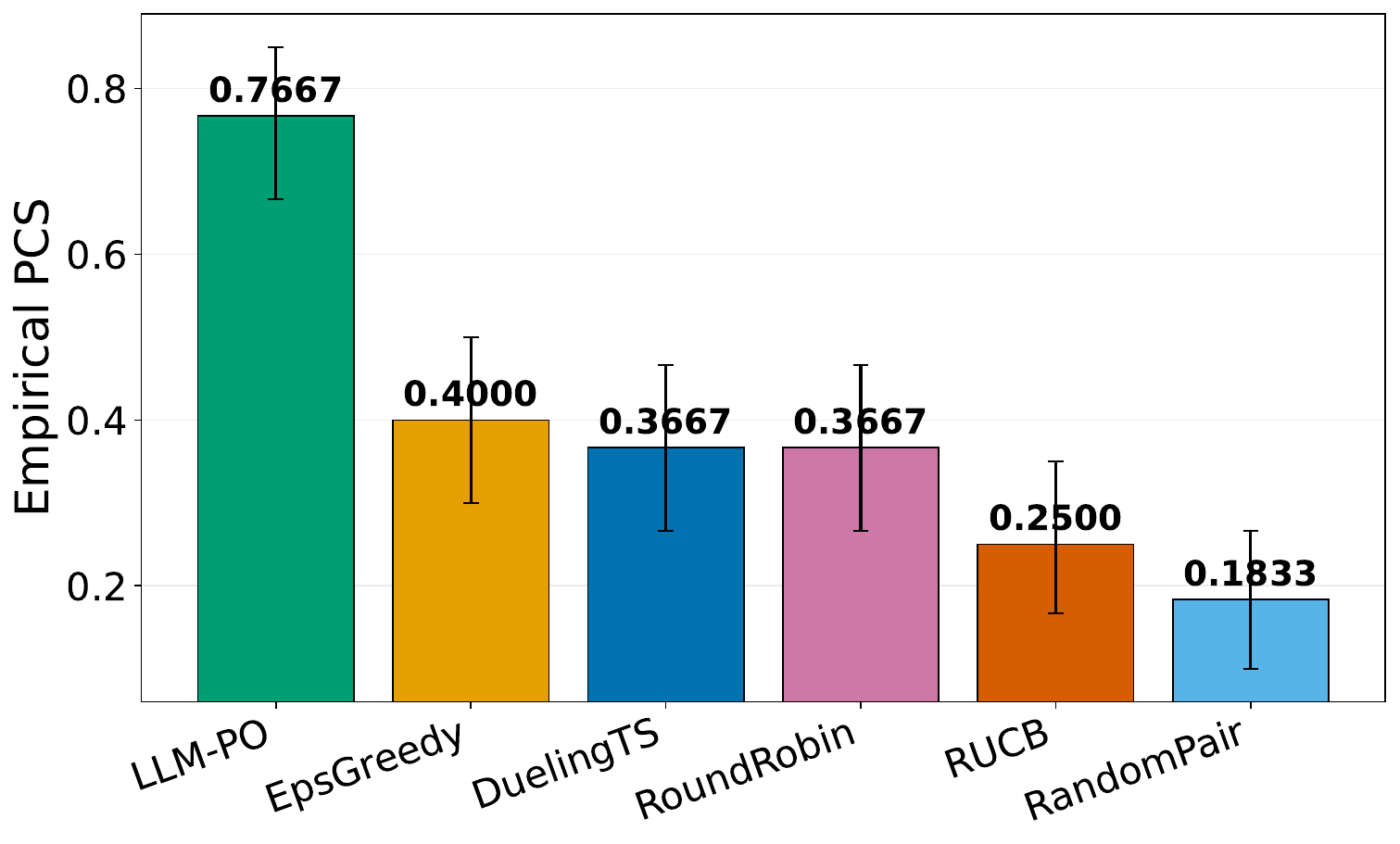}
        \caption{Object Counting (60 replications)}
        \label{fig:sub1}
    \end{subfigure}
    \hfill
    \begin{subfigure}[b]{0.45\textwidth}
        \centering
        \includegraphics[width=\textwidth]{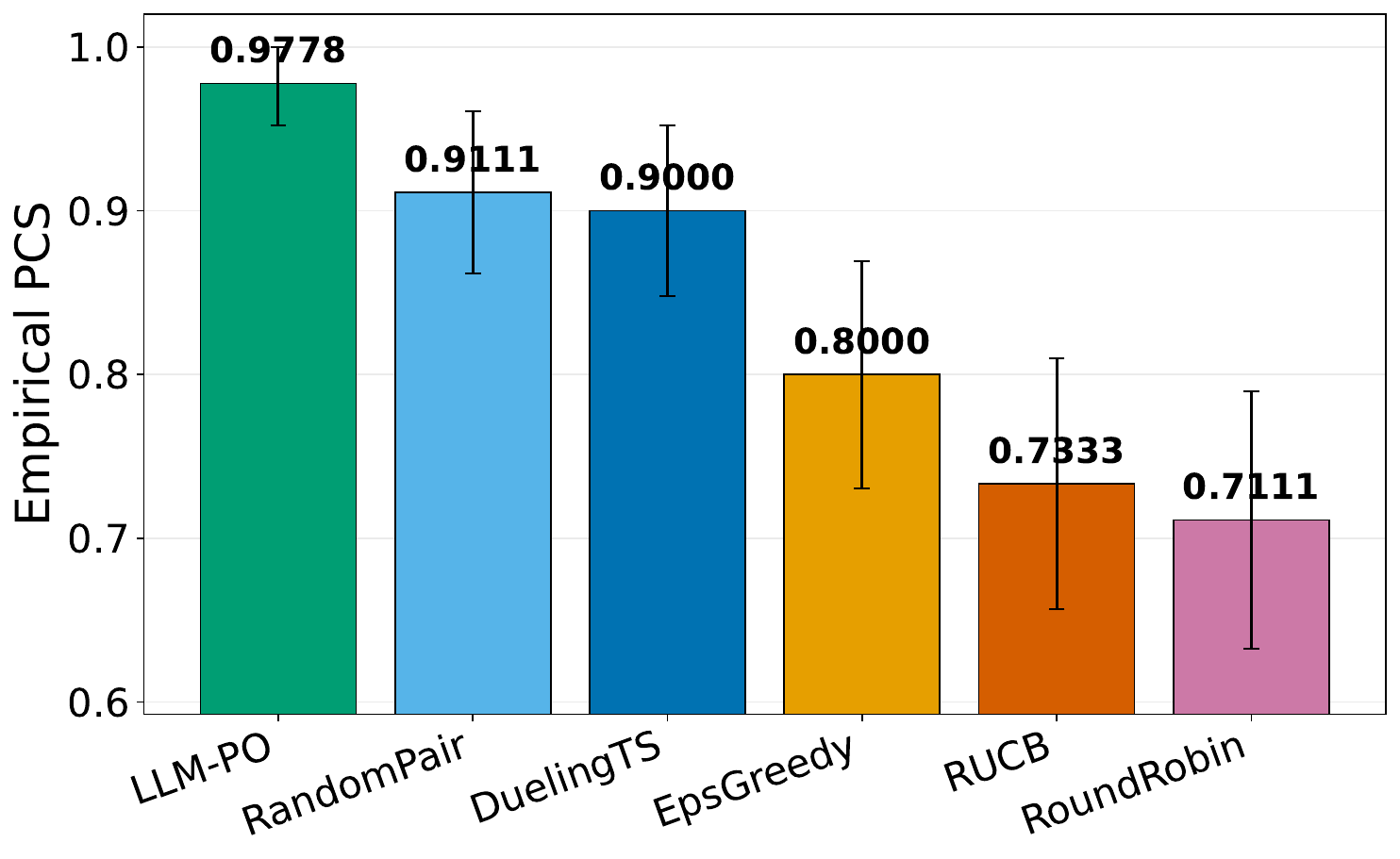}
        \caption{World Unscrambling (90 replications)}
        \label{fig:sub2}
    \end{subfigure}
    
    \vspace{0.5cm}
    
    \begin{subfigure}[b]{0.45\textwidth}
        \centering
        \includegraphics[width=\textwidth]{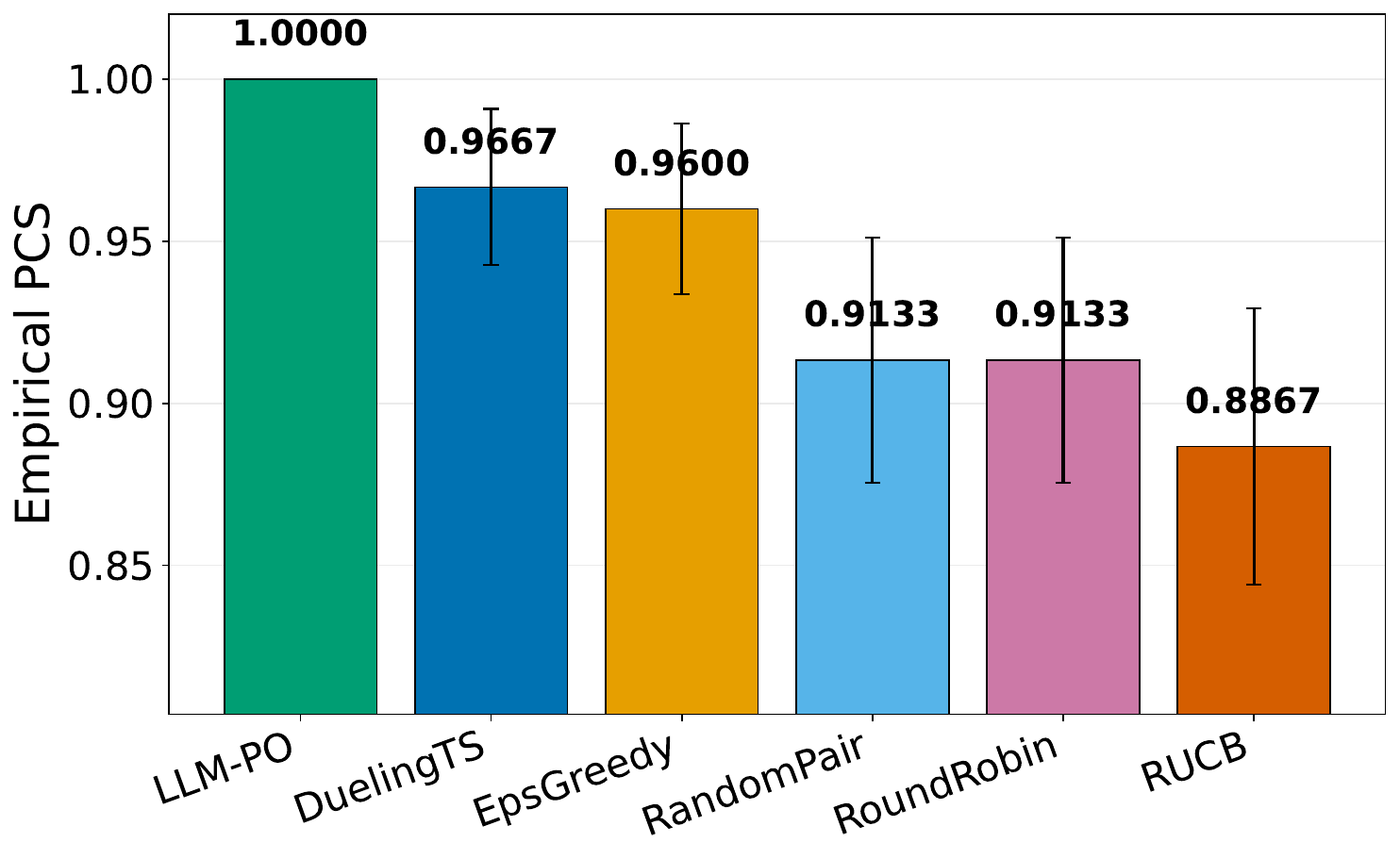}
        \caption{Second Word Letter (150 replications)}
        \label{}
    \end{subfigure}
    \hfill
    \begin{subfigure}[b]{0.45\textwidth}
        \centering
        \includegraphics[width=\textwidth]{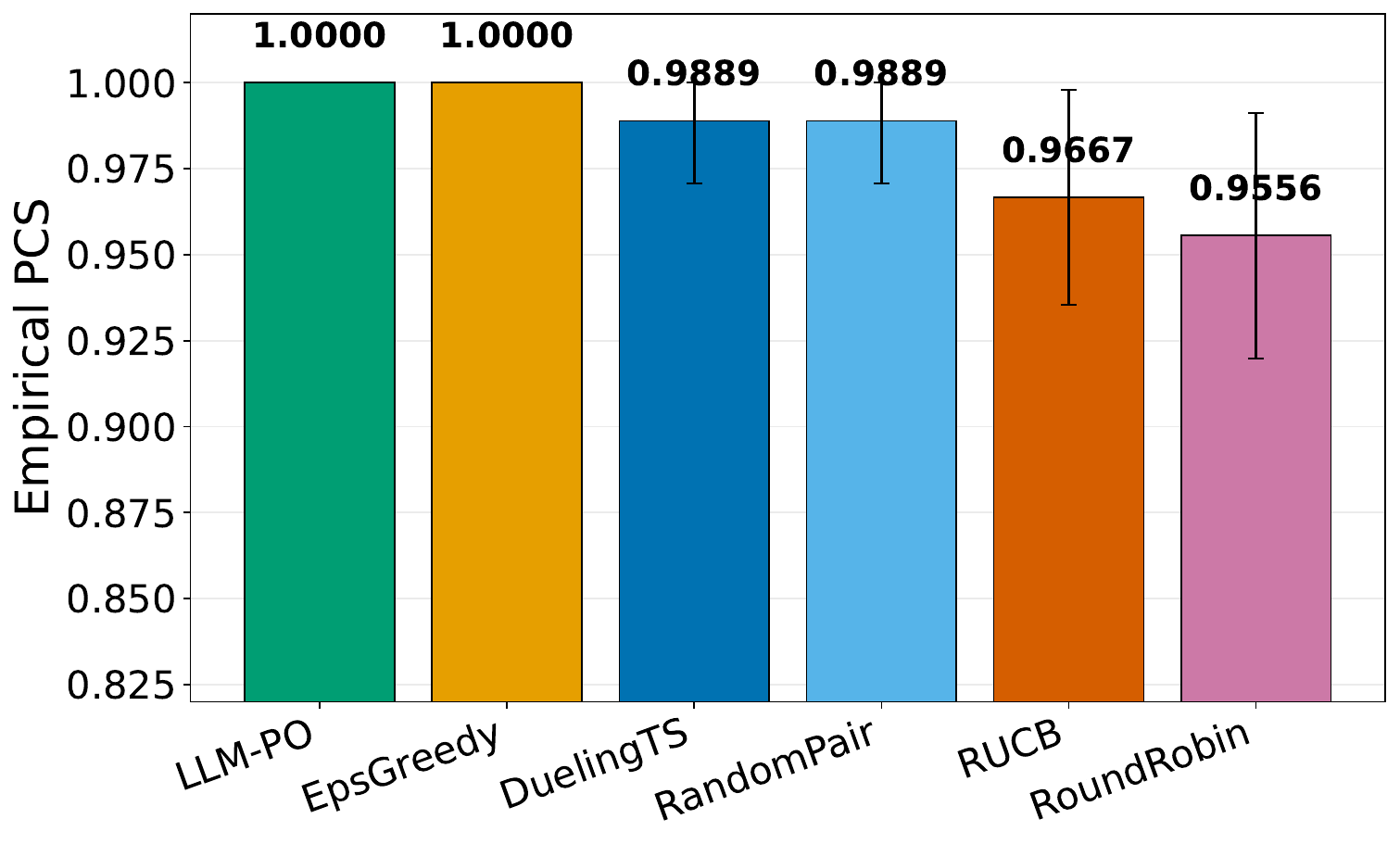}
        \caption{Sum (90 replications)}
        \label{fig:sub4}
    \end{subfigure}
    
    \caption{Empirical Probability of Correction Selection with 90\% Confidence Intervals}
    \label{fig:real result}
\end{figure}

\section{Conclusion}
\label{sec: conclusion}
This paper studies policy optimization for LLMs deployed in operational management settings. We propose a pairwise-comparison-based adaptive simulation experiment framework to identify the optimal policy from a finite set of candidates with a high probability. The framework covers both an unstructured policy space without parametric assumptions and a structured policy space under the Bradley-Terry preference model. For both settings, we characterize the fundamental data requirements and develop methods for computing the optimal sampling proportions. Building on these results, we design a procedure termed LLM-PO and show that it attains the desired statistical guarantee while asymptotically matching the fundamental data requirements. Numerical experiments further demonstrate the effectiveness of the proposed method. 

This work points to an important direction for studying LLMs in operations management. As foundation models become more capable and more widely available, the key challenge is often not only how to build better models, but also how to deploy existing models in a controllable and effective way. In this context, policy choices such as prompts, guardrails, and sampling parameters are not just technical details. They directly shape the response quality and user experience. Our work suggests that these choices can be improved through principled adaptive experimentation. We hope this perspective can encourage future research on how to systematically evaluate, manage, and improve LLM-based systems in real operational environments.

\footnotesize
\bibliographystyle{abbrvnat}
\bibliography{reference}

\ECSwitch


\ECHead{Proofs of Statements}
\section{Proof of Theorem \ref{thm: lower bound}}
\begin{lemma}[Lemma~1 in \citet{kaufmann2016complexity}]
\label{lem:kaufmann16-lem1}
Let $\mu\in \mathcal{H}$ and $\lambda\in \mathrm{Alt}(\mu)$ denote two instances of the LLM policy optimization problem. For any experiment design, every stopping time $\tau$ and event $\mathcal{E}$, it holds that
\begin{equation*}
    \sum_{(i,j)\in\mathcal{S}}\mathbb{E}_{\mu}[N_{ij}(\tau)] d(\mu(i,j),\lambda(i,j))\ge \text{kl}(\mathbb{P}_{\mu}(\mathcal{E}),\mathbb{P}_{\lambda}(\mathcal{E})),
\end{equation*}
where $N_{ij}(\tau)$ is the number of samples from policy pair $(i,j)$ up to time step $\tau$, and $\text{kl}(x,y)$ is defined as 
$$\text{kl}(x,y) := x\log\left(\frac{x}{y}\right) + (1-x)\log\left(\frac{1-x}{1-y}\right).$$
\end{lemma}

Fix $\delta\in(0,1/2)$ and an instance $\mu\in\mathcal{H}$. Consider any $\delta$-PAC adaptive simulation experiment with stopping time $\tau$ and final decision $\hat{i}_\tau$.
By the $\delta$-PAC requirement, for all $\mu\in\mathcal{H}$,
\begin{equation}
\label{eq:pac_mu}
\mathbb{P}_\mu(\hat{i}_\tau = i^\star(\mu)) \ge 1-\delta,
\qquad\text{equivalently}\qquad
\mathbb{P}_\mu(\hat{i}_\tau \neq i^\star(\mu)) \le \delta .
\end{equation}

Moreover, for any alternative instance $\lambda\in\mathrm{Alt}(\mu)$ (so that $i^\star(\lambda)\neq i^\star(\mu)$), the same decision rule must satisfy
\begin{equation}
\label{eq:pac_lambda}
\mathbb{P}_\lambda(\hat{i}_\tau = i^\star(\lambda)) \ge 1-\delta
\quad\Rightarrow\quad
\mathbb{P}_\lambda(\hat{i}_\tau \neq i^\star(\mu)) \ge 1-\delta .
\end{equation}

Define the error event
\[
\mathcal{E}:=\{\hat{i}_\tau \neq i^\star(\mu)\}.
\]
Then \eqref{eq:pac_mu}--\eqref{eq:pac_lambda} give
\[
\mathbb{P}_\mu(\mathcal{E}) \le \delta,
\qquad
\mathbb{P}_\lambda(\mathcal{E}) \ge 1-\delta
\qquad \forall \lambda\in\mathrm{Alt}(\mu).
\]

Applying Lemma \ref{lem:kaufmann16-lem1} yields: for every $\lambda\in\mathrm{Alt}(\mu)$,
\begin{equation}
\label{eq:transport_pair}
\sum_{(i,j)\in\mathcal{S}}\mathbb{E}_\mu\big[N_{ij}(\tau)\big]\; d\big(\mu(i,j),\lambda(i,j)\big)
\;\ge\;
\text{kl}\big(\mathbb{P}_\mu(\mathcal{E}),\mathbb{P}_\lambda(\mathcal{E})\big)
\;\ge\;
\text{kl}(\delta,1-\delta),
\end{equation}
where the last inequality follows since 
$\delta<1/2$, for $x<y$ the binary relative entropy $\text{kl}(x,y)$ is decreasing in $x$ and increasing in $y$.

Now combine the inequalities over all alternatives $\lambda\in\mathrm{Alt}(\mu)$:
\begin{align}
\text{kl}(\delta,1-\delta)
&\le
\inf_{\lambda\in\mathrm{Alt}(\mu)}
\sum_{(i,j)\in\mathcal{S}}\mathbb{E}_\mu\big[N_{ij}(\tau)\big]\; d\big(\mu(i,j),\lambda(i,j)\big) \nonumber\\
&=
\inf_{\lambda\in\mathrm{Alt}(\mu)}
\mathbb{E}_\mu[\tau]\sum_{(i,j)\in\mathcal{S}}\frac{\mathbb{E}_\mu\big[N_{ij}(\tau)\big]}{\mathbb{E}_\mu[\tau]}\;
d\big(\mu(i,j),\lambda(i,j)\big) \nonumber\\
&\le
\mathbb{E}_\mu[\tau]\;
\sup_{\omega\in\Omega}
\inf_{\lambda\in\mathrm{Alt}(\mu)}
\sum_{(i,j)\in\mathcal{S}}\omega_{ij}\; d\big(\mu(i,j),\lambda(i,j)\big),
\label{eq:maxmin_step}
\end{align}
where in the second line we used $\mathbb{E}_\mu[\tau]=\sum_{(i,j)\in\mathcal{S}}\mathbb{E}_\mu[N_{ij}(\tau)]$, and in the last line we introduced
\[
\omega_{ij}:=\frac{\mathbb{E}_\mu[N_{ij}(\tau)]}{\mathbb{E}_\mu[\tau]},
\qquad\text{so that}\qquad
\omega\in\Omega:=\Big\{\omega\ge 0:\sum_{i<j}\omega_{ij}=1\Big\}.
\]

Rearranging \eqref{eq:maxmin_step} gives
\begin{equation}
\label{eq:lower_bound_tau}
\mathbb{E}_\mu[\tau]
\;\ge\;
\frac{\text{kl}(\delta,1-\delta)}{\sup_{\omega\in\Omega}\inf_{\lambda\in\mathrm{Alt}(\mu)}\sum_{(i,j)\in\mathcal{S}}\omega_{ij} d(\mu(i,j),\lambda(i,j))}
\;=\;
\mathcal{T}^\star(\mu)\,\text{kl}(\delta,1-\delta),
\end{equation}
where $\mathcal{T}^\star(\mu)$ is defined by \eqref{eq: complexity opt}. This proves the first claim.

Finally, since $\text{kl}(\delta,1-\delta)\sim \log(1/\delta)$ as $\delta\to 0$, dividing \eqref{eq:lower_bound_tau} by $\log(1/\delta)$ and taking $\liminf_{\delta\to 0}$ yields
\[
\liminf_{\delta\to 0}\frac{\mathbb{E}_\mu[\tau]}{\log(1/\delta)}\ge \mathcal{T}^\star(\mu),
\]
which completes the proof.
\hfill$\square$
\qed

\section{Proof of Corollary \ref{corollary: optimal ratio}}
For each $i\in[K]$, define the dominance region
\begin{equation}
\label{eq:Ci}
\mathcal{C}_i
:=\Big\{\lambda\in\mathcal{H}:\ \min_{j\neq i}\lambda(i,j)>\tfrac12\Big\}.
\end{equation}
Then $i^\star(\lambda)=i$ if and only if $\lambda\in\mathcal{C}_i$. In particular,
\begin{equation}
\label{eq:Alt_union_rigorous}
\mathrm{Alt}(\mu)
=\{\lambda\in\mathcal{H}: i^\star(\lambda)\neq i^\star(\mu)\}
=\bigcup_{i\neq i^\star(\mu)} \mathcal{C}_i .
\end{equation}

The optimization problem in \eqref{eq: complexity opt} is furthermore equivalent to
\begin{equation*}
        \mathcal{T}^\star(\mu)^{-1} = \max_{\omega\in\Omega}\inf_{\lambda\in \mathrm{Alt} (\mu)} \sum_{(i,j)\in\mathcal{S}} \omega_{ij} d(\mu(i,j),\lambda(i,j)) = \max_{\omega \in \Omega} \min_{i\neq i^\star(\mu)}\inf_{\lambda\in\mathcal{C}_i}\sum_{(i,j)\in\mathcal{S}} \omega_{ij} d(\mu(i,j),\lambda(i,j)).
\end{equation*}

Fix $i\neq i^\star(\mu)$ and $\omega\in\Omega$.
Since the constraints defining $\mathcal{C}_i$ only involve comparisons between $i$ and other policies, the inner objective decomposes, and only terms involving $i$ may change. Define
\begin{equation}
\label{eq:Fi}
F_i(\omega,\lambda)
:=
\sum_{j<i}\omega_{ji}\, d\big(\mu(j,i),\lambda(j,i)\big)
+\sum_{h>i}\omega_{ih}\, d\big(\mu(i,h),\lambda(i,h)\big).
\end{equation}
Then
\begin{equation}
\label{eq:reduce_to_Fi}
\inf_{\lambda\in\mathcal{C}_i}\sum_{i<j} \omega_{ij} d(\mu(i,j),\lambda(i,j))
=
\inf_{\lambda\in\mathcal{C}_i} F_i(\omega,\lambda).
\end{equation}

We next characterize the infimum over \(\mathcal C_i\). For \(\lambda\in\mathcal{C}_i\), we require \(\lambda(i,j)>\tfrac12\) for all \(j\neq i\). Thus, for every \(j<i\) with \(\mu(j,i)>\tfrac12\), the value \(\lambda(j,i)\) must be moved from above \(\tfrac12\) to below \(\tfrac12\); similarly, for every \(h>i\) with \(\mu(i,h)<\tfrac12\), the value \(\lambda(i,h)\) must be moved from below \(\tfrac12\) to above \(\tfrac12\).

Note that if \(p<\tfrac12\), then
\begin{equation}
\label{eq:projection}
\inf_{q>\tfrac12} d(p,q) = d\left(p,\tfrac12\right),
\end{equation}
and the infimum is approached by \(q\downarrow\tfrac12\). Likewise, if \(p>\tfrac12\), then
\[
\inf_{q<\tfrac12} d(p,q)=d\left(p,\tfrac12\right),
\]
and the infimum is approached by \(q\uparrow \tfrac12\).

Applying these facts coordinate-wise, an infimizing sequence over \(\mathcal C_i\) is obtained by setting
\[
\lambda(j,i)\uparrow \tfrac12 \quad \text{for all } j\in\mathcal I_1(i),
\qquad
\lambda(i,h)\downarrow \tfrac12 \quad \text{for all } h\in\mathcal I_2(i),
\]
while keeping all other pairs unchanged. Consequently,
\begin{equation}
\label{eq: simple opt}
     \mathcal{T}^\star(\mu)^{-1} = \max_{\omega \in \Omega} \min_{i\neq i^\star(\mu)} \sum_{j\in \mathcal{I}_1(i)} \omega_{ji} d\left(\mu(j,i),\frac{1}{2}\right)+\sum_{h\in\mathcal{I}_2(i)} \omega_{ih} d\left(\mu(i,h),\frac{1}{2}\right),
\end{equation}
where the set $\mathcal{I}_1(i) = \{j\in[K]: j<i,\mu(j,i)>\frac{1}{2}\}$, and $\mathcal{I}_2(i) = \{h\in[K]: i<h,\mu(i,h)<\frac{1}{2}\}$.

Consider the epigraph formulation \eqref{eq: epi opt}:
\begin{equation}
\label{eq: epi opt}
\begin{aligned}
    &\max\quad\nu\\
    \text{s.t.}  &\sum_{j\in \mathcal{I}_1(i)} \omega_{ji} d\left(\mu(j,i),\frac{1}{2}\right)+\sum_{h\in\mathcal{I}_2(i)} \omega_{ih} d\left(\mu(i,h),\frac{1}{2}\right) \geq \nu,\quad \forall i\neq i^\star(\mu),\quad (\rho_i)\\
    & \sum_{i<j}\omega_{ij} = 1,\quad(\gamma)\\
    & \omega_{ij} \geq 0, \quad\forall  i<j, \quad(\xi_{ij}).
\end{aligned}
\end{equation}

Let $(\nu,\omega)$ be a primal optimal solution and $(\rho,\gamma,\xi)$ be KKT multipliers associated with
the Lagrangian 
\begin{equation}
\label{eq:Lagrangian_aligned}
\begin{aligned}
\mathcal{L}(\nu,\omega,\rho,\gamma,\xi)
=
&\ \nu
+\sum_{i\neq i^\star(\mu)}\rho_i
\left(
\sum_{j\in \mathcal{I}_1(i)} \omega_{ji} d\left(\mu(j,i),\tfrac12\right)
+\sum_{h\in\mathcal{I}_2(i)} \omega_{ih} d\left(\mu(i,h),\tfrac12\right)
-\nu
\right)\\
&\ -\gamma\left(\sum_{i<j}\omega_{ij}-1\right)
+\sum_{i<j}\xi_{ij}\omega_{ij},
\end{aligned}
\end{equation}
where $\rho_i\ge 0$ and $\xi_{ij}\ge 0$.

The KKT conditions include:
\begin{itemize}
    \item Stationarity with respect to $\nu$:
    \begin{equation}
    \label{eq: normalization}
        \sum_{i\neq i^\star(\mu)} \rho_i = 1,
    \end{equation}
    \item Stationarity with respect to $\omega_{ij}$ with $i<j$:
    \begin{equation}
    \label{eq: stationary 1}
       \rho_i d\left(\mu(i,j),\frac{1}{2}\right)\mathbb{I}\left(\mu(i,j)<\frac{1}{2}\right) + \rho_j d\left(\mu(i,j),\frac{1}{2}\right)\mathbb{I}\left(\mu(i,j)>\frac{1}{2}\right)+ \xi_{ij} = \gamma
    \end{equation}
    \item Complementary Slackness:
    \begin{equation}
    \label{eq: cs1}
        \rho_i \left(\sum_{j\in \mathcal{I}_1(i)} \omega_{ji} d\left(\mu(j,i),\frac{1}{2}\right)+\sum_{h\in\mathcal{I}_2(i)} \omega_{ih} d\left(\mu(i,h),\frac{1}{2}\right) - \nu\right) = 0,\quad \forall i\neq i^\star(\mu),
    \end{equation}
    and
    \begin{equation}
    \label{eq: cs2}
        \xi_{ij}\omega_{ij} = 0, \quad \forall i<j.
    \end{equation}
\end{itemize}

Since \(d(\cdot,\cdot)\ge 0\) and \(\omega_{ij}\ge 0\), the left-hand side of each inequality constraint in \eqref{eq: epi opt} is nonnegative. Therefore, for any \(\omega\in\Omega\), the pair \((\nu,\omega)=(0,\omega)\) is feasible. Hence the optimal value of \eqref{eq: epi opt} is at least \(0\), that is, \(\nu^\star\ge 0\).

We next show that the optimal value is strictly positive. Fix any \(i\neq i^\star(\mu)\). Since \(\mu\in\mathcal H\) admits a unique optimal policy \(i^\star(\mu)\), policy \(i\) must be defeated by at least one other policy. Equivalently,
\[
\mathcal I_1(i)\cup \mathcal I_2(i)\neq \varnothing .
\]
Hence there exists at least one comparison involving \(i\) whose associated divergence term is strictly positive, that is,
\[
d\left(\mu(j,i),\tfrac12\right)>0 \quad \text{for some } j\in\mathcal I_1(i),
\]
or
\[
d\left(\mu(i,h),\tfrac12\right)>0 \quad \text{for some } h\in\mathcal I_2(i).
\]

Now choose any \(\omega\in\Omega\) with full support, i.e., \(\omega_{ij}>0\) for all \(i<j\). Then, for every \(i\neq i^\star(\mu)\), the left-hand side of the \(i\)-th constraint in \eqref{eq: epi opt} is strictly positive. Since there are only finitely many such \(i\), their minimum is also strictly positive. Therefore, there exists some \(\bar\nu>0\) such that \((\bar\nu,\omega)\) is feasible for \eqref{eq: epi opt}. It follows that the optimal value satisfies
$
\nu^\star>0.
$

According to \eqref{eq: normalization}, there exists some \(p\neq i^\star(\mu)\) such that \(\rho_p>0\). Since \(p\) is suboptimal under \(\mu\), it must be defeated by at least one other policy. Equivalently, either there exists \(k<p\) such that \(\mu(k,p)>\tfrac12\), or there exists \(k>p\) such that \(\mu(p,k)<\tfrac12\). In either case, the corresponding divergence term is strictly positive. If \(k<p\), then applying \eqref{eq: stationary 1} to the pair \((k,p)\) gives
\[
\gamma
=
\rho_p\, d\left(\mu(k,p),\tfrac12\right)+\xi_{kp}.
\]
Since \(\rho_p>0\), \(d(\mu(k,p),\tfrac12)>0\), and \(\xi_{kp}\ge 0\), it follows that \(\gamma>0\).
If instead \(k>p\), then applying \eqref{eq: stationary 1} to the pair \((p,k)\) yields
\[
\gamma
=
\rho_p\, d\left(\mu(p,k),\tfrac12\right)+\xi_{pk}.
\]
Again, since \(\rho_p>0\), \(d(\mu(p,k),\tfrac12)>0\), and \(\xi_{pk}\ge 0\), we obtain \(\gamma>0\).
Hence, \(\gamma>0\).

Suppose, for contradiction, that there exists \(i\neq i^\star(\mu)\) such that \(\rho_i=0\). 
First, consider any \(h\in \mathcal I_2(i)\). Then \(i<h\) and \(\mu(i,h)<\tfrac12\). Applying \eqref{eq: stationary 1} to the pair \((i,h)\) gives
\[
\rho_i d\left(\mu(i,h),\tfrac12\right)-\gamma+\xi_{ih}=0.
\]
Since \(\rho_i=0\) and \(\gamma>0\), we obtain \(\xi_{ih}=\gamma>0\). By complementary slackness \eqref{eq: cs2}, it follows that
\[
\omega_{ih}=0, \qquad \forall h\in\mathcal I_2(i).
\]

Next, consider any \(j\in \mathcal I_1(i)\). Then \(j<i\) and \(\mu(j,i)>\tfrac12\). Applying \eqref{eq: stationary 1} to the pair \((j,i)\) yields
\[
\rho_i d\left(\mu(j,i),\tfrac12\right)-\gamma+\xi_{ji}=0.
\]
Again, since \(\rho_i=0\) and \(\gamma>0\), we have \(\xi_{ji}=\gamma>0\), and thus \(\omega_{ji}=0\) by \eqref{eq: cs2}. Therefore,
\[
\omega_{ji}=0, \qquad \forall j\in\mathcal I_1(i).
\]

Consequently, every term in the \(i\)-th constraint of \eqref{eq: epi opt} vanishes, so
\[
\sum_{j\in \mathcal I_1(i)} \omega_{ji} d\left(\mu(j,i),\tfrac12\right)
+\sum_{h\in\mathcal I_2(i)} \omega_{ih} d\left(\mu(i,h),\tfrac12\right)
=0.
\]
Since \((\nu,\omega)\) is primal feasible and \((\nu,\omega)\) is optimal, we have \(\nu=\nu^\star\). Therefore, the \(i\)-th constraint implies \(0\ge \nu=\nu^\star\), contradicting \(\nu^\star>0\). Hence \(\rho_i>0\) for all \(i\neq i^\star(\mu)\).

Since \(\rho_i>0\) for all \(i\neq i^\star(\mu)\), complementary slackness \eqref{eq: cs1} implies
\begin{equation}
\label{eq: all_tight}
\sum_{j\in \mathcal{I}_1(i)} \omega_{ji} d\left(\mu(j,i),\tfrac12\right)
+\sum_{h\in\mathcal{I}_2(i)} \omega_{ih} d\left(\mu(i,h),\tfrac12\right)
=\nu,\quad \forall i\neq i^\star(\mu).
\end{equation}

Fix \(i\neq i^\star(\mu)\).
If \(\omega_{ji}>0\) for some \(j\in\mathcal I_1(i)\), then \eqref{eq: cs2} gives \(\xi_{ji}=0\), and since \(\mu(j,i)>\tfrac12\), \eqref{eq: stationary 1} yields
\begin{equation}
\label{eq: rho_gamma_case1}
\rho_i\, d\left(\mu(j,i),\tfrac12\right)=\gamma.
\end{equation}
If \(\omega_{ih}>0\) for some \(h\in\mathcal I_2(i)\), then \eqref{eq: cs2} gives \(\xi_{ih}=0\), and since \(\mu(i,h)<\tfrac12\), \eqref{eq: stationary 1} yields
\begin{equation}
\label{eq: rho_gamma_case2}
\rho_i\, d\left(\mu(i,h),\tfrac12\right)=\gamma.
\end{equation}
Hence, within the \(i\)-th tight constraint, all comparisons carrying positive mass must have the same divergence value.

Now let
\[
d_i^\star
:=
\max\Big\{
\max_{j\in\mathcal{I}_1(i)} d\left(\mu(j,i),\tfrac12\right),
\ \max_{h\in\mathcal{I}_2(i)} d\left(\mu(i,h),\tfrac12\right)
\Big\},
\]
and let \(\tilde j(i)\in\mathcal I_1(i)\cup\mathcal I_2(i)\) be any maximizer achieving \(d_i^\star\). Since each comparison pair appears in exactly one suboptimal-policy constraint, for each \(i\neq i^\star(\mu)\) we may redistribute all mass assigned to the \(i\)-th constraint onto the single comparison associated with \(\tilde j(i)\), without affecting any other constraint or the total mass constraint. More precisely, if \(\tilde j(i)<i\), then the selected comparison is \((\tilde j(i),i)\); if \(i<\tilde j(i)\), then the selected comparison is \((i,\tilde j(i))\).

Moreover, any mass assigned to pairs outside
\[
\bigcup_{i\neq i^\star(\mu)}
\Big(
\{(j,i): j\in\mathcal I_1(i)\}\cup\{(i,h): h\in\mathcal I_2(i)\}
\Big)
\]
does not contribute to any constraint in \eqref{eq: epi opt}. Hence, such mass can be reallocated to the selected maximizing comparisons without decreasing the objective value. 
Therefore, there exists an optimal solution such that, for every \(i\neq i^\star(\mu)\), the only nonzero term in the \(i\)-th constraint is the one associated with \(\tilde j(i)\). In particular, \eqref{eq: all_tight} becomes
\[
\omega_{\tilde j(i),i}\, d_i^\star\,\mathbb{I}\{\tilde j(i)<i\}
+\omega_{i,\tilde j(i)}\, d_i^\star\,\mathbb{I}\{i<\tilde j(i)\}
=\nu,
\]
so that
\[
\omega_{\tilde j(i),i}=\frac{\nu}{d_i^\star}\qquad \text{if }\tilde j(i)<i,
\]
and
\[
\omega_{i,\tilde j(i)}=\frac{\nu}{d_i^\star}\qquad \text{if }i<\tilde j(i).
\]

Summing over all \(i\neq i^\star(\mu)\) and using \(\sum_{i<j}\omega_{ij}=1\), we obtain
\[
1=\sum_{i\neq i^\star(\mu)}\frac{\nu}{d_i^\star},
\qquad\Longrightarrow\qquad
\nu=\frac{1}{\sum_{k\neq i^\star(\mu)}\frac{1}{d_k^\star}}.
\]
Substituting this identity back yields
\[
\omega_{\tilde j(i),i}
=
\frac{\frac{1}{d_i^\star}}{\sum_{k\neq i^\star(\mu)}\frac{1}{d_k^\star}}
\qquad \text{if }\tilde j(i)<i,
\]
and
\[
\omega_{i,\tilde j(i)}
=
\frac{\frac{1}{d_i^\star}}{\sum_{k\neq i^\star(\mu)}\frac{1}{d_k^\star}}
\qquad \text{if }i<\tilde j(i),
\]
with all other \(\omega_{ab}=0\). If the maximizer defining \(d_i^\star\) is not unique, then the optimal solution is not unique.
\qed

\section{Proof of Theorem \ref{thm: combinatorial_structure}}
For any suboptimal policy \(i\neq i^\star(\mu)\), define $$\bar{\mathcal{C}}_i(\mu) := \left\{\lambda\in\mathcal{H}: \lambda(i,i^\star(\mu))>\frac{1}{2}\right\}.$$ Since $$\mathcal{C}_i=\left\{\lambda\in\mathcal{H}:\min_{j\neq i}\lambda(i,j)>\tfrac12\right\}\subset \bar{\mathcal{C}}_i(\mu),$$ we have, for every \(\omega\in\Omega\),
\begin{equation}
\label{eq:structured_relax}
\inf_{\lambda\in\mathcal C_i}
\sum_{a<b}\omega_{ab}\,d\bigl(\mu(a,b),\lambda(a,b)\bigr)
\ge
\inf_{\lambda\in\bar{\mathcal C}_i(\mu)}
\sum_{a<b}\omega_{ab}\,d\bigl(\mu(a,b),\lambda(a,b)\bigr).
\end{equation}

Under the Bradley-Terry parameterization, any $\lambda\in\mathcal{H}$ can be written as
$
\lambda(a,b)=\sigma(\tilde\theta^\top z_{ab})
$
for some parameter \(\tilde\theta\).
Hence, the relaxed inner problem becomes
\begin{equation}
\label{eq:relaxed_param_problem}
\inf_{\tilde\theta:\,\tilde\theta^\top z_{i\,i^\star(\mu)}>0}
\sum_{a<b}\omega_{ab}\,
d\Bigl(\sigma(\theta_\star^\top z_{ab}),\sigma(\tilde\theta^\top z_{ab})\Bigr).
\end{equation}
Because the objective is continuous in \(\tilde\theta\), the infimum over the open half-space above is equal to the infimum over its closure. Therefore,
\begin{equation}
\label{eq:relaxed_closed}
\inf_{\lambda\in\bar{\mathcal C}_i(\mu)}
\sum_{a<b}\omega_{ab}\,d\bigl(\mu(a,b),\lambda(a,b)\bigr)
=
\inf_{\tilde\theta:\,\tilde\theta^\top z_{i\,i^\star(\mu)}\ge 0}
\sum_{a<b}\omega_{ab}\,
d\Bigl(\sigma(\theta_\star^\top z_{ab}),\sigma(\tilde\theta^\top z_{ab})\Bigr).
\end{equation}

The KL divergence between two Bernoulli random variables can be represented as \citep{jun2021improved}
\begin{equation*}
\begin{aligned}
d(\sigma(\theta_\star^\top z_{ab}),\sigma(\tilde{\theta}^\top z_{ab}))
&=
\sigma(\theta_\star^\top z_{ab})
\log \frac{\sigma(\theta_\star^\top z_{ab})}{\sigma(\tilde{\theta}^\top z_{ab})}
+
\bigl(1-\sigma(\theta_\star^\top z_{ab})\bigr)
\log\frac{1-\sigma(\theta_\star^\top z_{ab})}{1-\sigma(\tilde{\theta}^\top z_{ab})}\\
&=
\sigma(\theta_\star^\top z_{ab})\,z_{ab}^\top(\theta_\star-\tilde{\theta})
+
\log\frac{1-\sigma(\theta_\star^\top z_{ab})}{1-\sigma(\tilde{\theta}^\top z_{ab})}\\
&=
\sigma(\theta_\star^\top z_{ab})\,z_{ab}^\top(\theta_\star-\tilde{\theta})
+
\log\bigl(1-\sigma(\theta_\star^\top z_{ab})\bigr)
-
\log\bigl(1-\sigma(\tilde{\theta}^\top z_{ab})\bigr).
\end{aligned}
\end{equation*}
Using \(\sigma^\prime(x)=\sigma(x)\bigl(1-\sigma(x)\bigr)\), we have
\begin{equation*}
\begin{aligned}
\nabla_{\tilde{\theta}} d(\sigma(\theta_\star^\top z_{ab}),\sigma(\tilde{\theta}^\top z_{ab}))
&=
-\sigma(\theta_\star^\top z_{ab})\,z_{ab}
+
\frac{\sigma^\prime(\tilde{\theta}^\top z_{ab})}{1-\sigma(\tilde{\theta}^\top z_{ab})}z_{ab}\\
&=
\bigl(\sigma(\tilde{\theta}^\top z_{ab})-\sigma(\theta_\star^\top z_{ab})\bigr)z_{ab}.
\end{aligned}
\end{equation*}
Therefore,
\begin{equation}
\label{eq:grad_sum_app}
\nabla_{\tilde{\theta}}
\sum_{a<b}\omega_{ab}\,
d(\sigma(\theta_\star^\top z_{ab}),\sigma(\tilde{\theta}^\top z_{ab}))
=
\sum_{a<b}\omega_{ab}
\bigl(\sigma(\tilde{\theta}^\top z_{ab})-\sigma(\theta_\star^\top z_{ab})\bigr)z_{ab}.
\end{equation}

Since the objective is convex in \(\tilde{\theta}\), the constraint is affine, and Slater's condition holds for the relaxed problem \eqref{eq:relaxed_closed}, the KKT conditions are necessary and sufficient. The Lagrangian is
\[
\mathcal{L}(\tilde{\theta},\psi)
=
\sum_{a<b}\omega_{ab}\,
d\Big(\sigma(\theta_\star^\top z_{ab}),\sigma(\tilde{\theta}^\top z_{ab})\Big)
-\psi\,\tilde{\theta}^\top z_{i\,i^\star(\mu)},
\qquad \psi\ge 0.
\]
The KKT conditions are
\begin{align}
\label{eq:kkt_stationarity_app}
\sum_{a<b}\omega_{ab}\bigl(\sigma(\tilde{\theta}^\top z_{ab})-\sigma(\theta_\star^\top z_{ab})\bigr)\,z_{ab}
-\psi\, z_{i\,i^\star(\mu)} &= 0,\\
\label{eq:kkt_primal_app}
\tilde{\theta}^\top z_{i\,i^\star(\mu)} &\ge 0,\\
\label{eq:kkt_dual_app}
\psi &\ge 0,\\
\label{eq:kkt_cs_app}
\psi\,\tilde{\theta}^\top z_{i\,i^\star(\mu)} &= 0.
\end{align}

Moreover, since $i\neq i^\star(\mu)$ implies $\mu\big(i,i^\star(\mu)\big)<\tfrac12$, we have
$\theta_\star^\top z_{i\,i^\star(\mu)}<0$.
Thus $\tilde{\theta}=\theta_\star$ is infeasible. Since the unconstrained objective is minimized at \(\theta_\star\), the optimum of the constrained problem must be attained on the boundary; hence, at any optimal solution,
\begin{equation}
\label{eq:active_app}
\tilde{\theta}^\top z_{i\,i^\star(\mu)}=0
.
\end{equation}

Next, define
\[
\alpha_{ab}(\tilde{\theta},\theta_\star)
:=
\begin{cases}
\displaystyle
\frac{\sigma(\tilde{\theta}^\top z_{ab})-\sigma(\theta_\star^\top z_{ab})}
{(\tilde{\theta}-\theta_\star)^\top z_{ab}},
& \text{if } (\tilde{\theta}-\theta_\star)^\top z_{ab}\neq 0,\\[1.2em]
\displaystyle
\sigma'(\theta_\star^\top z_{ab}),
& \text{if } (\tilde{\theta}-\theta_\star)^\top z_{ab}=0,
\end{cases}
\]
and
\[
G(\omega,\tilde{\theta},\theta_\star)
:=
\sum_{a<b}\omega_{ab}\,\alpha_{ab}(\tilde{\theta},\theta_\star)\,z_{ab}z_{ab}^\top.
\]
Then
\[
\sum_{a<b}\omega_{ab}\bigl(\sigma(\tilde{\theta}^\top z_{ab})-\sigma(\theta_\star^\top z_{ab})\bigr)z_{ab}
=
G(\omega,\tilde{\theta},\theta_\star)(\tilde{\theta}-\theta_\star).
\]
Hence, by \eqref{eq:kkt_stationarity_app},
\[
G(\omega,\tilde{\theta},\theta_\star)(\tilde{\theta}-\theta_\star)
=
\psi\,z_{i\,i^\star(\mu)}.
\]
In the local hard-instance regime, \(\tilde{\theta}\to\theta_\star\), and hence
\(G(\omega,\tilde{\theta},\theta_\star)\to H(\theta_\star,\omega)\).
At the maximizing \(\omega\) in the outer supremum, \(H(\theta_\star,\omega)\) is non-singular; therefore, for \(\tilde{\theta}\) sufficiently close to \(\theta_\star\), \(G(\omega,\tilde{\theta},\theta_\star)\) is also non-singular.
\begin{equation}
\label{eq:theta_diff_G_app}
\tilde{\theta}-\theta_\star
=
\psi\,G(\omega,\tilde{\theta},\theta_\star)^{-1}z_{i\,i^\star(\mu)}.
\end{equation}
Using \eqref{eq:active_app}, we obtain
\[
0
=
\tilde{\theta}^\top z_{i\,i^\star(\mu)}
=
\theta_\star^\top z_{i\,i^\star(\mu)}
+
\psi\,z_{i\,i^\star(\mu)}^\top
G(\omega,\tilde{\theta},\theta_\star)^{-1}
z_{i\,i^\star(\mu)},
\]
and thus
\[
\psi
=
-\frac{\theta_\star^\top z_{i\,i^\star(\mu)}}
{\|z_{i\,i^\star(\mu)}\|^2_{G(\omega,\tilde{\theta},\theta_\star)^{-1}}}.
\]
Substituting this into \eqref{eq:theta_diff_G_app} yields
\begin{equation}
\label{eq:tilde_theta_G_app}
\tilde{\theta}
=
\theta_\star
-
\frac{\theta_\star^\top z_{i\,i^\star(\mu)}}
{\|z_{i\,i^\star(\mu)}\|^2_{G(\omega,\tilde{\theta},\theta_\star)^{-1}}}
\,G(\omega,\tilde{\theta},\theta_\star)^{-1}z_{i\,i^\star(\mu)}.
\end{equation}

We next specialize this expression to the local hard-instance regime. By the integral form of Taylor's theorem,
\begin{equation*}
\begin{aligned}
d(\sigma(\theta_\star^\top z_{ab}),\sigma(\tilde{\theta}^\top z_{ab}))
&=
\left(\int_0^1
(1-t)\,
\sigma'\bigl(\theta_\star^\top z_{ab}
+t(\tilde{\theta}-\theta_\star)^\top z_{ab}\bigr)\,dt\right)
\bigl((\tilde{\theta}-\theta_\star)^\top z_{ab}\bigr)^2\\
&=
(\tilde{\theta}-\theta_\star)^\top
\left[
\left(\int_0^1
(1-t)\,
\sigma'\bigl(\theta_\star^\top z_{ab}
+t(\tilde{\theta}-\theta_\star)^\top z_{ab}\bigr)\,dt\right)
z_{ab}z_{ab}^\top
\right]
(\tilde{\theta}-\theta_\star).
\end{aligned}
\end{equation*}
Therefore,
\begin{equation*}
\sum_{a<b}\omega_{ab}\,
d(\sigma(\theta_\star^\top z_{ab}),\sigma(\tilde{\theta}^\top z_{ab}))
=
(\tilde{\theta}-\theta_\star)^\top
\left[
\sum_{a<b}\omega_{ab}
\left(\int_0^1
(1-t)\,
\sigma'\bigl(\theta_\star^\top z_{ab}
+t(\tilde{\theta}-\theta_\star)^\top z_{ab}\bigr)\,dt\right)
z_{ab}z_{ab}^\top
\right]
(\tilde{\theta}-\theta_\star).
\end{equation*}
In the local regime \(\tilde{\theta}\to\theta_\star\), we have
\[
\int_0^1
(1-t)\,
\sigma'\bigl(\theta_\star^\top z_{ab}
+t(\tilde{\theta}-\theta_\star)^\top z_{ab}\bigr)\,dt
\;\to\;
\int_0^1(1-t)\,\sigma'(\theta_\star^\top z_{ab})\,dt
=
\frac12\,\sigma'(\theta_\star^\top z_{ab}),
\]
and
\[
\alpha_{ab}(\tilde{\theta},\theta_\star)\to \sigma'(\theta_\star^\top z_{ab}).
\]
Hence, the objective admits the quadratic approximation
\[
\sum_{a<b}\omega_{ab}\,
d\Bigl(\sigma(\theta_\star^\top z_{ab}),\sigma(\tilde{\theta}^\top z_{ab})\Bigr)
=
\frac12(\tilde{\theta}-\theta_\star)^\top
H(\theta_\star,\omega)
(\tilde{\theta}-\theta_\star)
+
o\bigl(\|\tilde{\theta}-\theta_\star\|_2^2\bigr),
\]
where
\[
H(\theta_\star,\omega)
:=
\sum_{a<b}\omega_{ab}\sigma'(\theta_\star^\top z_{ab})\,z_{ab}z_{ab}^\top.
\]
Moreover,
\[
G(\omega,\tilde{\theta},\theta_\star)
=
H(\theta_\star,\omega)+o(1).
\]
Thus, \eqref{eq:tilde_theta_G_app} reduces to
\[
\tilde{\theta}
=
\theta_\star
-
\frac{\theta_\star^\top z_{i\,i^\star(\mu)}}
{\|z_{i\,i^\star(\mu)}\|^2_{H(\theta_\star,\omega)^{-1}}}
\,H(\theta_\star,\omega)^{-1}z_{i\,i^\star(\mu)}
+
o\bigl(\|\tilde{\theta}-\theta_\star\|_2\bigr).
\]
Plugging this into the quadratic approximation yields
\[
\inf_{\tilde{\theta}:\,\tilde{\theta}^\top z_{i\,i^\star(\mu)}\ge 0}
\sum_{a<b}\omega_{ab}\,
d\Bigl(\sigma(\theta_\star^\top z_{ab}),\sigma(\tilde{\theta}^\top z_{ab})\Bigr)
=
\frac12\,
\frac{\bigl(\theta_\star^\top z_{i\,i^\star(\mu)}\bigr)^2}
{\|z_{i\,i^\star(\mu)}\|^2_{H(\theta_\star,\omega)^{-1}}}
+
o\!\left(\|\tilde{\theta}-\theta_\star\|_2^2\right).
\]
Combining this with \eqref{eq:structured_relax}, we obtain
\[
\inf_{\lambda\in\mathcal C_i}
\sum_{a<b}\omega_{ab}\,d\bigl(\mu(a,b),\lambda(a,b)\bigr)
\ge
\frac12\,
\frac{\bigl(\theta_\star^\top z_{i\,i^\star(\mu)}\bigr)^2}
{\|z_{i\,i^\star(\mu)}\|^2_{H(\theta_\star,\omega)^{-1}}}
\]
under the local hard-instance approximation. Since this holds for every \(i\neq i^\star(\mu)\), taking the minimum over \(i\neq i^\star(\mu)\) and then the supremum over \(\omega\in\Omega\) yields
\[
\mathcal{T}^\star(\mu)^{-1}
\ge
\sup_{\omega\in\Omega}
\min_{i\neq i^\star(\mu)}
\frac12\,
\frac{\bigl(\theta_\star^\top z_{i\,i^\star(\mu)}\bigr)^2}
{\|z_{i\,i^\star(\mu)}\|^2_{H(\theta_\star,\omega)^{-1}}}
=
\mathcal{U}^\star(\mu)^{-1}.
\]
Equivalently,
\[
\mathcal{T}^\star(\mu)\le \mathcal{U}^\star(\mu).
\]
This completes the proof.
\qed

\section{Proof of Lemma \ref{lemma: param_concentration}}
To prove Lemma \ref{lemma: param_concentration}, we first establish several auxiliary lemmas, stated and proved below.

\begin{lemma}[Uniform Self-normalized Bound]
\label{lem:St_bound_lambda0_is_lambda_t0}
Let $\{z_t\}_{t\ge1}\subset \mathbb{R}^d$ be presictable with respect to $\{\mathcal{F}_t\}_{t\ge 0}$, i.e., $z_t\in\mathcal{F}_{t-1}$ for all $t\ge1$.
For each $s\ge 1$, define
\[
\varepsilon_s := D_s-\mu(\theta_*^\top z_s),\qquad
S_t := \sum_{s=1}^{t-1}\varepsilon_s z_s,\qquad
V_t := \sum_{s=1}^{t-1} z_s z_s^\top .
\]
Suppose there exist $t_0\ge 1$ and a nondecreasing function $\lambda(\cdot)$ such that
$\lambda_{\min}(V_t)\ge \lambda(t)$ for all $t\ge t_0$.
Define the constant $\lambda_0 := \lambda(t_0) > 0$.
Then for any $\delta\in(0,1)$, with probability at least $1-\delta$, for all $t\ge t_0$,
\[
\|S_t\|_{V_t^{-1}}^2
\;\le\;
\Big(1+\frac{\lambda_0}{\lambda(t)}\Big)
\left(
2\log\frac{1}{\delta}
+
\log \frac{\det(V_t)}{\lambda(t)^d}
+
d\log \Big(1+\frac{\lambda_0}{\lambda(t)}\Big)
+
d\log \frac{\lambda(t)}{\lambda_0}
\right).
\]
\end{lemma}

\textit{Proof of Lemma \ref{lem:St_bound_lambda0_is_lambda_t0}.}
Since $D_s\in\{0,1\}$ and $\mathbb{E}[D_s\mid\mathcal{F}_{s-1}]=\mu(\theta_*^\top z_s)$, we have
$\varepsilon_s\in[-1,1]$ and $\mathbb{E}[\varepsilon_s\mid\mathcal{F}_{s-1}]=0$ a.s.
By Hoeffding's lemma, for any $\alpha\in\mathbb{R}$,
\begin{equation}
\label{eq:hoeffding_eps_lambda0_is_t0}
\mathbb{E}\big[\exp(\alpha\varepsilon_s)\mid\mathcal{F}_{s-1}\big]
\le \exp(\alpha^2/2)\qquad\text{a.s.}
\end{equation}

Fix $u\in\mathbb{R}^d$ and define, for $t\ge 1$,
\[
M_t(u)
:=
\exp\left(
u^\top S_t-\frac12 u^\top V_t\,u
\right).
\]
Using $S_{t+1}=S_t+\varepsilon_t z_t$ and $V_{t+1}=V_t+z_t z_t^\top$,
\begin{align*}
\mathbb{E}[M_{t+1}(u)\mid\mathcal{F}_{t-1}]
&=
M_t(u)\cdot
\mathbb{E}\left[
\exp\Big(u^\top z_t\,\varepsilon_t-\frac12 (u^\top z_t)^2\Big)
\Bigm|\mathcal{F}_{t-1}
\right]\\
&\le
M_t(u)\cdot
\exp\left(-\frac12 (u^\top z_t)^2\right)
\cdot
\mathbb{E}\left[
\exp\big((u^\top z_t)\varepsilon_t\big)
\Bigm|\mathcal{F}_{t-1}
\right]\\
&\le
M_t(u),
\end{align*}
where the last inequality uses \eqref{eq:hoeffding_eps_lambda0_is_t0} with $\alpha=u^\top z_t$.
Thus, for each fixed $u$, $\{M_t(u)\}_{t\ge 1}$ is a nonnegative supermartingale with respect to the shifted filtration $\{\mathcal{F}_{t-1}\}_{t\ge 1}$.

Let $\phi_0(u)$ be the density of $\mathcal{N}(0,\lambda_0^{-1}I_d)$:
\[
\phi_0(u)=\frac{\lambda_0^{d/2}}{(2\pi)^{d/2}}
\exp\left(-\frac{\lambda_0}{2}\|u\|_2^2\right).
\]
Define the mixture process
\[
\widetilde M_t := \int_{\mathbb{R}^d} M_t(u)\,\phi_0(u)\,du.
\]
By Tonelli's theorem (nonnegativity) and the supermartingale property of $M_t(u)$ for each fixed $u$,
$\{\widetilde M_t\}_{t\ge 1}$ is a nonnegative supermartingale and satisfies
$\mathbb{E}[\widetilde M_t]\le \mathbb{E}[\widetilde M_1]=1$.

Moreover, completing the square yields
\begin{equation}
\label{eq:mixture_closed_form_lambda0_is_t0}
\widetilde M_t
=
\lambda_0^{d/2}\det(V_t+\lambda_0 I_d)^{-1/2}
\exp\left(\frac12 \|S_t\|_{(V_t+\lambda_0 I_d)^{-1}}^2\right).
\end{equation}

Ville's inequality implies
\[
\mathbb{P}\Big(\sup_{t\ge 1}\widetilde M_t \ge \delta^{-1}\Big)\le \delta.
\]
Hence, on an event of probability at least $1-\delta$, for all $t\ge 1$ we have
$\widetilde M_t<\delta^{-1}$, and therefore by \eqref{eq:mixture_closed_form_lambda0_is_t0},
\begin{equation}
\label{eq:base_ineq}
\|S_t\|_{(V_t+\lambda_0 I_d)^{-1}}^2
\le
2\log\frac{1}{\delta}
+
\log \frac{\det(V_t+\lambda_0 I_d)}{\lambda_0^d},
\qquad \forall t\ge 1.
\end{equation}

Fix any $t\ge t_0$. Since $\lambda_{\min}(V_t)\ge \lambda(t)$, we have $V_t\succeq \lambda(t)I_d$, and thus
\begin{equation}
\label{eq:PSD_step}
V_t+\lambda_0 I_d \preceq \Big(1+\frac{\lambda_0}{\lambda(t)}\Big)V_t.
\end{equation}
Taking inverses reverses the Loewner order:
\[
V_t^{-1}
\preceq
\Big(1+\frac{\lambda_0}{\lambda(t)}\Big)(V_t+\lambda_0 I_d)^{-1}.
\]
Therefore,
\begin{equation}
\label{eq:norm_compare}
\|S_t\|_{V_t^{-1}}^2
\le
\Big(1+\frac{\lambda_0}{\lambda(t)}\Big)\|S_t\|_{(V_t+\lambda_0 I_d)^{-1}}^2.
\end{equation}

Moreover, \eqref{eq:PSD_step} implies (by eigenvalue monotonicity) that
\begin{equation}
\label{eq:det_step}
\det(V_t+\lambda_0 I_d)
\le
\det\Big(\Big(1+\frac{\lambda_0}{\lambda(t)}\Big)V_t\Big)
=
\Big(1+\frac{\lambda_0}{\lambda(t)}\Big)^d \det(V_t).
\end{equation}
Hence,
\begin{align}
\label{eq:logdet_expand_userstyle}
\log \frac{\det (V_t+\lambda_0 I_d)}{\lambda^d_0}
&\le
\log \frac{\det(V_t)}{\lambda(t)^d}
+
d\log \Big(1+\frac{\lambda_0}{\lambda(t)}\Big)
+
d\log \frac{\lambda(t)}{\lambda_0}.
\end{align}

Combining \eqref{eq:base_ineq}, \eqref{eq:norm_compare}, and \eqref{eq:logdet_expand_userstyle}
yields the claimed result for all $t\ge t_0$.
\qed

\begin{lemma}
\label{lemma: bound g}
Assume there exist $t_0\ge 1$ and a nondecreasing function $\lambda(\cdot)$ such that
$\lambda_{\min}(V_t)\ge \lambda(t)$ for all $t\ge t_0$, and define $\lambda_0:=\lambda(t_0)>0$.
Then for any $\delta\in(0,1)$, with probability at least $1-\delta$, for all $t\ge t_0$,
\begin{equation}
\label{eq:param_conc_main}
\|g_t(\hat\zeta_t)-g_t(\theta_\star)\|_{H_t(\theta_\star)^{-1}}
\le
\sqrt{m_0^{-1}\Psi_t(\delta)}
+
\sqrt{\lambda_t}B,
\end{equation}
where
\[
m_0:=\min_{|u|\le BL}\sigma'(u)>0,
\qquad
\Psi_t(\delta):=
\Big(1+\frac{\lambda_0}{\lambda(t)}\Big)
\left(
2\log\frac{1}{\delta}
+
\log \frac{\det(V_t)}{\lambda(t)^d}
+
d\log \Big(1+\frac{\lambda_0}{\lambda(t)}\Big)
+
d\log \frac{\lambda(t)}{\lambda_0}
\right).
\]
\end{lemma}

\textit{Proof of Lemma \ref{lemma: bound g}.} By definition, $\hat\zeta_t$ maximizes $\mathcal{L}_t(\theta)-\frac{\lambda_t}{2}\|\theta\|_2^2$ over $\mathbb{R}^d$.
Hence, it satisfies the first-order optimality condition
\[
\nabla_\theta\left(\mathcal{L}_t(\theta)-\frac{\lambda_t}{2}\|\theta\|_2^2\right)\Big|_{\theta=\hat\zeta_t}=0.
\]
Using
\[
\nabla_\theta \mathcal{L}_t(\theta)=\sum_{s=1}^{t-1}\big(D_s-\sigma(\theta^\top z_s)\big)z_s,
\]
we obtain
\begin{equation}
\label{eq:score_eqn}
\sum_{s=1}^{t-1}D_s z_s
=
\sum_{s=1}^{t-1}\sigma(\hat\zeta_t^\top z_s)z_s+\lambda_t\hat\zeta_t
=
g_t(\hat\zeta_t).
\end{equation}

By adding and subtracting $\sum_{s=1}^{t-1}\sigma(\theta_\star^\top z_s)z_s$,
and defining $\varepsilon_s:=D_s-\sigma(\theta_\star^\top z_s)$, we have
\[
\sum_{s=1}^{t-1}D_s z_s
=
\sum_{s=1}^{t-1}\sigma(\theta_\star^\top z_s)z_s
+
S_t.
\]
Thus, using \eqref{eq:score_eqn} and the definition of $g_t(\theta_\star)$,
\begin{equation}
\label{eq:g_diff_main}
g_t(\hat\zeta_t)-g_t(\theta_\star)
=
S_t-\lambda_t\theta_\star.
\end{equation}

Taking the $H_t(\theta_\star)^{-1}$-norm and using the triangle inequality gives
\begin{equation}
\label{eq:key_triangle}
\|g_t(\hat\zeta_t)-g_t(\theta_\star)\|_{H_t(\theta_\star)^{-1}}
\le
\|S_t\|_{H_t(\theta_\star)^{-1}}
+
\|\lambda_t\theta_\star\|_{H_t(\theta_\star)^{-1}}.
\end{equation}
Since $H_t(\theta_\star)\succeq \lambda_t I_d$, we have
\begin{equation}
\label{eq:lambda_bias}
\|\lambda_t\theta_\star\|_{H_t(\theta_\star)^{-1}}
\le
\|\lambda_t\theta_\star\|_{(\lambda_t I_d)^{-1}}
=
\sqrt{\lambda_t}B.
\end{equation}

By compactness of $\Theta$ and the assumption $\|z_s\|_2\le L$, we have
$|\theta^\top z_s|\le \|\theta\|_2\|z_s\|_2\le BL$ for all $\theta\in\Theta$.
Since $\sigma'(u)=\sigma(u)(1-\sigma(u))$ is continuous and strictly positive on $[-BL,BL]$, define
\[
m_0:=\min_{|u|\le BL}\sigma'(u)>0.
\]
Then
\[
\sum_{s=1}^{t-1}\sigma'(\theta_\star^\top z_s) z_s z_s^\top
\succeq
m_0 \sum_{s=1}^{t-1} z_s z_s^\top
=
m_0 V_t,
\]
and therefore
\begin{equation}
\label{eq:H_lower_m0}
H_t(\theta_\star)\succeq m_0 V_t.
\end{equation}
Taking inverses yields
\begin{equation}
\label{eq:Hinv_upper_m0}
H_t(\theta_\star)^{-1}\preceq \frac{1}{m_0}V_t^{-1},
\end{equation}
so
\begin{equation}
\label{eq:St_H_by_V}
\|S_t\|_{H_t(\theta_\star)^{-1}}^2
\le
\frac{1}{m_0}\|S_t\|_{V_t^{-1}}^2,
\qquad\text{i.e.,}\qquad
\|S_t\|_{H_t(\theta_\star)^{-1}}
\le
\frac{1}{\sqrt{m_0}}\|S_t\|_{V_t^{-1}}.
\end{equation}

By Lemma~\ref{lem:St_bound_lambda0_is_lambda_t0},
with probability at least $1-\delta$, for all $t\ge t_0$,
\begin{equation}
\label{eq:St_V_bound}
\|S_t\|_{V_t^{-1}}^2
\le
\Psi_t(\delta).
\end{equation}
Combining \eqref{eq:key_triangle}, \eqref{eq:lambda_bias}, \eqref{eq:St_H_by_V}, and \eqref{eq:St_V_bound},
we obtain \eqref{eq:param_conc_main}.
\qed

To continue the analysis, we introduce the following notation.
For $\theta_1,\theta_2\in\Theta$ and $z\in\mathbb{R}^d$, define
\[
\alpha(z,\theta_1,\theta_2)
:=\int_0^1 \sigma'\big(z^\top(\theta_2+v(\theta_1-\theta_2))\big)\,dv>0.
\]
Then the mean-value theorem implies
\[
\sigma(z^\top\theta_1)-\sigma(z^\top\theta_2)
=\alpha(z,\theta_1,\theta_2)\, z^\top(\theta_1-\theta_2).
\]
Recall $g_t(\theta)=\sum_{s=1}^{t-1}\sigma(\theta^\top z_s)z_s+\lambda_t\theta$.
Define the (regularized) matrix
\[
G_t(\theta_1,\theta_2)
:=\sum_{s=1}^{t-1}\alpha(z_s,\theta_1,\theta_2)\, z_s z_s^\top + \lambda_t I_d.
\]
Then,
\[
g_t(\theta_1)-g_t(\theta_2)=G_t(\theta_1,\theta_2)(\theta_1-\theta_2),
\]
and since $G_t(\theta_1,\theta_2)\succeq \lambda_t I_d\succ 0$,
\[
\|\theta_1-\theta_2\|_{G_t(\theta_1,\theta_2)}
=
\|g_t(\theta_1)-g_t(\theta_2)\|_{G_t(\theta_1,\theta_2)^{-1}}.
\]

\begin{lemma}[Lemma 10 \citealt{faury2020improved}]
\label{lemma: tranformation}
For all $\theta_1,\theta_2\in\Theta$ and all $t\ge 1$,
\[
G_t(\theta_1,\theta_2)\succeq (1+2LB)^{-1}H_t(\theta_1)
\qquad\text{and}\qquad
G_t(\theta_1,\theta_2)\succeq (1+2LB)^{-1}H_t(\theta_2).
\]
\end{lemma}

\textit{Proof of Lemma \ref{lemma: param_concentration}.}  
We first prove that for all $t\ge t_0$,
\begin{equation}
\label{eq:proj_reduce_key_shifted}
\|\hat\theta_t-\theta_\star\|_2
\le
\frac{2(1+2LB)}{\sqrt{m_0\,\lambda_{\min}(V_t)}}
\ \|g_t(\hat\zeta_t)-g_t(\theta_\star)\|_{H_t(\theta_\star)^{-1}}.
\end{equation}
Since $\theta_\star\in\Theta$, the definition of $\hat\theta_t$ gives
\[
\|g_t(\hat\theta_t)-g_t(\hat\zeta_t)\|_{H_t(\hat\theta_t)^{-1}}
\le
\|g_t(\theta_\star)-g_t(\hat\zeta_t)\|_{H_t(\theta_\star)^{-1}}
=
\|g_t(\hat\zeta_t)-g_t(\theta_\star)\|_{H_t(\theta_\star)^{-1}}.
\]
By Lemma \ref{lemma: tranformation}, for $(\theta_1,\theta_2)=(\hat\theta_t,\hat\zeta_t)$,
\[
\|\hat\theta_t-\hat\zeta_t\|_{H_t(\hat\theta_t)}
\le (1+2LB)\,\|g_t(\hat\theta_t)-g_t(\hat\zeta_t)\|_{H_t(\hat\theta_t)^{-1}}.
\]
Moreover, since $\hat\theta_t\in\Theta$ and $\|z_s\|_2\le L$, we have $H_t(\hat\theta_t)\succeq m_0V_t$, hence
\begin{equation*}
\|\hat\theta_t-\hat\zeta_t\|_2
\le
\frac{1}{\sqrt{\lambda_{\min}(H_t(\hat\theta_t))}}\|\hat\theta_t-\hat\zeta_t\|_{H_t(\hat\theta_t)}
\le
\frac{1+2LB}{\sqrt{m_0\,\lambda_{\min}(V_t)}}
\|g_t(\hat\zeta_t)-g_t(\theta_\star)\|_{H_t(\theta_\star)^{-1}}.
\end{equation*}
Applying the same bound to $\|\hat\zeta_t-\theta_\star\|_2$ and using the triangle inequality yields \eqref{eq:proj_reduce_key_shifted}.
Using $\lambda_{\min}(V_t)\ge c t^{1/2}$ for $t\ge t_0$, \eqref{eq:proj_reduce_key_shifted} and
Lemma \ref{lemma: bound g} imply that with probability at least $1-\delta$, for all $t\ge t_0$,
\begin{equation}
\label{eq:hp_euclid_before_det}
\|\hat\theta_t-\theta_\star\|_2
\le
C_b\,t^{-1/4}\Big(\sqrt{\Psi_t(\delta)}+\sqrt{\lambda_t}\Big),
\end{equation}
for a constant $C_b>0$ absorbing $(1+2LB),m_0,B,c$.
Since $V_t=\sum_{s=1}^{t-1}z_s z_s^\top\succeq 0$,
\[
\mathrm{tr}(V_t)=\sum_{s=1}^{t-1}\|z_s\|_2^2 \le (t-1)L^2 \le tL^2.
\]
Let $\{\nu_i\}_{i=1}^d$ be the eigenvalues of $V_t$. Then,
\begin{equation}
\label{eq:det_trace_bound_shifted}
\det(V_t)=\prod_{i=1}^d \nu_i
\le \left(\frac{1}{d}\sum_{i=1}^d \nu_i\right)^d
=\left(\frac{\mathrm{tr}(V_t)}{d}\right)^d
\le \left(\frac{tL^2}{d}\right)^d.
\end{equation}
Therefore, with $\lambda(t)= ct^{1/2}$,
\begin{align}
\label{eq:logdet_simplified_shifted}
\log\frac{\det(V_t)}{\lambda(t)^d}
&\le
d\log\Big(\frac{tL^2}{d}\Big)-d\log(c t^{1/2})
=
d\log\Big(\frac{L^2}{cd}\Big)+\frac{d}{2}\log t
\ \le\ C_0 + \frac{d}{2}\log t,
\end{align}
where $C_0>0$ is a constant independent of $t$.

Moreover, for $t\ge t_0$, $\frac{\lambda_0}{\lambda(t)}=(t_0/t)^{1/2}\le 1$, hence
\begin{equation}
\label{eq:prefactor_simplify_shifted}
1+\frac{\lambda_0}{\lambda(t)}\le 2,\qquad
d\log\Big(1+\frac{\lambda_0}{\lambda(t)}\Big)\le d\log 2,\qquad
d\log\frac{\lambda(t)}{\lambda_0}=\frac{d}{2}\log\frac{t}{t_0}\le C_0+\frac{d}{2}\log t.
\end{equation}
Combining \eqref{eq:logdet_simplified_shifted} and \eqref{eq:prefactor_simplify_shifted} yields
\begin{equation}
\label{eq:Psi_simplified_shifted}
\Psi_t(\delta)\ \le\ C_1\Big(\log\frac{1}{\delta}+\frac{d}{2}\log t+1\Big)
\end{equation}
for some constant $C_1>0$ independent of $t,\delta$.

Fix $t\ge t_0$ and $\varepsilon>0$. Choose $\delta\in(0,1)$ such that
\begin{equation}
\label{eq:delta_choice_shifted}
C_b\,t^{-1/4}\sqrt{C_1\Big(\log\frac{1}{\delta}+\frac{d}{2}\log t+1\Big)}
\le \varepsilon.
\end{equation}
Then on the $1-\delta$ event of \eqref{eq:hp_euclid_before_det}--\eqref{eq:Psi_simplified_shifted}, we have
\[
\|\hat\theta_t-\theta_\star\|_2
\le
\varepsilon + C_b\,t^{-1/4}\sqrt{\lambda_t}.
\]

Therefore,
\[
\mathbb{P}\left(\|\hat\theta_t-\theta_\star\|_2 \ge \varepsilon + b_t\right)\le \delta.
\]

It remains to upper bound $\delta$ in closed form.
From \eqref{eq:delta_choice_shifted}, it suffices to require
\[
\log\frac{1}{\delta}
\ge
\frac{\varepsilon^2}{C_b^2 C_1}\,t^{1/2} - \frac{d}{2}\log t - 1.
\]
Let
$
c_1:=\frac{1}{C_b^2C_1},
c_2:=e.
$
If
$
c_2\,t^{d/2}\exp\big(-c_1\varepsilon^2 t^{1/2}\big)\ge 1,
$ then the desired bound is trivial since probabilities are always at most one. Otherwise,
$
c_2\,t^{d/2}\exp\big(-c_1\varepsilon^2 t^{1/2}\big)\in(0,1),
$ and we may choose
$
\delta
=
c_2\,t^{d/2}\exp\big(-c_1\varepsilon^2 t^{1/2}\big).
$
In either case, we obtain
\[
\mathbb{P}\left(\|\hat\theta_t-\theta_\star\|_2 \ge \varepsilon + b_t\right)
\le
c_2\,t^{d/2}\exp\big(-c_1\,\varepsilon^2 t^{1/2}\big).
\]
\qed

\section{Proof of Lemma \ref{lemma: param_consistence}}
Recall that $z_s=x_{i_s}-x_{j_s}$, 
\[
V_t=\sum_{s=1}^{t-1} z_s z_s^\top,\qquad 
H_t(\theta)=\sum_{s=1}^{t-1}\sigma'(\theta^\top z_s) z_s z_s^\top + \lambda_t I_d,
\]
and
\[
g_t(\theta)=\sum_{s=1}^{t-1}\sigma(\theta^\top z_s)z_s+\lambda_t\theta.
\]
Let the $\ell_2$-regularized maximum likelihood estimator be
\[
\hat\zeta_t\in\arg\max_{\theta\in\mathbb{R}^d}\Big\{\mathcal L_t(\theta)-\frac{\lambda_t}{2}\|\theta\|_2^2\Big\},
\]
and the projected estimator be
\[
\hat\theta_t=\argmin_{\theta\in\Theta}\left\|g_t(\theta)-g_t(\hat\zeta_t)\right\|_{H_t(\theta)^{-1}}.
\]
Since $\Theta$ is compact and $\|z_{s}\|_2\le L$, we have
\[
|\theta^\top z_s|\le \|\theta\|_2\|z_s\|_2\le BL,\qquad \forall \theta\in\Theta,\ s\ge 1.
\]
For the logistic function, $\sigma'(u)=\sigma(u)(1-\sigma(u))$ is continuous and strictly positive on any compact interval.
Hence, define
\[
m_0:=\min_{|u|\le BL}\sigma'(u)>0,\qquad M_0:=\max_{|u|\le BL}\sigma'(u)\le \frac14.
\]
It follows that for all $\theta\in\Theta$,
\begin{equation}
\label{eq:H_lower_upper}
m_0 V_t+\lambda_t I_d \ \preceq\ H_t(\theta)\ \preceq\ M_0 V_t+\lambda_t I_d .
\end{equation}
For all sufficiently large $t$, let $\tilde\theta_t$ denote an unregularized maximum likelihood estimator over $\mathbb{R}^d$.
Under the design-growth assumption $\Lambda_{\min}(t)=\lambda_{\min}(V_t)\ge c\sqrt{t}$ a.s. for all $t\ge t_0$, we also have
\[
\Lambda_{\max}(t)=\lambda_{\max}(V_t)\le \mathrm{tr}(V_t)=\sum_{s=1}^{t-1}\|z_s\|_2^2\le (t-1)L^2,
\]
so that
\[
\frac{\Lambda_{\min}(t)}{\log\Lambda_{\max}(t)}\ \ge\ \frac{c\sqrt{t}}{\log((t-1)L^2)}\ \xrightarrow[t\to\infty]{}\ \infty\qquad\text{a.s.}
\]

This is a standard sufficient condition for strong consistency of (quasi-)MLEs in GLMs under adaptive designs (Theorem 2 in \citealt{chen1999strong}), and yields
\begin{equation}
\label{eq:unreg_consistency}
\tilde\theta_t \xrightarrow{\mathrm{a.s.}} \theta_\star .
\end{equation}

We next show that the regularization bias is asymptotically negligible.
The first-order optimality conditions for $\hat\zeta_t$ and $\tilde\theta_t$ are
\begin{equation}
\label{eq:score_reg}
\sum_{s=1}^{t-1}\big(D_s-\sigma(\hat\zeta_t^\top z_s)\big)z_s-\lambda_t\hat\zeta_t=0,
\end{equation}
and
\begin{equation}
\label{eq:score_unreg}
\sum_{s=1}^{t-1}\big(D_s-\sigma(\tilde\theta_t^\top z_s)\big)z_s=0,
\end{equation}
respectively. Subtracting \eqref{eq:score_unreg} from \eqref{eq:score_reg} yields
\begin{equation}
\label{eq:score_diff}
\sum_{s=1}^{t-1}\Big(\sigma(\hat\zeta_t^\top z_s)-\sigma(\tilde\theta_t^\top z_s)\Big)z_s+\lambda_t\hat\zeta_t=0.
\end{equation}

By the mean-value theorem, for each $s$ there exists a (data-dependent) point
$\bar\theta_{s,t}$ on the line segment joining $\hat\zeta_t$ and $\tilde\theta_t$ such that
\[
\sigma(\hat\zeta_t^\top z_s)-\sigma(\tilde\theta_t^\top z_s)
=\sigma'(\bar\theta_{s,t}^\top z_s)\,z_s^\top(\hat\zeta_t-\tilde\theta_t).
\]
Plugging this into \eqref{eq:score_diff} gives
\begin{equation}
\label{eq:mvt_matrix}
\left(\sum_{s=1}^{t-1}\sigma'(\bar\theta_{s,t}^\top z_s)\,z_sz_s^\top\right)(\hat\zeta_t-\tilde\theta_t)
=-\lambda_t\hat\zeta_t.
\end{equation}

Moreover, since $\tilde\theta_t\to\theta_\star$ a.s. by \eqref{eq:unreg_consistency}, for any $r>0$ there exists an a.s. finite time $T_r$
such that $\tilde\theta_t\in\mathbb B(\theta_\star,r)$ for all $t\ge T_r$.
In particular, fixing any $r$ with $\mathbb B(\theta_\star,r)\subset \Theta$ (possible since $\theta_\star\in\mathrm{int}(\Theta)$),
we have $\|\tilde\theta_t\|_2\le B$ for all $t\ge T_r$.

Next, we show that $\hat\zeta_t$
 is also eventually bounded by $B$.
Since $\hat\zeta_t$ maximizes $\mathcal L_t(\theta)-\frac{\lambda_t}{2}\|\theta\|_2^2$, for any $\theta\in\mathbb R^d$,
\[
\mathcal L_t(\hat\zeta_t)-\frac{\lambda_t}{2}\|\hat\zeta_t\|_2^2
\ \ge\
\mathcal L_t(\theta)-\frac{\lambda_t}{2}\|\theta\|_2^2.
\]
Taking $\theta=\tilde\theta_t$ and using $\mathcal L_t(\hat\zeta_t)\le \mathcal L_t(\tilde\theta_t)$ (since $\tilde\theta_t$ maximizes $\mathcal L_t$) yields
\[
\mathcal L_t(\tilde\theta_t)-\frac{\lambda_t}{2}\|\hat\zeta_t\|_2^2
\ \ge\
\mathcal L_t(\tilde\theta_t)-\frac{\lambda_t}{2}\|\tilde\theta_t\|_2^2,
\]
and hence
\begin{equation}
\label{eq:zeta_norm_bound}
\|\hat\zeta_t\|_2\ \le\ \|\tilde\theta_t\|_2 \qquad \text{for all } t\ge 1.
\end{equation}
Therefore, $\|\hat\zeta_t\|_2\le B$ for all $t\ge T_r$ as well.

Consequently, for all $t\ge T_r$ and $s\le t-1$, since $\bar\theta_{s,t}$ is a convex combination of $\hat\zeta_t$ and $\tilde\theta_t$,
we have $\|\bar\theta_{s,t}\|_2\le \max\{\|\hat\zeta_t\|_2,\|\tilde\theta_t\|_2\}\le B$, and thus
$|\bar\theta_{s,t}^\top z_s|\le \|\bar\theta_{s,t}\|_2\|z_s\|_2\le BL$.
By the definition of $m_0$, we then have $\sigma'(\bar\theta_{s,t}^\top z_s)\ge m_0$ for all such $t$ and $s$, which implies
\begin{equation}
\label{eq:weighted_V_lower}
\sum_{s=1}^{t-1}\sigma'(\bar\theta_{s,t}^\top z_s)\,z_sz_s^\top
\ \succeq\
m_0 \sum_{s=1}^{t-1} z_sz_s^\top
=
m_0 V_t ,
\qquad \forall t\ge T_r.
\end{equation}
Plugging \eqref{eq:weighted_V_lower} into \eqref{eq:mvt_matrix} and taking operator norms yields, for all $t\ge \max\{T_r,t_0\}$,
\[
\|\hat\zeta_t-\tilde\theta_t\|_2
\le
\left\|\Big(\sum_{s=1}^{t-1}\sigma'(\bar\theta_{s,t}^\top z_s)\,z_sz_s^\top\Big)^{-1}\right\|_{\mathrm{op}}
\cdot \lambda_t\|\hat\zeta_t\|_2
\le
\frac{\lambda_t}{m_0\,\Lambda_{\min}(t)}\|\hat\zeta_t\|_2.
\]
Using \eqref{eq:zeta_norm_bound} and $\|\tilde\theta_t\|_2\le B$ for $t\ge T_r$, we further obtain
\begin{equation}
\label{eq:reg_bias_bound}
\|\hat\zeta_t-\tilde\theta_t\|_2
\le
\frac{\lambda_t}{m_0\,\Lambda_{\min}(t)}\|\tilde\theta_t\|_2
\le
\frac{B\,\lambda_t}{m_0\,\Lambda_{\min}(t)},
\qquad \forall t\ge T_r.
\end{equation}
Since $\lambda_t\to0$ and $\Lambda_{\min}(t)\ge c\sqrt t\to\infty$ a.s., we have
\[
\frac{\lambda_t}{\Lambda_{\min}(t)}\to0
\qquad\text{a.s.},
\]
and therefore
\eqref{eq:reg_bias_bound} implies
\begin{equation}
\label{eq:reg_bias}
\|\hat\zeta_t-\tilde\theta_t\|_2\xrightarrow{\mathrm{a.s.}}0.
\end{equation}
Combining \eqref{eq:unreg_consistency} and \eqref{eq:reg_bias} yields
\begin{equation}
\label{eq:zeta_consistency}
\hat\zeta_t \xrightarrow{\mathrm{a.s.}} \theta_\star.
\end{equation}
Since $\theta_\star\in\mathrm{int}(\Theta)$, there exists $r_0>0$ such that $\mathbb{B}(\theta_\star,r_0)\subset\Theta$.
By \eqref{eq:zeta_consistency}, almost surely there exists a (random) time $T$ such that for all $t\ge T$,
$\|\hat\zeta_t-\theta_\star\|_2<r_0$, and hence $\hat\zeta_t\in\Theta$.
For such $t$, the point $\theta=\hat\zeta_t$ is feasible in the projection program and achieves objective value $0$:
\[
\left\|g_t(\hat\zeta_t)-g_t(\hat\zeta_t)\right\|_{H_t(\hat\zeta_t)^{-1}}=0.
\]
Therefore, the minimizer must satisfy $\hat\theta_t=\hat\zeta_t$ for all $t\ge T$.
Combining with \eqref{eq:zeta_consistency} yields $\hat\theta_t\to\theta_\star$ almost surely.

We now prove the almost-sure convergence rate. Fix any $\beta\in(0,1/4)$ and let $\eta>0$ be arbitrary.
Applying Lemma \ref{lemma: param_concentration} with
$
\varepsilon_t:=\eta t^{-\beta},
$
we obtain for all $t\ge t_0$,
\[
\mathbb{P}\left(\|\hat\theta_t-\theta_\star\|_2 \ge \eta t^{-\beta}+b_t\right)
\le
c_2\, t^{d/2}\exp\big(-c_1\eta^2 t^{1/2-2\beta}\big).
\]
Since $\beta<1/4$, we have $1/2-2\beta>0$, and hence
\[
\sum_{t=1}^\infty t^{d/2}\exp\big(-c_1\eta^2 t^{1/2-2\beta}\big)<\infty.
\]
Therefore,
\[
\sum_{t=1}^\infty
\mathbb{P}\left(\|\hat\theta_t-\theta_\star\|_2 \ge \eta t^{-\beta}+b_t\right)
<\infty.
\]
By the first Borel--Cantelli lemma,
\[
\mathbb{P}\left(
\|\hat\theta_t-\theta_\star\|_2 \ge \eta t^{-\beta}+b_t
\ \text{i.o.}
\right)=0.
\]
Thus, almost surely, there exists a random finite time $T_\eta$ such that for all $t\ge T_\eta$,
\[
\|\hat\theta_t-\theta_\star\|_2 \le \eta t^{-\beta}+b_t.
\]
Multiplying both sides by $t^\beta$ gives
\[
t^\beta\|\hat\theta_t-\theta_\star\|_2
\le
\eta + t^\beta b_t
=
\eta + C_b\, t^{\beta-1/4}\sqrt{\lambda_t}.
\]
Since $\beta<1/4$, we have $t^{\beta-1/4}\to0$, and since $\lambda_t\to0$, we also have $\sqrt{\lambda_t}\to0$.
Hence
\[
t^\beta b_t
=
C_b\, t^{\beta-1/4}\sqrt{\lambda_t}\to0.
\]
Therefore,
\[
\limsup_{t\to\infty} t^\beta\|\hat\theta_t-\theta_\star\|_2 \le \eta
\qquad \text{a.s.}
\]
Since $\eta>0$ is arbitrary, we conclude that
\[
t^\beta\|\hat\theta_t-\theta_\star\|_2\to 0
\qquad \text{a.s.},
\]
that is,
\[
\|\hat\theta_t-\theta_\star\|_2=o(t^{-\beta})
\qquad \text{a.s.}
\]
for every $\beta\in(0,1/4)$.
\qed

\section{Proof of Lemma \ref{eq: beta_order}}
Recall that
\[
\beta(\delta,t)
=
\left(2(1+2LB)\big(\sqrt{m_0^{-1}\Psi_t(\delta)}+\sqrt{\lambda_t}B\big)\right)^2.
\]
Using \((u+v)^2\le 2u^2+2v^2\), we obtain
\[
\beta(\delta,t)
\le
8(1+2LB)^2\Big(m_0^{-1}\Psi_t(\delta)+B^2\lambda_t\Big).
\]
Since $\lambda_t\ge 0$ and \(\lambda_t\to0\), the sequence \(\{\lambda_t\}_{t\ge1}\) is bounded. Hence there
exists a constant \(\bar\lambda<\infty\) such that
\[
\lambda_t\le \bar\lambda,\qquad \forall t\ge1.
\]
Therefore,
\[
\beta(\delta,t)
\le
8(1+2LB)^2\Big(m_0^{-1}\Psi_t(\delta)+B^2\bar\lambda\Big).
\]

By \eqref{eq:Psi_simplified_shifted}, for all \(t\ge t_0\),
\[
\Psi_t(\delta)\le C_1\Big(\log\frac{1}{\delta}+d\log t+1\Big).
\]
Hence, for all \(t\ge t_0\),
\[
\beta(\delta,t)
\le
C_2\Big(\log\frac{1}{\delta}+d\log t+1\Big)
=
C_2\log\Big(\frac{e\,t^d}{\delta}\Big)
\]
for some constant \(C_2>0\). Therefore, the claim holds for all $t\ge t_0$ with
$
c_1=C_2, c_2=e, \alpha=d.
$
\qed

\section{Proof of Lemma \ref{lemma: unstructured-PAC-guarantee}}
Fix any true instance \(\mu\in\mathcal H\), and denote \(i^\star:=i^\star(\mu)\).
Recall that, in the unstructured policy space, the stopping rule is
\[
\tau
=
\inf\Bigg\{
t\ge 1:\ 
\inf_{\lambda\in \mathrm{Alt}(\hat\mu(t))}
\sum_{(i,j)\in\mathcal{S}} N_{ij}(t)\, d\big(\hat\mu(i,j;t),\lambda(i,j)\big)
\ge \rho(\delta,t)
\Bigg\},
\]
and the recommendation at stopping is \(\hat i_\tau=i^\star(\hat\mu(\tau))\). Define the error event
$
\mathcal E
:=
\{\tau<\infty,\ \hat i_\tau\neq i^\star\}.
$
On the event \(\mathcal E\), at time \(\tau\) we have
$
i^\star(\hat\mu(\tau))\neq i^\star,
$ which means exactly that \(\mu\in \mathrm{Alt}(\hat\mu(\tau))\). Therefore, by the definition of the stopping rule,
\[
\inf_{\lambda\in \mathrm{Alt}(\hat\mu(\tau))}
\sum_{(i,j)\in\mathcal{S}} N_{ij}(\tau)\, d\big(\hat\mu(i,j;\tau),\lambda(i,j)\big)
\ge \rho(\delta,\tau).
\]
Since \(\mu\in \mathrm{Alt}(\hat\mu(\tau))\), the infimum can be upper bounded by evaluating it at \(\lambda=\mu\). Hence,
\[
\sum_{(i,j)\in\mathcal{S}} N_{ij}(\tau)\, d\big(\hat\mu(i,j;\tau),\mu(i,j)\big)\ge \rho(\delta,\tau).
\]
It follows that
\[
\mathcal E
\subseteq
\Bigg\{
\exists t\ge 1:\ 
\sum_{(i,j)\in\mathcal{S}} N_{ij}(t)\, d\big(\hat\mu(i,j;t),\mu(i,j)\big)\ge \rho(\delta,t)
\Bigg\}.
\]
Taking probabilities and applying the union bound over time,
\begin{align*}
\mathbb P(\mathcal E)
&\le
\mathbb P\Bigg(
\exists t\ge 1:\ 
\sum_{(i,j)\in\mathcal{S}} N_{ij}(t)\, d\big(\hat\mu(i,j;t),\mu(i,j)\big)\ge \rho(\delta,t)
\Bigg) \\
&\le
\sum_{t=1}^\infty
\mathbb P\Bigg(
\sum_{(i,j)\in\mathcal{S}} N_{ij}(t)\, d\big(\hat\mu(i,j;t),\mu(i,j)\big)\ge \rho(\delta,t)
\Bigg).
\end{align*}

Now we invoke the self-normalized deviation inequality for one-dimensional exponential families (Theorem 2 of \citealt{magureanu2014lipschitz})
applied to all pairwise comparisons. 
 For every \(t\ge 1\),
\[
\mathbb P\Bigg(
\sum_{(i,j)\in\mathcal{S}} N_{ij}(t)\, d\big(\hat\mu(i,j;t),\mu(i,j)\big)\ge \rho(\delta,t)
\Bigg)
\le
e^{|\mathcal{S}|+1}
\left(
\frac{\lceil \rho(\delta,t)\log t\rceil\,\rho(\delta,t)}{|\mathcal{S}|}
\right)^{|\mathcal{S}|}
e^{-\rho(\delta,t)},
\]
Therefore,
\[
\mathbb P(\mathcal E)
\le
\sum_{t=1}^\infty
e^{|\mathcal{S}|+1}
\left(
\frac{\lceil \rho(\delta,t)\log t\rceil\,\rho(\delta,t)}{|\mathcal{S}|}
\right)^{|\mathcal{S}|}
e^{-\rho(\delta,t)}.
\]
Recall that
\begin{equation*}
   \rho(\delta,t) = \log\left(\frac{Ct^{2}\log(1/\delta)^{2|\mathcal{S}|+1}}{\delta}\right).
\end{equation*}
Hence, if the constant $C$ is chosen large enough such that
\begin{equation}
\label{eq: beta-condition-unstructured}
\sum_{t=1}^\infty
e^{|\mathcal{S}|+1}
\left(
\frac{\lceil \rho(\delta,t)\log t\rceil\,\rho(\delta,t)}{|\mathcal{S}|}
\right)^{|\mathcal{S}|}
e^{-\rho(\delta,t)}
\le \delta,
\end{equation}
then
\[
\mathbb P(\tau<\infty,\hat i_\tau\neq i^\star(\mu))
=
\mathbb P(\mathcal E)
\le \delta.
\]
\qed

\begin{lemma}[Lemma 17 in \citealt{garivier2016optimal}]
\label{lem:tracking_properties}
Consider the policy pair sampling rule in \eqref{eq: sampling rule}. Then there exists $t_1 \ge 1$ such that, for all $t \ge t_1$,
$
N_{ij}(t) \ge c^\prime\sqrt{t}-1,
\forall (i,j)\in \mathcal A_0.
$
Moreover, if $\omega_t \to \omega$ almost surely for some $\omega \in \Omega$, then for every $(i,j)\in \mathcal S$,
\[
\frac{N_{ij}(t)}{t} \to \omega_{ij}
\qquad \text{a.s.}
\]
\end{lemma}

\begin{lemma}[Lemma 8 in \citealt{jedra2020optimal}]
\label{lem:technical_log}
For any two constants \(c_1,c_2>0\) with \(c_2/c_1>1\), we have
\[
\inf\left\{t\in\mathbb N^\star:\ c_1 t\ge \log(c_2 t)\right\}
\le
\frac{1}{c_1}
\left(
\log\left(\frac{c_2 e}{c_1}\right)
+
\log\log\left(\frac{c_2}{c_1}\right)
\right).
\]
\end{lemma}

\section{Proof of Theorem \ref{thm: unstructured almost surely opt}}
For \(\nu\in\mathcal H\) and \(\omega\in\Omega\), define
\[
\Gamma(\nu,\omega)
:=
\inf_{\lambda\in\mathrm{Alt}(\nu)}
\sum_{(i,j)\in\mathcal{S}}\omega_{ij}\,d\big(\nu(i,j),\lambda(i,j)\big).
\]
Then
$
\mathcal T^\star(\mu)^{-1}
=
\sup_{\omega\in\Omega}\Gamma(\mu,\omega).
$
The above optimization problem admits a unique maximizer,
denoted by \(\omega^\star\).

We first establish the almost sure convergence of the empirical means and the sampling proportions.
Since \(\mathcal A_0=\mathcal S\), Lemma~\ref{lem:tracking_properties} implies that there exists
\(t_{\rm tr}\ge1\) such that for all \(t\ge t_{\rm tr}\),
\[
N_{ij}(t)\ge c^\prime\sqrt t-1,
\qquad \forall (i,j)\in\mathcal S.
\]
Hence \(N_{ij}(t)\to\infty\) for every \((i,j)\in\mathcal S\). By the strong law of large numbers,
\[
\hat\mu(i,j;t)\to\mu(i,j),
\qquad \forall (i,j)\in\mathcal S,
\qquad\text{a.s.}
\]

Since \(i^\star(\mu)\) is unique, there exists a neighborhood of \(\mu\) on which the best policy remains
unchanged. Moreover, for each \(i\neq i^\star(\mu)\), the sets \(\mathcal I_1(i)\) and
\(\mathcal I_2(i)\) remain unchanged when \(\nu\) is sufficiently close to \(\mu\). Therefore, by
\eqref{eq: simple opt}, the map \((\nu,\omega)\mapsto \Gamma(\nu,\omega)\) is continuous in a
neighborhood of \((\mu,\omega^\star)\). Since the maximizer of \(\Gamma(\mu,\cdot)\) over the compact set \(\Omega\)
is unique, by Berge’s maximum theorem, the plug-in optimizer \(\omega(\hat\mu(t))\) satisfies
\[
\omega(\hat\mu(t))\to\omega^\star
\qquad\text{a.s.}
\]
Applying Lemma~\ref{lem:tracking_properties} again, we obtain
\[
\frac{N_{ij}(t)}{t}\to\omega^\star_{ij},
\qquad \forall (i,j)\in\mathcal S,
\qquad\text{a.s.}
\]

Define the event
\[
\mathcal E
:=
\left\{
\forall (i,j)\in\mathcal S,\ \hat\mu(i,j;t)\to\mu(i,j)
\quad\text{and}\quad
\forall (i,j)\in\mathcal S,\ \frac{N_{ij}(t)}{t}\to\omega^\star_{ij}
\right\}.
\]
Then \(\mathbb P_\mu(\mathcal E)=1\).

On the event \(\mathcal E\), since \(i^\star(\mu)\) is unique, there exists \(t_{\rm id}\ge1\) such that
for all \(t\ge t_{\rm id}\),
$
i^\star(\hat\mu(t))=i^\star(\mu).
$
By definition,
\[
Z(t)
:=
\inf_{\lambda\in\mathrm{Alt}(\hat\mu(t))}
\sum_{i<j}N_{ij}(t)\,d\big(\hat\mu(i,j;t),\lambda(i,j)\big).
\]
Then for all \(t\ge t_{\rm id}\),
\[
Z(t)
=
t\,\Gamma\left(\hat\mu(t),\frac{N(t)}{t}\right),
\qquad
\frac{N(t)}{t}:=\left(\frac{N_{ij}(t)}{t}\right)_{i<j}.
\]

By the same local representation \eqref{eq: simple opt}, the map \((\nu,\omega)\mapsto \Gamma(\nu,\omega)\)
is continuous at \((\mu,\omega^\star)\). Therefore, on \(\mathcal E\),
\[
\Gamma\left(\hat\mu(t),\frac{N(t)}{t}\right)
\to
\Gamma(\mu,\omega^\star)
=
\mathcal T^\star(\mu)^{-1}.
\]
Hence, for every \(\varepsilon>0\), there exists \(t_\varepsilon\ge1\) such that for all
\(t\ge t_\varepsilon\),
\[
\Gamma\left(\hat\mu(t),\frac{N(t)}{t}\right)
\ge
(1-\varepsilon)\mathcal T^\star(\mu)^{-1}.
\]
Equivalently, for all \(t\ge \max\{t_\varepsilon,t_{\rm id}\}\),
\begin{equation}
\label{eq:Z_lower_unstructured_final}
Z(t)\ge t(1-\varepsilon)\mathcal T^\star(\mu)^{-1}.
\end{equation}

Since the stopping rule is
\[
\tau
=
\inf\left\{
t\ge 1:\ Z(t)>\beta(\delta,t)
\right\},
\]
it follows from \eqref{eq:Z_lower_unstructured_final} that, on \(\mathcal E\),
\[
\tau
\le
\max\{t_\varepsilon,t_{\rm id}\}
\vee
\inf\Bigl\{t\ge1:\ (1-\varepsilon)t\mathcal T^\star(\mu)^{-1}>\rho(\delta,t)\Bigr\}.
\]

Now, with
\[
\rho(\delta,t)
=
\log\left(\frac{Ct^{2}\log(1/\delta)^{2|\mathcal S|+1}}{\delta}\right),
\]
we have
\[
\rho(\delta,t)
=
\log C + 2\log t + (2|\mathcal S|+1)\log\log(1/\delta)+\log(1/\delta).
\]
Therefore,
\[
\tau
\le
\max\{t_\varepsilon,t_{\rm id}\}
\vee
\inf\Bigl\{t\ge 1:\ 
(1-\varepsilon)t\mathcal T^\star(\mu)^{-1}
>
\log C + 2\log t + (2|\mathcal S|+1)\log\log(1/\delta)+\log(1/\delta)
\Bigr\}.
\]
Since the right-hand side grows as \(\log(1/\delta)+\log t\), Lemma \ref{lem:technical_log} yields,
for sufficiently small \(\delta\),
\[
\tau
\lesssim
\max\left\{
t_\varepsilon,t_{\rm id},\,
\frac{1}{1-\varepsilon}\mathcal T^\star(\mu)\log\frac1\delta
\right\}.
\]
Hence, on \(\mathcal E\),
\[
\limsup_{\delta\rightarrow 0}\frac{\tau}{\log(1/\delta)}
\lesssim
\frac{1}{1-\varepsilon}\mathcal T^\star(\mu).
\]
Letting \(\varepsilon\rightarrow0\), we conclude that
\[
\limsup_{\delta\rightarrow0}\frac{\tau}{\log(1/\delta)}
\lesssim
\mathcal T^\star(\mu)
\qquad\text{on }\mathcal E.
\]
Since \(\mathbb P_\mu(\mathcal E)=1\), the desired result follows.
Finally, the above bound implies \(\tau<\infty\) on \(\mathcal E\), and therefore
$
\mathbb P_{\mu}(\tau<\infty)=1.
$
\qed

\section{Proof of Lemma \ref{lemma: opt property}}
Because the policy set is finite and the latent reward $r_i$ of each policy is continuous in $\theta$,
the uniqueness of $i^\star(\theta_\star)$ implies that there exists $\varepsilon_0>0$ such that
\[
i^\star(\theta)=i^\star(\theta_\star),
\qquad
\forall \theta:\ \|\theta-\theta_\star\|_2<\varepsilon_0.
\]
Hence, for all $(\theta,\omega)$ with $\|\theta-\theta_\star\|_2<\varepsilon_0$,
\[
\Psi(\theta,\omega)=\min_{i\neq i^\star(\theta_\star)}\phi_i(\theta,\omega),
\]
where
\[
\phi_i(\theta,\omega)
:=
\frac12
\frac{(\theta^\top z_{i\,i^\star(\theta_\star)})^2}
{z_{i\,i^\star(\theta_\star)}^\top H(\theta,\omega)^{-1} z_{i\,i^\star(\theta_\star)}}.
\]

We first prove the continuity of $\Psi(\theta,\omega)$.
The map $(\theta,\omega)\mapsto H(\theta,\omega)$ is continuous, since it is a finite sum of continuous terms.
Because matrix inversion is continuous on the cone of positive definite matrices and
$H(\theta,\omega)\succ 0$ for all $(\theta,\omega)$ under consideration, it follows that
$(\theta,\omega)\mapsto H(\theta,\omega)^{-1}$ is continuous.
Therefore, for each fixed $i\neq i^\star(\theta_\star)$, both the numerator
$
(\theta^\top z_{i\,i^\star(\theta_\star)})^2
$
and the denominator
$
z_{i\,i^\star(\theta_\star)}^\top H(\theta,\omega)^{-1} z_{i\,i^\star(\theta_\star)}
$
are continuous, and the denominator is strictly positive. Hence each $\phi_i(\theta,\omega)$ is continuous.
Since the minimum is taken over a finite set, $\Psi(\theta,\omega)$ is continuous in both $\theta$ and $\omega$.

We next prove the local uniform Lipschitz continuity in $\theta$ around $\theta_\star$.
Fix $r_0\in(0,\varepsilon_0)$ and define
\[
\Theta_0:=\{\theta:\ \|\theta-\theta_\star\|_2\le r_0\}.
\]
Since $\Theta_0\times\Omega$ is compact and $H(\theta,\omega)$ is continuous and positive definite on this set,
there exist constants $0<m_H\le M_H<\infty$ such that
\begin{equation}
\label{eq:H_uniform_bounds_opt_property}
m_H I_d \preceq H(\theta,\omega)\preceq M_H I_d,
\qquad
\forall (\theta,\omega)\in\Theta_0\times\Omega.
\end{equation}
In particular,
\begin{equation}
\label{eq:Hinv_uniform_bounds_opt_property}
\|H(\theta,\omega)^{-1}\|_{\mathrm{op}}\le m_H^{-1},
\qquad
\forall (\theta,\omega)\in\Theta_0\times\Omega.
\end{equation}

Let
\[
R:=\sup_{\theta\in\Theta_0}\|\theta\|_2<\infty,
\qquad
L:=\max_{(a,b)}\|z_{ab}\|_2<\infty.
\]
Since $\sigma'$ is continuously differentiable, it is Lipschitz on the compact interval $[-RL,RL]$.
Let $L_{\sigma'}$ denote a Lipschitz constant of $\sigma'$ on this interval. Then for any
$\theta,\theta'\in\Theta_0$ and any $\omega\in\Omega$,
\begin{align*}
\|H(\theta,\omega)-H(\theta',\omega)\|_{\mathrm{op}}
&=
\left\|
\sum_{(a,b)}\omega_{ab}\bigl(\sigma'(\theta^\top z_{ab})-\sigma'(\theta'^\top z_{ab})\bigr)z_{ab}z_{ab}^\top
\right\|_{\mathrm{op}} \\
&\le
\sum_{(a,b)}\omega_{ab}
\bigl|
\sigma'(\theta^\top z_{ab})-\sigma'(\theta'^\top z_{ab})
\bigr|
\,\|z_{ab}z_{ab}^\top\|_{\mathrm{op}} \\
&\le
\sum_{(a,b)}\omega_{ab}
L_{\sigma'}\,\|z_{ab}\|_2\,\|\theta-\theta'\|_2\,\|z_{ab}\|_2^2 \\
&\le
L_{\sigma'} L^3 \|\theta-\theta'\|_2.
\end{align*}
Thus $H(\theta,\omega)$ is Lipschitz in $\theta$, uniformly over $\omega\in\Omega$.

Using the identity
\[
A^{-1}-B^{-1}=A^{-1}(B-A)B^{-1},
\]
together with \eqref{eq:Hinv_uniform_bounds_opt_property}, we obtain
\begin{align}
\|H(\theta,\omega)^{-1}-H(\theta',\omega)^{-1}\|_{\mathrm{op}}
&\le
\|H(\theta,\omega)^{-1}\|_{\mathrm{op}}
\|H(\theta,\omega)-H(\theta',\omega)\|_{\mathrm{op}}
\|H(\theta',\omega)^{-1}\|_{\mathrm{op}} \notag\\
&\le
m_H^{-2}L_{\sigma'}L^3\,\|\theta-\theta'\|_2.
\label{eq:Hinv_lipschitz_opt_property}
\end{align}

Fix $i\neq i^\star(\theta_\star)$, and define
\[
a_i(\theta):=\frac12(\theta^\top z_{i\,i^\star(\theta_\star)})^2,
\qquad
b_i(\theta,\omega):=
z_{i\,i^\star(\theta_\star)}^\top H(\theta,\omega)^{-1}z_{i\,i^\star(\theta_\star)}.
\]
Then
\[
\phi_i(\theta,\omega)=\frac{a_i(\theta)}{b_i(\theta,\omega)}.
\]

For the numerator, for any $\theta,\theta'\in\Theta_0$,
\begin{align*}
|a_i(\theta)-a_i(\theta')|
&=
\frac12\left|(\theta^\top z_{i\,i^\star(\theta_\star)})^2-(\theta'^\top z_{i\,i^\star(\theta_\star)})^2\right| \\
&\le
\frac12\bigl(|\theta^\top z_{i\,i^\star(\theta_\star)}|+|\theta'^\top z_{i\,i^\star(\theta_\star)}|\bigr)
\bigl|(\theta-\theta')^\top z_{i\,i^\star(\theta_\star)}\bigr| \\
&\le
RL^2\,\|\theta-\theta'\|_2.
\end{align*}
Hence $a_i(\theta)$ is Lipschitz in $\theta$ on $\Theta_0$.

For the denominator, \eqref{eq:H_uniform_bounds_opt_property} implies
\[
b_i(\theta,\omega)
=
z_{i\,i^\star(\theta_\star)}^\top H(\theta,\omega)^{-1}z_{i\,i^\star(\theta_\star)}
\ge
\frac{\|z_{i\,i^\star(\theta_\star)}\|_2^2}{M_H},
\qquad
\forall (\theta,\omega)\in\Theta_0\times\Omega.
\]
Thus $b_i(\theta,\omega)$ is uniformly bounded away from zero. Moreover, by \eqref{eq:Hinv_lipschitz_opt_property},
\begin{align*}
|b_i(\theta,\omega)-b_i(\theta',\omega)|
&=
\left|
z_{i\,i^\star(\theta_\star)}^\top
\bigl(H(\theta,\omega)^{-1}-H(\theta',\omega)^{-1}\bigr)
z_{i\,i^\star(\theta_\star)}
\right| \\
&\le
\|z_{i\,i^\star(\theta_\star)}\|_2^2
\,
\|H(\theta,\omega)^{-1}-H(\theta',\omega)^{-1}\|_{\mathrm{op}} \\
&\le
\|z_{i\,i^\star(\theta_\star)}\|_2^2\,m_H^{-2}L_{\sigma'}L^3\,\|\theta-\theta'\|_2.
\end{align*}
Hence $b_i(\theta,\omega)$ is Lipschitz in $\theta$, uniformly over $\omega\in\Omega$.

Since $a_i(\theta)$ is bounded on $\Theta_0$ and $b_i(\theta,\omega)$ is uniformly bounded away from zero,
it follows that there exists a constant $L_i>0$ such that
\[
|\phi_i(\theta,\omega)-\phi_i(\theta',\omega)|
\le
L_i\|\theta-\theta'\|_2,
\qquad
\forall \theta,\theta'\in\Theta_0,\ \forall \omega\in\Omega.
\]
Because the set $\{i:\ i\neq i^\star(\theta_\star)\}$ is finite, letting
$
L_{\Psi}:=\max_{i\neq i^\star(\theta_\star)}L_i,
$
we obtain for all $\theta\in\Theta_0$ and all $\omega\in\Omega$,
\[
|\phi_i(\theta,\omega)-\phi_i(\theta_\star,\omega)|
\le
L_{\Psi}\|\theta-\theta_\star\|_2,
\qquad
\forall i\neq i^\star(\theta_\star).
\]
Using the elementary inequality
\[
\left|\min_i x_i-\min_i y_i\right|
\le
\max_i |x_i-y_i|,
\]
we conclude that
\begin{align*}
\big|
\Psi(\theta,\omega)-\Psi(\theta_\star,\omega)
\big|
&=
\left|
\min_{i\neq i^\star(\theta_\star)}\phi_i(\theta,\omega)
-
\min_{i\neq i^\star(\theta_\star)}\phi_i(\theta_\star,\omega)
\right| \\
&\le
\max_{i\neq i^\star(\theta_\star)}
|\phi_i(\theta,\omega)-\phi_i(\theta_\star,\omega)| \\
&\le
L_{\Psi}\|\theta-\theta_\star\|_2.
\end{align*}
Taking the supremum over $\omega\in\Omega$ proves the local uniform Lipschitz property.

Finally, the claim on $C^\star(\theta)$ is a consequence of Berge's maximum theorem. Since
$\Psi$ is continuous in $(\theta,\omega)$, and $\Omega$ is non-empty and compact. Hence,
for each fixed $\theta$, the maximizer set $C^\star(\theta)$ is non-empty and compact.
Moreover, for each fixed $\theta$, the map $\omega\mapsto\Psi(\theta,\omega)$ is concave on $\Omega$,
since it can be expressed as the infimum of linear functions in $\omega$ (by replacing the $d(\mu(i,j),\lambda(i,j))$ in \eqref
{eq:relaxed_closed} with local regime). Therefore,
the optimal solution set $C^\star(\theta)$ is convex. This proves the result.
\qed

\section{Proof of Theorem \ref{prop: ratio_convergence}}
By Lemma~\ref{lem:tracking_properties}, there exists $t_1$ such that for all $t\ge t_1$,
\[
N_{ij}(t)\ge c^\prime\sqrt t-1,
\qquad \forall (i,j)\in\mathcal A_0.
\]
Since the set $\mathcal A_0$ is chosen so that
$
\mathrm{span}\{z_{ij}:(i,j)\in\mathcal A_0\}=\mathbb R^d,
$
the matrix
\[
G_0:=\sum_{(i,j)\in\mathcal A_0} z_{ij}z_{ij}^\top
\]
is positive definite. Therefore,
\[
V_t
=
\sum_{(i,j)\in\mathcal S}N_{ij}(t)z_{ij}z_{ij}^\top
\ \succeq\
\sum_{(i,j)\in\mathcal A_0}N_{ij}(t)z_{ij}z_{ij}^\top
\ \succeq\
(c^\prime\sqrt t-1)\sum_{(i,j)\in\mathcal A_0} z_{ij}z_{ij}^\top
=
(c^\prime\sqrt t-1)G_0.
\]
Hence, there exists some constant $c>0$ such that 
\[
\Lambda_{\min}(t)=\lambda_{\min}(V_t)\ge (c^\prime\sqrt t-1)\lambda_{\min}(G_0)\ge c\sqrt{t}.
\]
Thus, the condition of Lemma~\ref{lemma: param_consistence} is satisfied.

By Lemma~\ref{lemma: param_consistence}, for every $\beta\in(0,1/4)$,
$
\|\hat\theta_t-\theta_\star\|_2=o(t^{-\beta})
$ $\text{a.s.}
$. Choose any $\beta\in(1/8,1/4)$. Since $\gamma_t=t^{-1/8}$, it follows that
\begin{equation}
\label{eq:theta_over_gamma_checked}
\|\hat\theta_t-\theta_\star\|_2=o(\gamma_t)
\qquad\text{a.s.}
\end{equation}
Since $i^\star(\theta_\star)$ is unique and the policy set is finite, there exists $r^\prime_0>0$ such that
\[
i^\star(\theta)=i^\star(\theta_\star)
\qquad
\forall \theta:\ \|\theta-\theta_\star\|_2\le r^\prime_0.
\]
Let $\bar r:=\min\{r_0,r^\prime_0\}$. By \eqref{eq:theta_over_gamma_checked}, on an almost sure event we have
\[
\|\hat\theta_t-\theta_\star\|_2\le \bar r
\qquad
\text{for all sufficiently large }t.
\]
Hence, by Lemma \ref{lemma: opt property}, for all sufficiently large $t$,
\[
\delta_t
:=
\sup_{\omega\in\Omega}
\big|
\Psi(\hat\theta_t,\omega)-\Psi(\theta_\star,\omega)
\big|
\le
L_\Psi \|\hat\theta_t-\theta_\star\|_2 .
\]
Combining this with \eqref{eq:theta_over_gamma_checked} yields
\begin{equation}
\label{eq:delta_over_gamma_checked}
\delta_t=o(\gamma_t)
\qquad\text{a.s.}
\end{equation}
In particular,
\[
\delta_t\to 0,\quad\text{a.s.}
\qquad
\gamma_t\to 0.
\qquad
\]
Because $C^\star(\theta_\star)$ is nonempty, compact, and convex, and $\|\omega\|_2^2$ is strictly convex,
the minimizer
\[
\omega^\dagger
:=
\argmin_{\omega\in C^\star(\theta_\star)} \|\omega\|^2_2
\]
exists and is unique.

Define
\[
F_t(\omega):=\Psi(\hat\theta_t,\omega)-\frac{\gamma_t}{2}\|\omega\|_2^2.
\]
By definition, $\omega_t$ maximizes $F_t$ over $\Omega$, so for every $\omega\in\Omega$,
\begin{equation}
\label{eq:Ft_optimality_checked}
\Psi(\hat\theta_t,\omega_t)-\frac{\gamma_t}{2}\|\omega_t\|_2^2
\ge
\Psi(\hat\theta_t,\omega)-\frac{\gamma_t}{2}\|\omega\|_2^2.
\end{equation}

We first show that every accumulation point of $\{\omega_t\}$ belongs to $C^\star(\theta_\star)$.
Since $\Omega$ is compact, let $\omega_{t_k}\to\bar\omega$ be any convergent subsequence.

Fix any $\omega\in\Omega$. Applying \eqref{eq:Ft_optimality_checked} at time $t_k$, we get
\[
\Psi(\hat\theta_{t_k},\omega_{t_k})-\frac{\gamma_{t_k}}{2}\|\omega_{t_k}\|_2^2
\ge
\Psi(\hat\theta_{t_k},\omega)-\frac{\gamma_{t_k}}{2}\|\omega\|_2^2.
\]
By the definition of $\delta_t$,
\[
\Psi(\hat\theta_{t_k},\omega_{t_k})
\le
\Psi(\theta_\star,\omega_{t_k})+\delta_{t_k},
\qquad
\Psi(\hat\theta_{t_k},\omega)
\ge
\Psi(\theta_\star,\omega)-\delta_{t_k}.
\]
Therefore,
\[
\Psi(\theta_\star,\omega_{t_k})+\delta_{t_k}-\frac{\gamma_{t_k}}{2}\|\omega_{t_k}\|_2^2
\ge
\Psi(\theta_\star,\omega)-\delta_{t_k}-\frac{\gamma_{t_k}}{2}\|\omega\|_2^2,
\]
that is,
\[
\Psi(\theta_\star,\omega_{t_k})
\ge
\Psi(\theta_\star,\omega)
+\frac{\gamma_{t_k}}{2}\bigl(\|\omega_{t_k}\|_2^2-\|\omega\|_2^2\bigr)
-2\delta_{t_k}.
\]
Letting $k\to\infty$, and using $\omega_{t_k}\to\bar\omega$, continuity of $\Psi(\theta_\star,\cdot)$,
and $\gamma_{t_k}\to0$, $\delta_{t_k}\to0$, we obtain
\[
\Psi(\theta_\star,\bar\omega)\ge \Psi(\theta_\star,\omega),
\qquad \forall \omega\in\Omega.
\]
Hence
$
\bar\omega\in C^\star(\theta_\star).
$

Next, we show that every accumulation point must equal $\omega^\dagger$. 
Taking $\omega=\omega^\dagger$ in \eqref{eq:Ft_optimality_checked} gives
\[
\Psi(\hat\theta_t,\omega_t)-\frac{\gamma_t}{2}\|\omega_t\|_2^2
\ge
\Psi(\hat\theta_t,\omega^\dagger)-\frac{\gamma_t}{2}\|\omega^\dagger\|_2^2,
\]
or equivalently,
\begin{equation}
\label{eq:norm_compare_start_checked}
\Psi(\hat\theta_t,\omega_t)-\Psi(\hat\theta_t,\omega^\dagger)
\ge
\frac{\gamma_t}{2}\bigl(\|\omega_t\|_2^2-\|\omega^\dagger\|_2^2\bigr).
\end{equation}
On the other hand, since $\omega^\dagger\in C^\star(\theta_\star)$ and $\omega_t\in\Omega$,
\[
\Psi(\theta_\star,\omega_t)\le \Psi(\theta_\star,\omega^\dagger).
\]
Using the definition of $\delta_t$,
\[
\Psi(\hat\theta_t,\omega_t)
\le
\Psi(\theta_\star,\omega_t)+\delta_t
\le
\Psi(\theta_\star,\omega^\dagger)+\delta_t
\le
\Psi(\hat\theta_t,\omega^\dagger)+2\delta_t.
\]
Thus,
\[
\Psi(\hat\theta_t,\omega_t)-\Psi(\hat\theta_t,\omega^\dagger)\le 2\delta_t.
\]
Combining this with \eqref{eq:norm_compare_start_checked}, we obtain
\[
\frac{\gamma_t}{2}\bigl(\|\omega_t\|_2^2-\|\omega^\dagger\|_2^2\bigr)\le 2\delta_t,
\]
hence
\[
\|\omega_t\|_2^2
\le
\|\omega^\dagger\|_2^2+\frac{4\delta_t}{\gamma_t}.
\]
By \eqref{eq:delta_over_gamma_checked},
\begin{equation}
\label{eq:norm_limsup_checked}
\limsup_{t\to\infty}\|\omega_t\|^2_2
\le
\|\omega^\dagger\|^2_2
\qquad\text{a.s.}
\end{equation}
Let $\bar\omega$ be any accumulation point of $\{\omega_t\}$. We already know that $\bar\omega\in C^\star(\theta_\star)$.
Passing to the limit along a subsequence realizing $\bar\omega$ in \eqref{eq:norm_limsup_checked}, we get
\[
\|\bar\omega\|^2_2\le \|\omega^\dagger\|^2_2.
\]
Since $\omega^\dagger$ is the unique minimum-norm element of $C^\star(\theta_\star)$, this forces
$
\bar\omega=\omega^\dagger.
$
Hence every accumulation point of $\{\omega_t\}$ equals $\omega^\dagger$, and therefore
\begin{equation}
\label{eq:omega_t_to_omega_dagger_checked}
\omega_t\to\omega^\dagger
\qquad\text{a.s.}
\end{equation}

Therefore, Lemma~\ref{lem:tracking_properties} implies that
\[
\mathbb{P}\left(
\forall (i,j)\in\mathcal S,\ 
\lim_{t\to\infty}\frac{N_{ij}(t)}{t}
=
\omega^\dagger_{ij}
\right)=1,
\]
This completes the proof.
\qed

\section{Proof of Lemma \ref{lemma: delta-guarantee}}
Let $\mathcal E_\delta$ be the event on which, for all $t\ge t_0$,
\begin{equation}
\label{eq:good_event_score}
\|g_t(\hat\zeta_t)-g_t(\theta_\star)\|_{H_t(\theta_\star)^{-1}}
\le
\sqrt{m_0^{-1}\Psi_t(\delta)}+\sqrt{\lambda_t}B.
\end{equation}
By Lemma~\ref{lemma: bound g}, $\mathbb P(\mathcal E_\delta)\ge 1-\delta$.

We next prove that on the event $\mathcal E_\delta$, for all $t\ge t_0$,
\begin{equation}
\label{eq:Hnorm_closeness}
\|\hat\theta_t-\theta_\star\|_{H_t(\hat\theta_t)}
\le
2(1+2LB)\left(\sqrt{m_0^{-1}\Psi_t(\delta)}+\sqrt{\lambda_t}B\right)
\end{equation}
Since $\theta_\star\in\Theta$, the definition of $\hat\theta_t$ gives
\[
\|g_t(\hat\theta_t)-g_t(\hat\zeta_t)\|_{H_t(\hat\theta_t)^{-1}}
\le
\|g_t(\theta_\star)-g_t(\hat\zeta_t)\|_{H_t(\theta_\star)^{-1}}
=
\|g_t(\hat\zeta_t)-g_t(\theta_\star)\|_{H_t(\theta_\star)^{-1}}.
\]
Moreover, by Lemma \ref{lemma: tranformation},
\begin{equation}
\label{eq:Hnorm_vs_obj_LB}
\|\hat\theta_t-\hat\zeta_t\|_{H_t(\hat\theta_t)}
\le
(1+2LB)\,\|g_t(\hat\theta_t)-g_t(\hat\zeta_t)\|_{H_t(\hat\theta_t)^{-1}}.
\end{equation}
Applying the same comparison to the pair $(\hat\zeta_t,\theta_\star)$ gives
\begin{equation}
\label{eq:Hnorm_zeta_star_LB}
\|\hat\zeta_t-\theta_\star\|_{H_t(\hat\theta_t)}
\le
(1+2LB)\,\|g_t(\hat\zeta_t)-g_t(\theta_\star)\|_{H_t(\theta_\star)^{-1}}.
\end{equation}
Combining \eqref{eq:Hnorm_vs_obj_LB}-\eqref{eq:Hnorm_zeta_star_LB} with the triangle inequality
\[
\|\hat\theta_t-\theta_\star\|_{H_t(\hat\theta_t)}
\le
\|\hat\theta_t-\hat\zeta_t\|_{H_t(\hat\theta_t)}+\|\hat\zeta_t-\theta_\star\|_{H_t(\hat\theta_t)}
\]
yields
\[
\|\hat\theta_t-\theta_\star\|_{H_t(\hat\theta_t)}
\le
2(1+2LB)\,\|g_t(\hat\zeta_t)-g_t(\theta_\star)\|_{H_t(\theta_\star)^{-1}}
\le
2(1+2LB)\left(\sqrt{m_0^{-1}\Psi_t(\delta)}+\sqrt{\lambda_t}B\right),
\]
where the last inequality uses \eqref{eq:good_event_score}. 
Fix any $t\ge t_0$. Suppose $i^\star(\hat\mu(t))\neq i^\star(\mu)$, hence
\begin{equation}
\label{eq:ZLB_le_one_choice}
Z(t) =  \min_{i\neq i^\star(\hat{\mu}(t))}
    \frac{1}{2}\,
    \frac{\big(\hat{\theta}^\top_t z_{i\,i^\star(\hat{\mu}(t))}\big)^2}{\big\|z_{i\,i^\star(\hat{\mu}(t))}\big\|^2_{H_t(\hat{\theta}_t)^{-1}}} \le
\frac12\,\frac{(\hat{\theta}^\top_tz_{i^\star(\mu)\,i^\star(\hat{\mu}(t))})^2}{\big\| z_{i^\star(\mu)\,i^\star(\hat{\mu}(t))} \big\|^2_{H_t(\hat{\theta}_t)^{-1}}}.
\end{equation}
We now bound the right-hand side by $\frac12\|\hat\theta_t-\theta_\star\|_{H_t(\hat\theta_t)}^2$.

First, by optimality of $i^\star(\mu)$, we have
$\theta_\star^\top (x_{i^\star(\mu)}-x_{i^\star(\hat{\mu}(t))})\ge 0$, i.e.,
\begin{equation}
\label{eq:true_pref_sign}
\theta_\star^\top z_{i^\star(\mu)i^\star(\hat{\mu}(t))}\ge 0.
\end{equation}

Therefore,
\[0<
-\hat\theta_t^\top z_{i^\star(\mu)i^\star(\hat{\mu}(t))}
\le
(\theta_\star-\hat\theta_t)^\top z_{i^\star(\mu)i^\star(\hat{\mu}(t))}
\le
|(\hat\theta_t-\theta_\star)^\top z_{i^\star(\mu)i^\star(\hat{\mu}(t))}|.
\]
Second, by Cauchy--Schwarz in the metric induced by $H_t(\hat\theta_t)$,
\[
|(\hat\theta_t-\theta_\star)^\top z_{i^\star(\mu)i^\star(\hat{\mu}(t))}|
\le
\|\hat\theta_t-\theta_\star\|_{H_t(\hat\theta_t)}\ \|z_{i^\star(\mu)i^\star(\hat{\mu}(t))}\|_{H_t(\hat\theta_t)^{-1}}.
\]
Then, 
\begin{equation*}
\frac12\,\frac{(\hat{\theta}^\top_tz_{i^\star(\mu)\,i^\star(\hat{\mu}(t))})^2}{\big\| z_{i^\star(\mu)\,i^\star(\hat{\mu}(t))} \big\|^2_{H_t(\hat{\theta}_t)^{-1}}} \le \frac12\|\hat\theta_t-\theta_\star\|_{H_t(\hat\theta_t)}^2.
\end{equation*}
Plugging this into \eqref{eq:ZLB_le_one_choice} yields the key implication:
\begin{equation}
\label{eq:wrong_implies_ZLB_small}
i^\star(\hat\mu(t))\neq i^\star(\mu)
\quad\Longrightarrow\quad
Z(t)\le \frac12\|\hat\theta_t-\theta_\star\|_{H_t(\hat\theta_t)}^2.
\end{equation}
On the event $\mathcal E_\delta$, for all $t\ge t_0$, \eqref{eq:Hnorm_closeness} gives
\[
\frac12\|\hat\theta_t-\theta_\star\|_{H_t(\hat\theta_t)}^2
\le
\left(2(1+2LB)\big(\sqrt{m_0^{-1}\Psi_t(\delta)}
+
\sqrt{\lambda_t}B\big)\right)^2
=
\beta(\delta,t).
\]
Therefore, by \eqref{eq:wrong_implies_ZLB_small}, on $\mathcal E_\delta$ we have:
if $i^\star(\hat\mu(t))\neq i^\star(\mu)$ then $Z(t)\le\beta(\delta,t)$.
In particular, at the stopping time $\tau$,
the event $\{\tau<\infty,\ i^\star(\hat\mu(t))\neq i^\star(\mu)\}$
cannot occur on $\mathcal E_\delta$ because the stopping condition requires
$Z(t)>\beta(\delta,t)$.

Hence,
\[
\mathbb{P}\left(\tau<\infty,\ i^\star(\hat\mu(\tau))\neq i^\star(\mu)\right)
\le
\mathbb{P}(\mathcal E_\delta^c)
\le
\delta,
\]
which proves the lemma.
\qed

\section{Proof of Theorem \ref{thm: almost_sure_optimality}}
Let
\[
\mathcal E
:=
\left\{
\hat\theta_t\to\theta_\star
\quad\text{and}\quad
\forall (i,j)\in\mathcal S,\ \frac{N_{ij}(t)}{t}\to \omega^\dagger_{ij}
\right\}.
\]
By Lemma~\ref{lemma: param_consistence} and
Theorem~\ref{prop: ratio_convergence}, we have \(\mathbb P(\mathcal E)=1\).

On the event \(\mathcal E\), since
\[
\hat\theta_t\to\theta_\star,
\qquad
\frac{N(t)}{t}:=\left(\frac{N_{ij}(t)}{t}\right)_{(i,j)\in\mathcal S}
\to \omega^\dagger,
\]
and \(\Psi(\theta,\omega)\) is continuous in \((\theta,\omega)\) on a neighborhood
of \(\theta_\star\times\Omega\) by Lemma~\ref{lemma: opt property}, it follows that
\[
\Psi\left(\hat\theta_t,\frac{N(t)}{t}\right)
\to
\Psi(\theta_\star,\omega^\dagger).
\]
Since \(\omega^\dagger\in C^\star(\theta_\star)\), we have
\[
\Psi(\theta_\star,\omega^\dagger)
=
\max_{\omega\in\Omega}\Psi(\theta_\star,\omega)
=
\mathcal U^\star(\mu)^{-1}.
\]
Therefore, for every \(\varepsilon>0\), there exists \(t_\varepsilon\ge1\) such that for all
\(t\ge t_\varepsilon\),
\[
\Psi\left(\hat\theta_t,\frac{N(t)}{t}\right)
\ge
(1-\varepsilon)\mathcal U^\star(\mu)^{-1}.
\]

Since \(i^\star(\mu)\) is unique and \(\hat\theta_t\to\theta_\star\), there exists
\(t_{\rm id}\ge1\) such that for all \(t\ge t_{\rm id}\),
\[
i^\star(\hat\theta_t)=i^\star(\mu).
\]
Hence, for all \(t\ge t_{\rm id}\), by the definition of \(Z(t)\),
\[
Z(t)
=
\min_{i\neq i^\star(\hat\theta_t)}
\frac{1}{2}\,
\frac{\big(\hat{\theta}_t^\top z_{i\,i^\star(\hat\theta_t)}\big)^2}
{\big\|z_{i\,i^\star(\hat\theta_t)}\big\|^2_{H_t(\hat\theta_t)^{-1}}}.
\]
Moreover, since
\[
H_t(\hat\theta_t)
=
\sum_{(i,j)\in\mathcal S}N_{ij}(t)\,\sigma'\bigl(\hat\theta_t^\top z_{ij}\bigr)\,z_{ij}z_{ij}^\top
+\lambda_t I_d,
\]
we have that $
H_t(\hat\theta_t)=t\,\widetilde H_t
$, where
\[
\widetilde H_t
:=
\sum_{(i,j)\in\mathcal S}\frac{N_{ij}(t)}{t}\,\sigma'\bigl(\hat\theta_t^\top z_{ij}\bigr)\,z_{ij}z_{ij}^\top
+\frac{\lambda_t}{t}I_d.
\]
Therefore,
$
H_t(\hat\theta_t)^{-1}=\frac{1}{t}\widetilde H_t^{-1},
$
and hence
\[
\frac{Z(t)}{t}
=
\min_{i\neq i^\star(\hat\theta_t)}
\frac{1}{2}\,
\frac{\big(\hat{\theta}_t^\top z_{i\,i^\star(\hat\theta_t)}\big)^2}
{\big\|z_{i\,i^\star(\hat\theta_t)}\big\|^2_{\widetilde H_t^{-1}}}.
\]
Note that
\[
\widetilde H_t
-
H\!\left(\hat\theta_t,\frac{N(t)}{t}\right)
=
\frac{\lambda_t}{t}I_d.
\]
Since \(\lambda_t\rightarrow0\), we have
$
\frac{\lambda_t}{t}I_d\to 0.
$
Thus,
\[
\widetilde H_t - H\!\left(\hat\theta_t,\frac{N(t)}{t}\right)\to0.
\]

Moreover, on the event \(\mathcal E\),
\[
H\!\left(\hat\theta_t,\frac{N(t)}{t}\right)\to H(\theta_\star,\omega^\dagger).
\]
Since \(H(\theta_\star,\omega^\dagger)\succ0\), it follows that \(\widetilde H_t\succ0\) for all sufficiently large \(t\), and
\[
\widetilde H_t^{-1}
-
H\!\left(\hat\theta_t,\frac{N(t)}{t}\right)^{-1}
\to 0.
\]

Therefore,
\[
\frac{Z(t)}{t}
-
\Psi\!\left(\hat\theta_t,\frac{N(t)}{t}\right)
\to 0.
\]

Combining this with
\[
\Psi\!\left(\hat\theta_t,\frac{N(t)}{t}\right)
\to
\mathcal U^\star(\mu)^{-1},
\]
we obtain
\[
\frac{Z(t)}{t}\to \mathcal U^\star(\mu)^{-1}.
\]

Hence there exists \(t_\varepsilon'\ge1\) such that for all
\(t\ge t_\varepsilon'\),
\begin{equation}
\label{eq:Z_lower_final}
Z(t)\ge t(1-\varepsilon)\mathcal U^\star(\mu)^{-1}.
\end{equation}

Next, by Lemma~\ref{lem:tracking_properties}, there exists \(t_{\rm tr}\ge1\) such that
for all \(t\ge t_{\rm tr}\),
\[
N_{ij}(t)\ge c^\prime\sqrt t-1,
\qquad \forall (i,j)\in\mathcal A_0.
\]
Since \(\mathcal A_0\) is chosen so that
$
\mathrm{span}\{z_{ij}:(i,j)\in\mathcal A_0\}=\mathbb R^d,
$
the matrix
\[
G_0:=\sum_{(i,j)\in\mathcal A_0} z_{ij}z_{ij}^\top
\]
is positive definite.

Therefore, for all \(t\ge t_{\rm tr}\),
\[
V_t
=
\sum_{(i,j)\in\mathcal S}N_{ij}(t)z_{ij}z_{ij}^\top
\ \succeq\
(c^\prime\sqrt t-1)G_0.
\]
Hence
\[
\Lambda_{\min}(t)=\lambda_{\min}(V_t)\ge (c^\prime\sqrt t-1)\lambda_{\min}(G_0).
\]

It follows that there exist a constant \(c>0\) and a time \(t_0\ge1\) such that
\[
\Lambda_{\min}(t)\ge c\sqrt t,
\qquad \forall t\ge t_0.
\]

Since the stopping rule is defined by
\[
\tau
=
\inf\left\{t\ge t_0:\ \Lambda_{\min}(t)\ge c\sqrt t,\ Z(t)>\beta(\delta,t)\right\},
\]
we obtain
\[
\tau
\le
\max\{t_0,t_\varepsilon',t_{\rm id}\}
\vee
\inf\Bigl\{t\in\mathbb N^\star:\ (1-\varepsilon)t\mathcal U^\star(\mu)^{-1}>\beta(\delta,t)\Bigr\}.
\]
Using the threshold bound in Lemma \ref{eq: beta_order}
\[
\beta(\delta,t)\le c_1\log\Big(\frac{c_2 t^\alpha}{\delta}\Big),
\]
we further obtain
\[
\tau
\le
\max\{t_0,t_\varepsilon^\prime,t_{\rm id}\}
\vee
\inf\Bigl\{t\in\mathbb N^\star:\ (1-\varepsilon)t\mathcal U^\star(\mu)^{-1}
>
c_1\log\Big(\frac{c_2 t^\alpha}{\delta}\Big)\Bigr\}.
\]
Equivalently,
\[
\tau
\le
\max\{t_0,t_\varepsilon^\prime,t_{\rm id}\}
\vee
\inf\Bigl\{t\in\mathbb N^\star:\ 
\frac{(1-\varepsilon)}{c_1}\mathcal U^\star(\mu)^{-1} t
>
\log\Big(\frac{c_2 t^\alpha}{\delta}\Big)\Bigr\}.
\]
Applying Lemma \ref{lem:technical_log} with
\[
\widetilde c_1=\frac{(1-\varepsilon)}{c_1}\mathcal U^\star(\mu)^{-1},
\qquad
\widetilde c_2=\frac{c_2}{\delta},
\]
it follows that, for sufficiently small \(\delta\),
\[
\tau
\lesssim
\max\left\{
t_0,t_\varepsilon^\prime,t_{\rm id},\,
\frac{c_1}{1-\varepsilon}\mathcal U^\star(\mu)\log\frac1\delta
\right\}.
\]
Hence, on \(\mathcal E\),
\[
\limsup_{\delta\rightarrow0}\frac{\tau}{\log(1/\delta)}
\lesssim
\frac{c_1}{1-\varepsilon}\mathcal U^\star(\mu).
\]
Letting \(\varepsilon\rightarrow0\), we obtain
\[
\limsup_{\delta\rightarrow0}\frac{\tau}{\log(1/\delta)}
\lesssim
\mathcal U^\star(\mu)
\qquad\text{on }\mathcal E.
\]
Since \(\mathbb P(\mathcal E)=1\), the desired result follows.

Finally, the above bound implies \(\tau<\infty\) on \(\mathcal E\), hence
\(\mathbb P(\tau<\infty)=1\).
\qed

\section{Real Experiment Details}
\label{sec: real details}
We evaluate our method on four task types spanning counting, lexical reconstruction, character-level extraction, and basic arithmetic. 
\begin{itemize}
    \item \textbf{Object Counting} requires the model to identify all items belonging to a target category from a mixed list and return the correct count, thereby testing selective counting and category filtering.
    \item \textbf{Word Unscrambling} asks the model to recover the original English word from a scrambled sequence of letters, evaluating lexical knowledge and character-level reordering ability.
    \item \textbf{Second Word Letter} is a simple character extraction task from the Instruction-Induction benchmark; in our dataset, each example consists of a single English word, and the model must output its second letter.
    \item \textbf{Sum} is a basic arithmetic task in which the input contains two integers and the model is asked to return their sum, providing a controlled setting for deterministic numerical reasoning.
\end{itemize}  

Examples from the dataset are shown below:
\begin{itemize}
    \item \textbf{Object Counting}
    \begin{itemize}
        \item \begin{tabular}[t]{@{}ll@{}}
            \textbf{Question:} & ``I have a pig, two ducks, and a dog. How many animals do I have?'' \\
            \textbf{Answer:}   & ``four''
        \end{tabular}
        
        \item \begin{tabular}[t]{@{}ll@{}}
            \textbf{Question:} & ``I have a snake, a rabbit, two clarinets, a drum, a flute, and an accordion. \\
                               & How many musical instruments do I have?'' \\
            \textbf{Answer:}   & ``five''
        \end{tabular}
    \end{itemize}
    
    \item \textbf{Word Unscrambling}
    \begin{itemize}
        \item \begin{tabular}[t]{@{}ll@{}}
            \textbf{Question:} & ``The word \texttt{ceahb} is a scrambled version of the English word'' \\
            \textbf{Answer:}   & \texttt{beach}
        \end{tabular}
        
        \item \begin{tabular}[t]{@{}ll@{}}
            \textbf{Question:} & ``The word \texttt{iasso} is a scrambled version of the English word'' \\
            \textbf{Answer:}   & \texttt{oasis}
        \end{tabular}
    \end{itemize}
    
    \item \textbf{Second Word Letter}
    \begin{itemize}
        \item \begin{tabular}[t]{@{}ll@{}}
            \textbf{Question:} & \texttt{government} \\
            \textbf{Answer:}   & \texttt{o}
        \end{tabular}
        
        \item \begin{tabular}[t]{@{}ll@{}}
            \textbf{Question:} & \texttt{apple} \\
            \textbf{Answer:}   & \texttt{p}
        \end{tabular}
    \end{itemize}
    
    \item \textbf{Sum}
    \begin{itemize}
        \item \begin{tabular}[t]{@{}ll@{}}
            \textbf{Question:} & \texttt{12 29} \\
            \textbf{Answer:}   & \texttt{41}
        \end{tabular}
        
        \item \begin{tabular}[t]{@{}ll@{}}
            \textbf{Question:} & \texttt{38 57} \\
            \textbf{Answer:}   & \texttt{95}
        \end{tabular}
    \end{itemize}
\end{itemize}

%
%
%






\end{document}